\documentclass[11pt]{article}
\usepackage[margin=1in]{geometry} 
\usepackage{amsfonts}
\usepackage{amssymb}
\usepackage{amsmath}
\usepackage{amsthm}
\usepackage{tocloft}
\usepackage{multirow}
\usepackage{setspace}
\usepackage{parskip}
\usepackage{xspace}
\usepackage{authblk}
\usepackage{array}
\usepackage[round]{natbib}
\usepackage[pagebackref=true]{hyperref}
\hypersetup{
	colorlinks=true,
	linkcolor=magenta,
	citecolor=blue
}
\usepackage{bbm}
\usepackage{xspace}
\usepackage{xcolor}

\newtheorem{theorem}{Theorem}[section]
\newtheorem{proposition}{Proposition}[section]
\newtheorem{corollary}{Corollary}[section]
\newtheorem{lemma}[theorem]{Lemma}
\newtheorem{remark}{Remark}

\theoremstyle{definition}
\newtheorem{definition}{Definition}

\DeclareMathOperator{\klsf}{\textnormal{\textsf{kl}}}

\renewcommand{\subset}{\subseteq}
\renewcommand{\supset}{\supseteq}
\renewcommand{\le}{\leqslant}
\renewcommand{\ge}{\geqslant}
\renewcommand{\leq}{\le}
\renewcommand{\geq}{\ge}

\newcommand{\X}{\mathcal{X}}
\newcommand{\mix}{\text{mix}}
\newcommand{\la}{\langle}
\newcommand{\ra}{\rangle}
\newcommand{\R}{\mathbb{R}}
\newcommand{\E}{\mathbb{E}}
\newcommand{\N}{\mathbb{N}}
\newcommand{\F}{\mathcal{F}}
\newcommand{\id}{\mathbf{1}}
\newcommand {\Exp}{ \mathbb E }
\renewcommand {\Pr}{ \mathbb P }

\newcommand {\Var}{\textnormal{\textsf{Var}}}

\newcommand{\gp}{\mathcal{GP}}
\newcommand{\ind}{\mathbf{1}}
\newcommand{\sgn}{\textnormal{sgn }\xspace}

\newcommand{\kl}{D_{\operatorname{KL}}}

\renewcommand{\d}{\mathrm{d}}
\newcommand{\vp}{\varphi}
\newcommand{\risk}{R}
\newcommand{\mE}{\mathcal{E}}
\newcommand{\cZ}{\mathcal{Z}}
\newcommand{\cY}{\mathcal{Y}}
\newcommand{\cR}{\mathcal{R}}
\newcommand{\cG}{\mathcal{G}}
\newcommand{\Mspace}[1]{\mathcal{M}(#1)}

\newcommand{\ber}{\text{Ber}}

\newcommand{\hR}{\widehat{R}}

\newcommand{\dist}{\mathcal{D}}
\newcommand{\bin}{\text{Bin}}

\usepackage{xcolor}

\newcommand{\cT}{\mathcal{T}}
\newcommand{\fM}{\mathfrak{M}}

\newcommand{\il}{\mathsf{IL}}
\newcommand{\K}{\mathcal{K}}
\newcommand{\edit}[1]{{\color{black} #1\xspace}}

\title{A unified recipe for deriving\\ (time-uniform) PAC-Bayes bounds\footnote{Accepted to the Journal of Machine Learning Research}}
\author[1]{Ben Chugg}
\author[1]{Hongjian Wang}
\author[1,2]{Aaditya Ramdas}
\affil[1]{Machine Learning Department}
\affil[2]{Department of Statistics and Data Science}
\affil[ ]{Carnegie Mellon University}
\affil[ ]{\texttt{\{benchugg, hjnwang, aramdas\}@cmu.edu}}

\date{December 2023}

\begin{document}

\maketitle

\begin{abstract}%
    We present a unified framework for deriving PAC-Bayesian generalization bounds. Unlike most previous literature on this topic, our bounds are anytime-valid (i.e., time-uniform), meaning that they hold at all stopping times, not only for a fixed sample size. Our approach combines four tools in the following order: (a) nonnegative supermartingales or reverse submartingales, (b) the method of mixtures, (c) the Donsker-Varadhan formula (or other convex duality principles), and (d) Ville's inequality. 
    Our main result is a PAC-Bayes theorem which holds for a wide class of discrete stochastic processes. 
    We show how this result implies time-uniform versions of well-known classical PAC-Bayes bounds, such as those of Seeger, McAllester, Maurer, and Catoni, in addition to many recent bounds. We also present several novel bounds. 
  Our framework also enables us to relax traditional assumptions; in particular, we consider nonstationary loss functions and non-i.i.d.\ data. 
		In sum, we unify the derivation of past bounds and ease the search for future bounds: one may simply check if our supermartingale or submartingale conditions are met and, if so, be guaranteed a (time-uniform) PAC-Bayes bound. 
  \end{abstract}

% TOC 
{
    \small
    \setcounter{tocdepth}{1}
    \hypersetup{linkcolor=black}
    \tableofcontents
}
% \newpage 

\section{Introduction}
\label{sec:intro}

PAC-Bayesian theory is broadly concerned with providing generalization guarantees over mixtures of predictors in statistical learning problems.  It emerged in the late 1990s, catalyzed by an early paper of \citet{shawe1997pac} and shepherded forward by McAllester \citep{mcallester1998some,mcallester1999pac,mcallester2003simplified}, Catoni \citep{catoni2003pac,catoni2004statistical,catoni2007pac}, Maurer \citep{maurer2004note}, and Seeger \citep{seeger2002pac,seeger2003bayesian}, among others. The earliest works were focused mainly on classification settings but the techniques have expanded to regression settings \citep{audibert2004aggregated,alquier2008pac}, and more recently to settings beyond supervised learning (e.g., \citealp{seldin2010pac}). We refer the reader to \citet{alquier2021user} and \citet{guedj2019primer} for excellent surveys. 
 
    In the supervised learning setting, PAC-Bayesian (or simply ``PAC-Bayes'') theory seeks to bound the expected risk in terms of the expected empirical risk, where the expectation is with respect to a data-dependent distribution $\rho$ over the hypothesis space.  
    This is in contrast to uniform convergence  guarantees, which give worst case bounds over all hypotheses. The PAC-Bayes approach is not without limitations~\citep{livni2020limitation}, but has led to non-trivial guarantees for SVMs~\citep{ambroladze2006tighter}, sparse additive models~\citep{guedj2013pac}, and neural networks~\citep{dziugaite2017computing,letarte2019dichotomize}. 
    Whereas uniform convergence bounds typically rely on some 
    notion of the complexity of the hypothesis class, PAC-Bayes bounds depend on the distance between $\rho$ and a prior distribution $\nu$.   
    Depending on the choice of $\nu$ and $\rho$, the resulting bounds can be tighter and easier to compute. 

     Despite these successes, we point out two drawbacks.
     First, there does not seem to be a clearly established recipe to deriving PAC-Bayes bounds. Many full-length papers are dedicated to deriving one or two interesting bounds, using different techniques. Is there a common thread to tie the decades of work together? Can a unified view (achieved with the power of hindsight) yield new bounds with relative ease? 
     Second, most existing PAC-Bayes bounds are fixed-time results. That is, the bounds hold at a fixed number of observations determined \emph{a priori}, despite the fact that the distribution $\rho$ can be data-dependent. In fact, this is the case for the vast majority of the learning theory literature. Undoubtedly, this is a consequence of the extensive number of fixed time concentration inequalities stemming from the statistics literature (e.g., the Chernoff bound and the Azuma-Hoeffding inequality; see \citealp{boucheron2013concentration} for an overview). However, fixed-time bounds are not valid at stopping times; if the bound is computed at a sample size that is itself data-dependent (perhaps resulting from sequential decisions), then it is invalid. Naïve union bounds over all of time are too loose, falling short theoretically, practically, and aesthetically.

In this work, we take advantage of recent progress on unified schemes for deriving \emph{anytime-valid} concentration inequalities \citep{howard2020time,howard2021time} to give a general framework for developing anytime-valid (a.k.a.\ time-uniform)\footnote{In this paper, ``anytime-valid" and ``time-uniform" are synonymous.
However, this is not always the case.
See the discussion at the end of Section~\ref{sec:setting}.}
PAC-Bayes bounds. 
Anytime-valid bounds hold at all stopping times. Importantly, this means they hold regardless of whether one has looked at the data or not when deciding the final sample size. They are thus inherently immune to continuous monitoring of data and adaptive stopping.

Concurrent to our own work,  \citet{haddouche2022pac} derived several anytime-valid PAC-Bayes bounds. They also employ supermartingales and Ville's inequality, two ingredients which are  central to our approach. Our general framework will encompass their results, recovering their theorems as special cases of our own. More importantly however, our unified framework will cover a much broader slew of existing PAC-Bayes bounds. See Table~\ref{tab:generalization} for a summary of these results.

At a high level, our approach combines four tools in the following order: \edit{(A)} nonnegative supermartingales or reverse submartingales, \edit{(B)} mixtures of said processes (often called the ``method of mixtures''), \edit{(C)} a change-of-measure inequality which provides a variational representation of some convex divergence (e.g, the Donsker-Varadhan formula in the case of KL divergence), and \edit{(D)} Ville's inequality~\citep{ville1939etude}, a time-uniform extension of Markov's inequality to nonnegative supermartingales and reverse submartingales. Recent work has established that principles \edit{(A)+(D)} yield a unified approach to deriving time-uniform Chernoff bounds (e.g., \citealp{howard2020time}), while using \edit{(A)+(B)+(D)} yields a unified approach to deriving confidence sequences (e.g.,~\citealp{howard2021time}). This paper shows that adding \edit{(C)} yields a unified approach to PAC-Bayes bounds. See Figure~\ref{fig:lit-schema} for a schema of how this work relates to other unified recipes and time-uniform bounds.

\begin{figure*}
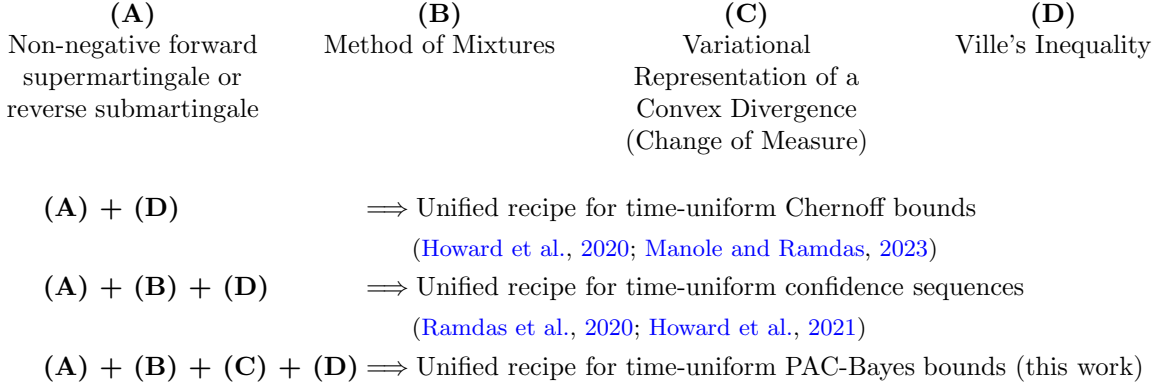

    \small 
     \centering
     \begin{minipage}[t]{0.24\textwidth}
     \centering
     \textbf{(A)} \\ 
     Non-negative forward supermartingale or 
     reverse submartingale
     \end{minipage}
     \begin{minipage}[t]{0.24\textwidth}
     \centering
     \textbf{(B)} \\ 
     Method of Mixtures
     \end{minipage}
     \begin{minipage}[t]{0.24\textwidth}
     \centering
     \textbf{(C)} \\ 
     Variational Representation of a Convex Divergence \\(Change of Measure)  
     \end{minipage}
     \begin{minipage}[t]{0.24\textwidth}
     \centering
     \textbf{(D)} \\ 
     Ville's Inequality
     \end{minipage}
     \begin{minipage}{0.7\textwidth}
     \begin{alignat*}{2}
         &\textbf{(A) + (D)} && \Longrightarrow \textnormal{Unified recipe for time-uniform Chernoff bounds} \\
         & && 
         \qquad {\footnotesize {\text{\citep{howard2020time,manole2021sequential}}}}
         \\ 
         &\textbf{(A) + (B) + (D)} &&\Longrightarrow \textnormal{Unified recipe for time-uniform confidence sequences} \\
         & &&          \qquad {\footnotesize\text{\citep{ramdas2020admissible,howard2021time}}}
         \\
         &\textbf{(A) + (B) + (C) + (D)} && \Longrightarrow 
         \textnormal{Unified recipe for time-uniform PAC-Bayes bounds (this work)}
     \end{alignat*}
     \end{minipage}
     \caption{An overview of the tools employed in this paper, and how they relate to previous work on time-uniform bounds.}
     \label{fig:lit-schema}
     \vspace{-0.5em}
 \end{figure*}

	\subsection{Setting}
	\label{sec:setting}
	We observe a sequence of data $(Z_t)_{t=1}^\infty$ where each $Z_i$ lies in some domain $\cZ$. The data have a distribution $\dist$ over $\cZ^\infty$. 
    We emphasize that $\dist$ is a distribution over \emph{sequences} of observations, enabling us to consider non-i.i.d.\ data. We will specify the precise distributional assumptions later on.  
	Each time step $t$ is associated with a function 
	$f_t:\cZ\times\Theta\to\R_{\geq 0}$, where $\Theta$ is some (measurable) parameter space. Each $\theta\in\Theta$ gives rise to the \emph{loss function} $f_t(\cdot,\theta)$. Thus, $f_t$ should be seen as a family of loss functions parameterized by $\Theta$. 
	If $f=f_t$ does not change with time, we say it is \emph{stationary}.

 In a typical supervised learning task, the domain is taken to be the product $\cZ=\X\times \cY$, where $\X$ is the feature space and $\cY$ the label space. In this case, we might consider the (stationary) loss function $f(Z_t,\theta)=(Y_t - \la \theta, X_t\ra)^2$, where $Z_t=(X_t,Y_t)$. However, PAC-Bayesian bounds have proven useful outside of supervised learning---for instance, estimating means~\citep{catoni2017dimension,catoni2018dimension}, clustering~\citep{seldin2010pac}, and discrete density estimation~\citep{seeger2003bayesian,seldin2009pac}.  Thus, we choose to adopt the more general notation.   
 We note that allowing the loss function to change as a function of time is not the typical assumption in the PAC-Bayes literature. However, we find that our framework can handle non-stationary losses at no extra cost, so we see no harm (and some benefit) in this additional level of generality.

	For a fixed $\theta\in\Theta$, the  empirical risk and the (conditional) risk at time $t$ are, respectively,
\begin{equation}
 \label{eq:risk}
		\hR_t(\theta) = \frac{1}{t}\sum_{i=1}^t f_i(Z_i,\theta), \quad \text{and}\quad \risk_t(\theta) = \frac{1}{t}\sum_{i=1}^t \E[f_i(Z_i,\theta)|\F_{i-1}].
\end{equation}
Here $\F_{i-1}$ is the $\sigma$-algebra generated by $Z_1,\dots,Z_{i-1}$ (formally introduced in Section~\ref{sec:background}). 
If the losses are stationary  
 and the data are i.i.d.\ (or, more generally, $\E[f_t(Z_t,\theta)|\F_{t-1}]$ is assumed to have a common mean across all $t\geq 1$) then the conditional risk is constant as a function of time, and we denote it as $\risk(\theta) = \E[f(Z,\theta)]$. 
 % and the total risk coincide: $\risk_t(\theta) = \risk(\theta)$. Most prior work makes such an assumption and is therefore concerned with bounding the difference between the empirical risk and total risk. When we consider non-i.i.d.\ data, however, we will concern ourselves with bounding the difference between the empirical risk and the conditional risk. 
 
Uniform convergence guarantees provide a natural and popular way to bound the risk in terms of the empirical risk. 
Such guarantees provide bounds simultaneously for all $\theta\in\Theta$, and typically depend on quantities such as the VC dimension or the Rademacher complexity of the family of losses (see, e.g., \citealp{wainwright2019high}). In contrast, PAC-Bayes bounds seek to give guarantees on the difference between $\E_{\theta\sim\rho} \hR_t(\theta)$ and $\E_{\theta\sim\rho} R_t(\theta)$ for all data-dependent mixture distributions $\rho\in\Mspace{\Theta}$, where $\Mspace{\Theta}$ is the set of probability distributions over $\Theta$. 
	Additionally, we typically begin with a (data-free) prior $\nu\in\Mspace{\Theta}$ over the parameters. 
 
 In order to orient the reader, we state a PAC-Bayes bound due to \citet{catoni2003pac} for bounded, stationary losses in $[0,1]$. The order of quantifiers below is particularly important to note. For all priors $\nu\in\Mspace{\Theta}$, error probabilities $\delta\in(0,1)$, sample sizes $n$ and tuning parameters $\lambda>0$, we have that with probability at least $1-\delta$, for all $\rho\in\Mspace{\Theta}$, 
	\begin{equation}
		\label{eq:catoni-bounded}
		\E_{\theta\sim\rho} [\risk_n(\theta) - \hR_n(\theta)] \leq   \frac{\lambda}{8n} + \frac{\kl(\rho\|\nu) + \log(1/\delta)}{\lambda},
	\end{equation}
	where $\kl(\rho\|\nu)$ is the KL divergence between $\rho$ and $\nu$ (defined in Section~\ref{sec:background}). Said differently, \emph{``Fixing $\nu,\delta,n,\lambda$, with probability $1-\delta$, \eqref{eq:catoni-bounded} holds simultaneously for all $\rho$.''}, emphasizing the quantities that are fixed before seeing the data. 
 %The interested reader can jump ahead to Corollary~\ref{cor:anytime-subGaussian-losses} for a time-uniform version of this statement (see~\eqref{eq:anytime-bounded-proc} in particular), where $n$ is not fixed in advance. 
 
 Notice that the generalization guarantee depends not on a measure of complexity of the class of functions $\{f(\cdot,\theta):\theta\in\Theta\}$ as it would in uniform convergence bounds. Instead, it depends on the divergence between our prior $\nu$ and a data-dependent $\rho$. 
	The KL divergence is the most common measure of divergence used in PAC-Bayes bounds because of the famous ``change of measure'' inequality by \citet{donsker1975large} but R\'enyi divergence \citep{begin2016pac}, $f$ divergences \citep{alquier2018simpler,ohnishi2021novel}, and Integral Probability Metrics~\citep{amit2022integral} have also been studied.

 Let us now introduce anytime-valid and time-uniform bounds. As stated, \eqref{eq:catoni-bounded} is a fixed-time bound because, as discussed above, the universal quantifier on $n$ is ``outside'' the probability statement. This is characteristic of most concentration inequalities. A time-uniform bound, on the other hand, incorporates the number of samples ``inside'' the probability statement. It is of the form ``with probability $1-\delta$, for all $n$, ...''. Moving forward, we will substitute $t$ (standing for time) in place of $n$ to draw attention to the distinction. \edit{
 For instance, here is the time-uniform equivalent of \eqref{eq:catoni-bounded} above. 
 For all priors $\nu\in\Mspace{\Theta}$, error probabilities $\delta\in(0,1)$, and tuning parameters $\lambda,n>0$,  with probability at least $1-\delta$, we have that simultaneously for all $t\geq 1$ and $\rho\in\Mspace{\Theta}$, 
 \begin{equation}
	\E_{\theta\sim\rho}[R_t(\theta) - \hR_t(\theta)] \leq 
 \frac{\lambda}{8n} + \frac{\kl(\rho\|\nu) + \log(1/\delta)}{\lambda t/n}.
	\end{equation}
 Here we have kept a pre-specified $n$ in the bound to facilitate easy comparison with \eqref{eq:catoni-bounded}; however, this parameter could be absorbed into $\lambda$. 
 }
 While the distinction between time-uniform and fixed-time bounds may seem a minor notational detail, 
 it is in fact a major mathematical difference with ramifications across science and any kind of data-driven decision-making~\citep{howard2021time,grunwald2020safe,ramdas2022game}. Importantly, time-uniform results are immune to ``peeking'' because they remain valid at stopping times.

 Anytime-valid bounds, meanwhile, are (in)equalities that hold at arbitrary stopping times. 
 A full discussion of the distinction between anytime-valid and time-uniform bounds is beyond the scope of this work, but we refer the interested reader to \citet{ramdas2020admissible} for further detail (see Lemmas 2 and 3 in particular).
 Suffice it to say that for probability statements like above, time-uniformity is synonymous with anytime-validity. 
 % For statements only involving expectations, however, they are not.  
 This manuscript is concerned with anytime-valid probability 
 statements, so we use the two terms interchangeably.

\subsection{Contributions and Outline}

 In this work, we identify a general martingale-like structure at the heart of many existing PAC-Bayes bounds. This structure takes the form of either a nonnegative supermartingale or a nonnegative reverse submartingale. Such an identification enables us to (i) give a general framework for seeking new bounds, and (ii) give time-uniform extensions of many existing PAC-Bayes bounds. 
 Our main contribution is a general result (Theorem~\ref{thm:pac-bayes-bounded-proc})  which provides a  time-uniform PAC-Bayes bound for any process which is (upper bounded by) a nonnegative supermartingale or reverse submartingale. 
 We proceed to instantiate this bound with a variety of particular processes and relate them to existing results in the literature (Table~\ref{tab:generalization}). 
 For those bounds which admit a supermartingale structure, we find that their time-uniform extensions remain as tight as their fixed-time counterparts. 
 For those that admit a reverse submartingale structure we provide two results: (a) a time-uniform bound holding for all $t\geq 1$ which 
loses at most a constant factor plus an iterated logarithm term (i.e., $\log\log t$) over the original, and (b) a bound which holds for all times $t\geq n$, where $n$ is some time of special interest chosen beforehand, which  remains just as tight as the original fixed-time bounds. 
Finally, our framework enables us to relax many traditional assumptions (Table~\ref{tab:conditions}). For instance, many of our bounds do not require i.i.d.\ data. In fact, our supermartingale-based bounds require no explicit distributional assumptions.  

\renewcommand{\arraystretch}{1.19}
    \begin{table}[p!]
		\small 
		\centering
		\begin{tabular}{rll}
			\emph{} & \emph{Existing result} & \emph{Our result} \\
			\hline 
			\multirow{11}{*}{\shortstack[r]{
   \emph{Forward} \\ \emph{supermartingale} 
   }} 
   & \citet{mcallester1999pac}, Thm. 1 & Corollary~\ref{cor:mixture-subGaussian} \\
   & \citet{catoni2003pac} & Corollary~\ref{cor:anytime-subGaussian-losses} \\
			& \citet{catoni2007pac} & Corollary~\ref{cor:cgf}\\
			&\citet{seldin2012pac}, Thm. 5 \& 6 &  Corollary~\ref{cor:mds-bounded} \\
            & \citet{seldin2012pac}, Thm. 7 \& 8 & Corollary~\ref{cor:mds-bernstein} \\
        & \citet{balsubramani2015pac}, Thm. 1 & Corollary~\ref{cor:mds-bernstein} \\
			& \citet{alquier2016properties}, Thm. 4.1 & Corollary~\ref{cor:anytime-subGaussian-losses} \\
   & \edit{\citet{kuzborskij2019efron}, Thm. 4} & \edit{Corollary~\ref{cor:anytime-pac-bayes-sub-psi}} \\ 
			&\citet{haddouche2021pac}, Thm. 3 & Corollary~\ref{cor:anytime-second-moment}\\
			&\citet{haddouche2022pac}, Thm. 5 & Corollary \ref{cor:anytime-pac-bayes-sub-psi}\\
			&\citet{haddouche2022pac}, Thm. 7 & Corollary \ref{cor:anytime-sn1} \\
   & \edit{\citet{jang2023tighter}, Thm. 1} & \edit{Corollary \ref{cor:betting-style-mart}} \\
			\hline 
			\multirow{14}{*}{\shortstack[r]{
   \emph{Reverse} \\ \emph{submartingale}  
   }} 
			& \citet{mcallester1999pac}, Thm. 1 & Corollary~\ref{cor:anytime-mcallester}\\
        & \citet{langford2001bounds}, Thm. 3 & Corollary~\ref{cor:anytime-seeger} \\
   & \edit{\citet{seeger2002pac}, Thm. 2} & \edit{Corollary~\ref{cor:anytime-gp}} \\
			& \citet{maurer2004note}, Thm. 5 & Corollary~\ref{cor:anytime-seeger} \\
   & \citet{catoni2007pac}, Thm. 1.2.6 & Corollary~\ref{cor:anytime-convex} \\
   & \citet{germain2009pac}, Thm. 2.1 & Corollary~\ref{cor:anytime-convex} \\
        & \citet{seldin2012pac}, Thm. 4 & Corollary~\ref{cor:mds-convex} \\ 
			& \citet{tolstikhin2013pac}, Eqn. 3 
   & Corollary~\ref{cor:anytime-seeger}  \\
   & \edit{\citet{tolstikhin2013pac}, Thm. 3 \& 4} & \edit{Corollary \ref{cor:tolstikhin}} \\
   & \citet{germain2015risk}, Thm. 18 & Corollary~\ref{cor:anytime-convex}\\
   & \citet{begin2016pac}, Thm. 9 & Corollary~\ref{cor:renyi-convex} \\
			& \citet{thiemann2017strongly}, Thm. 3 & Corollary~\ref{cor:anytime-convex} \\
   & \citet{alquier2021user}, Eqn. (3.1) & Corollary~\ref{cor:anytime-seeger} \\
             & \citet{amit2022integral}, Prop. 4 and 5 & Corollary~\ref{cor:ipm-convex} 
		\end{tabular}
		\caption{ 
  %\hrmk{i know we are kind of out of space here but cor 11 generalizing Mhammedi et al is missing here? we can shorten the caption if necessary}
  A summary of how various existing results are related to our framework. 
  The first column refers to the type of underlying process used to construct the bound. 
  For supermartingales, the time-uniform extension sacrifices no tightness compared to the original. For reverse submartingales, our anytime bound loses essentially an iterated logarithm factor over the fixed-time bound (but the fixed-time bound itself remains recoverable at no loss).  
  The final column points to which corollary implies the existing result (either directly or as a consequence of selecting certain parameters; the precise relationship will be described in the text). 
  The above results are mostly corollaries of Theorem~\ref{thm:pac-bayes-bounded-proc} (a PAC-Bayes framework with the KL divergence), but several rely on Theorem~\ref{thm:pac-bayes-phi-div} (a framework for general $\phi$-divergences) or Theorem~\ref{thm:anytime-alpha-div} (a framework for R\'{e}nyi divergences).  
			The PAC-Bayes literature is large and we cannot include all previous results and their relationships, but we hope this gives the reader an idea of the scope of our approach. 
   %All existing results, save for those of \citet{haddouche2022pac} and \citet{balsubramani2015pac}, are fixed-time bounds. 
   We do not provide numbers in the second and third rows because the bounds were not explicitly written out in \citet{catoni2003pac,catoni2007pac}. See \citet{alquier2021user} for a summary.} 
		\label{tab:generalization}
	\end{table}

As was mentioned in the introduction, the closest work to ours is the concurrent preprint of \citet{haddouche2022pac}. They apply Ville's inequality to a supermartingale identified by \citet{bercu2008exponential}, which gives a time-uniform PAC-Bayes bound for unbounded loss functions. In Section~\ref{sec:supermart_bounds} we will demonstrate that this supermartingale was known to be a part of a much wider class of stochastic processes known as sub-$\psi$ processes \citep{howard2020time}, and provide an anytime-valid PAC-Bayes result for this large class, recovering their result as a special case.

 Stepping back from the particulars, our work is best viewed in the spirit of recent progress in time-uniform Chernoff bounds and sequential estimation (Figure~\ref{fig:lit-schema}). We draw much inspiration from the recent works by \citet{howard2020time,howard2021time} who study a unifying approach to time-uniform bounds via supermartingales. 
	\citet{howard2020time} showed that many (or most, or all) Chernoff bounds can be made time-uniform at no loss (and sometimes a gain) by identifying an appropriate supermartingale and applying Ville's inequality (our Lemma~\ref{lem:ville_supermartingales}). 
	In other words, applying Ville's inequality to nonnegative supermartingales is a unifying strategy for generating Chernoff bounds. This insight was the inspiration for seeking to identify underlying supermartingales in PAC-Bayes bounds. 
	\citet{howard2021time} then built upon this foundation, and developed confidence sequences 
	(i.e., confidence intervals that hold at all stopping times) with zero asymptotic width using a variety of mixtures of supermartingales.
	This ``method of mixtures'' plays an important role in our results in two respects. For one, it is required since the PAC-Bayes framework gives bounds over mixtures of hypotheses. Second, it yields novel PAC-Bayes bounds by mixing the supermartingales that underlie existing bounds with various mixing distributions. 
    The former yields uniformity over distributions,  the latter over sample size.

	Interestingly, we find that not all existing PAC-Bayes bounds can be given time-uniform generalizations based on nonnegative supermartingales. For some, including those of \cite{seeger2003bayesian,tolstikhin2013pac,germain2015risk} which ultimately rely on applying convex functions to the risk and empirical risk, we must instead rely on \emph{reverse submartingales}. Our inspiration for such tools comes from recent work by \citet{manole2021sequential}, who showed that convex functionals and divergences are reverse submartingales (with respect to the exchangeable filtration). Since there also exists a reverse-time Ville's inequality, backwards  submartingales and (backwards) Ville's inequality provide a second unifying recipe for deriving time-uniform bounds. 
	
	In short, this paper shows how systematically combining four techniques provides a unified recipe to derive time-uniform PAC-Bayesian inequalities. 
 
 % \paragraph{Outline.}
 \textit{Outline.}
 The rest of the manuscript is organized as follows. Section~\ref{sec:background} provides relevant background on (reverse) martingales, Ville's inequalities, and the change-of-measure inequality which lies at the heart of PAC-Bayesian analysis. 
 Section~\ref{sec:general-theorem} provides a ``master theorem'' which gives an anytime-valid PAC-Bayes bound for general nonnegative stochastic processes which are upper bounded by either a   supermartingale or reverse submartingale. Section~\ref{sec:supermart_bounds} then explores various consequences in the supermartingale case, and Section~\ref{sec:submart_bounds} does the same for the reverse submartingale case. 
 Section~\ref{sec:extensions} then discusses a number of extensions;
 Sections~\ref{sec:beyond-kl} and \ref{sec:phi-div} study extensions of our master Theorem to Integral Probability Metrics, $\phi$-divergences, and the R\'{e}nyi divergence.  Section~\ref{sec:confseq} gives some connections to recent work on time-uniform confidence sequences, Section~\ref{sec:martingale-differences} demonstrates that our results hold for martingale difference sequences, \edit{and Section~\ref{sec:data-dependent-priors} investigates to what extent we can employ data-dependent priors. Finally, Section~\ref{sec:gp} ends with an application to Gaussian process classification.}

\section{Background}
 \label{sec:background}

 % \paragraph{Notation.} 
 \textit{Notation.}
 As discussed previously, we let $\dist$ be a distribution over sequences $(Z_t)\in \cZ^\infty$. 
	In order to save ourselves from an overload of notation, we will write $\E_\dist [\cdot]$ to denote the expectation when drawing $(Z_t)\sim \dist$, i.e., $\E_\dist[\cdot] = \E_{(Z_t)\sim \dist}[\cdot]$. 
	Furthermore, we will use the convention that expectation over lowercase Greek letters refer to expectation over parameters $\theta\in\Theta$, e.g., $\E_{\rho}[\cdot] = \E_{\theta\sim\rho}[\cdot]$. We also write $Z^n$ as shorthand for $Z_1,\dots,Z_n$. For a stochastic process $(A_t)_{t=t_0}^\infty$ (or infinite sequence more generally) we will often simply write $(A_t)$, where $t_0$ will be understood from context.  We write $\Mspace{\Theta}$ for the set of probability distributions over $\Theta$. 
    We use $\R_{\geq 0}$ to be the set of nonnegative reals (similarly for $\R_{>0}$).
 When we say that $\nu\in\Mspace{\Theta}$ is a prior, it should be assumed that it is data-free, i.e., independent of the data $(Z_t)$. Writing, e.g.,  $\E_\dist[g(Z_i)]$ for some function $g$ should be taken to mean that the sequence $(Z_t)$ was drawn from $\dist$ but we are restricting ourselves to the $i$-th value. We may also write $\E_{Z_i}[g(Z_i)]$ in this case. 
    Finally, we let $\mu_t(\theta) = \E_\dist [f_t(Z_t,\theta)|\F_{t-1}]$.
 
	A \emph{forward filtration} is a sequence of $\sigma$-algebras $(\F_t)_{t=1}^\infty$ such that $\F_t\subset\F_{t+1}$ for all $t\geq 1$. If $\F_t=\sigma(Z_1,\dots,Z_t)$, we call $(\F_t)_{t=1}^\infty$ the canonical (forward) filtration. %In this case, $\F_0=\{\emptyset, \Omega\}$. 
 Intuitively, we conceive of $\F_t$ as all the information available at time $t$. Thus, if a function $f$ is $\F_t$-measurable, it may depend on data $Z_1,\dots,Z_t$, but not on any $Z_i$ for $i>t$. If a sequence of functions $(f_t)_{t=1}^\infty$ is such that $f_t$ is $\F_{t}$ measurable for all $t$, then we say that $(f_t)_{t=1}^\infty$ is \emph{adapted} to $\F_t$. If $f_{t+1}$ is $\F_{t}$ measurable for all $t$, then we say the sequence is \emph{predictable}.

	A \emph{martingale} adapted to the forward filtration $(\F_t)_{t=1}^\infty$ is a stochastic process $(S_t)_{t=1}^\infty$ such that $S_t$ is $\F_t$ measurable and $\E[S_{t+1}|\F_{t}]=S_{t}$ for all $t\geq 1.$
	If the equality is replaced with $\leq$ (resp., $\geq$) we call $(S_t)$ a \emph{supermartingale} (resp., \emph{submartingale}). Supermartingales are thus decreasing with time in expectation, whereas submartingales are increasing. Martingales stay constant in expectation. For this reason, they  often represent fair games.
	Forward filtrations are in contrast to reverse filtrations, which we cover later in this section. Henceforth, if we discuss filtrations unencumbered by a preceding adjective, then it is a forward filtration.

    It's perhaps worth remarking that a martingale is only a martingale \emph{with respect to} a particular measure $\Pr$. For instance, the process $S_t=\frac{1}{t}\sum_i X_i - m$ for i.i.d.\ $X_i$ is a martingale iff $\Pr(X_i)=m$. Formally then, one should refer to $(S_t)$ as (possibly) being a $\Pr$-martingale. However, in our case the measure will usually be clear from context and we will simply refer to martingales. The same discussion holds for sub/supermartingales.

	Supermartingales are natural tools to use when deriving anytime-valid bounds due to Ville's inequality~\citep{ville1939etude}, given in Lemma~\ref{lem:ville_supermartingales}. 
 Informally, Ville's inequality is a time-uniform version of Markov's inequality. 
 It states that a nonnegative supermartingale with initial value 1 remains small (say, less than $1/\delta$) at all times with probability roughly $1-\delta$. A digestible proof of Ville's inequality may be found in \citet{howard2020time}. 
	
	\begin{lemma}[Ville's Inequality for Nonnegative Supermartingales]
		\label{lem:ville_supermartingales}
		Let $(N_t)_{t=1}^\infty$ be a nonnegative supermartingale with respect to the filtration $(\F_t)_{t=1}^\infty$. For all times $t_0$ and  $u\in \R_{>0}$, 
  \[\Pr(\exists t\geq t_0: N_t\geq u)  \leq \frac{\E[N_{t_0}]}{u}.\]
	\end{lemma}
	 Ville's inequality can be restated as $\Pr(\forall t\geq t_0: N_t/N_{t_0}< u)\geq 1- 1/u$. Written this way, its power for providing time-uniform guarantees becomes evident. 
	
	Under appropriate conditions, mixtures of martingales remain martingales. That is, if $V_t(\theta)$ is a (sub/super) martingale, then $\E_{\theta\sim\rho}V_t(\theta)$ for well-behaved mixtures $\rho$ is also a (sub/super) martingale. The precise statement and corresponding proof are given in Appendix~\ref{app:mixtures}. This is useful because if we have a family of nonnegative supermartingales (say) of the form $N_t(\lambda)$ for $\lambda\in\R$, we can look for appropriate mixture distributions $F$ and conclude that $\int_{\lambda\in\R} N_t(\lambda)\d F(\lambda)$ is also a nonnegative supermartingale, and thus by Ville's inequality:
	\[\Pr\bigg(\forall t\geq t_0: \int_{\lambda\in\R} N_t(\lambda)\d F(\lambda) \leq 1/\delta\bigg)\geq 1-\delta.\]
	This has been called the ``method of mixtures'', and was noticed by \citet{wald1945sequential} and \citet{robbins1970statistical}.  Depending on the mixture distribution $F$, this bound can be more desirable than that based solely on $N_t(\lambda)$. Indeed, this approach has been successfully leveraged to generate time-uniform confidence intervals (i.e., confidence \emph{sequences})  \citep{howard2021time,waudby2021time,waudby2022anytime}. For our part, in Section~\ref{sec:light-losses} we give a novel PAC-Bayes bound using a Gaussian mixture distribution, as a demonstrative example.

	The machinery of nonnegative supermartingales (and their mixtures) in addition to Ville's inequality is sufficient to give time-uniform PAC-Bayes bounds in a wide variety of situations. Section~\ref{sec:supermart_bounds} is dedicated to this task.  
 See the first half of Table~\ref{tab:generalization} for those bounds which are recovered using this technique. 
However, to recover time-uniform versions of other well-known PAC-Bayes bounds, we must rely on reverse-time martingales. We introduce these next.

	A \emph{reverse filtration} $(\cR_t)_{t=1}^\infty$ is a sequence of $\sigma$-algebras such that $\cR_t\supset \cR_{t+1}$ for all $t$. That is, a reverse filtration represents decreasing information with time. A \emph{reverse martingale} $(S_t)$ adapted to a reverse filtration $(\cR_t)$ is a stochastic process such that $S_t$ is $\cR_t$ measurable and $\E[S_t|\cR_{t+1}] = S_{t+1}$ for all $t\geq 1$.
	Again, replacing the equality with $\leq$ (resp., $\geq$) results in reverse supermartingales (resp., submartingales).
	Reverse  processes are also called \emph{backwards} or \emph{reverse-time} processes. We will use such language interchangeably.
    An example of a reverse martingale is the empirical mean $\frac{1}{t}\sum_{i=1}^t Z_i$ adapted to the canonical reverse filtration $\cR_t=\sigma(Z_t,Z_{t+1},\dots)$. 
	Since filtrations and stochastic processes are typically considered in the context of ``increasing'' time, reverse-time processes can be initially confounding. 
	When thinking about reverse martingales, we encourage the reader to imagine time flowing backwards, i.e., information being revealed first at time $t$, then at time $t-1$, $t-2$ and so on. Thus, reverse submartingales are increasing in expectation in reverse-time and, if one were to plot the expected values such a process would resemble a \emph{super}martingale in forward time. With this insight in mind, it is  relieving to know that there is a variant of Ville's inequality for reverse submartingales.
  Proofs may be found in \citet{lee2019u,manole2021sequential}. 
	
	\begin{lemma}[Reverse Ville's Inequality]
		\label{lem:ville_submartingales}
		Let $(M_t)$ be a nonnegative reverse submartingale with respect to a reverse filtration $(\cR_t)_{t=1}^\infty$. For all $t_0$ and $u\in \R_{>0}$, 
  \[\Pr(\exists t\geq t_0: M_t\geq u)  \leq \frac{\E[M_{t_0}]}{u}.\]
	\end{lemma}

	Section~\ref{sec:submart_bounds} will employ reverse submartingales in order to give time-uniform PAC-Bayes bounds on convex functions $\vp$ of the expected and empirical risk. This will enable us to give time-uniform versions of inequalities presented by \citet{seeger2003bayesian,mcallester1998some,maurer2004note,germain2009pac,germain2015risk,tolstikhin2013pac}, among others. Finally, we present the change-of-measure inequality due to \citet{donsker1975large} which is central to the majority of existing PAC-Bayes bounds. 
 Before it is stated, let us recall that the 
 Kullback-Leibler (KL) divergence~\citep{kullback1951information}
 between two distributions $\mu$ and $\pi$ 
 in $\Mspace{\Theta}$ is 
 \begin{equation*}
		\kl(\mu\|\pi) = \E_{\theta\sim\mu}\bigg[\log\bigg(\frac{\d\mu(\theta)}{\d\pi}\bigg)\bigg] = \int_{\Theta} \log\bigg(\frac{\d\mu}{\d\pi}(\theta)\bigg)\mu(\d\theta),
	\end{equation*}
  if $\mu$ is absolutely continuous with respect to $\pi$ (i.e., $\mu(A)=0$ whenever $\pi(A)=0$), and $+\infty$ otherwise. Here $\frac{\d\mu}{\d\pi}$ is the Radon-Nikodym derivative. 
  As stated in the introduction, the utility of the KL divergence in PAC-Bayes bounds comes from the following the change of measure formula. This was first stated by~\citet{kullback1959information} for finite parameter spaces, and then proved more generally by~\citet{donsker1975large} and \citet{csiszar1975divergence}. 

	\begin{lemma}[Change of Measure]
		\label{lem:change-of-measure}
		Let $h:\Theta\to\R$ be a measurable function. For any $\nu\in\Mspace{\Theta}$,
		\begin{equation*}
			\log \E_{\theta\sim\nu} \exp(h(\theta)) = \sup_{\rho\in\Mspace{\Theta}}\big\{\E_{\theta\sim\rho}[h(\theta)] - \kl(\rho\|\nu)\big\}. 
		\end{equation*}
	\end{lemma}

While the Donsker-Varadhan formula is the most popular change of measure formula, it is not unique in its ability to furnish PAC-Bayes  bounds.  In Appendix~\ref{sec:phi-div}, we provide change of measure inequalities for $\phi$ and R\'{e}nyi divergences and discuss how we can use such formulas in our bounds. 
	
%    Note this is also the Legendre transform~\citep{catoni2004statistical}. 

\section{A General Recipe for Stochastic Processes}	
\label{sec:general-theorem}

We now present results for nonnegative processes upper bounded by either a supermartingale or a reverse submartingale. We will consider processes $P(\theta)=(P_t(\theta))_{t\geq 1}$ which are functions of a parameter $\theta\in\Theta$. While the following theorem does not appear to be in the form of a traditional PAC-Bayes bound, a variety of typical bounds can be recovered by considering particular processes $P(\theta)$ (Table~\ref{tab:generalization}). 
Many such fruitful processes will be presented throughout the remainder  of this manuscript.

\begin{theorem}[Master anytime PAC-Bayes bound]
\label{thm:pac-bayes-bounded-proc}
 For each $\theta\in\Theta$, assume that a stochastic process of interest, $P(\theta) = (P_t(\theta))_{t=t_0}^\infty$, is upper bounded by another process $U(\theta) = (U_t(\theta))_{t=t_0}^\infty$, which is such that  $\exp U(\theta)$ is either a supermartingale or a reverse submartingale satisfying $\E_\dist[\exp U_{t_0}(\theta)]\leq 1$. 
 %\edit{on a filtered probability space $(\Omega, \mathcal{D}, (\mathcal{G}_t), \Pr)$.} 
 Then for any $\delta\in(0,1)$ and prior $\nu\in\Mspace{\Theta}$, with probability at least $1-\delta$, we have that
 %\edit{with respect to $\mathcal{D}$}
 for all $t\geq t_0$ and $\rho\in\Mspace{\Theta}$,
\begin{equation}
\label{eq:anytime-bounded-proc}
    \Exp_\rho P_t(\theta) \leq \kl(\rho\|\nu) + \log(1/\delta).
\end{equation}
\edit{In fact, the bound also holds with $\rho$ being replaced by $\rho_t$ on both sides of \eqref{eq:anytime-bounded-proc} for any adapted sequence of posteriors $(\rho_t)_{t \geq t_0}$.}
\end{theorem}

Note that the KL divergence in \eqref{eq:anytime-bounded-proc} can be replaced by a variety of other divergences, provided they have their own variational representations (which they typically do). We discuss several alternative divergences in Sections~\ref{sec:beyond-kl} and \ref{sec:phi-div}.

\begin{proof} 
    For $t\geq t_0$, set 
    \[V_t^{\mix}:=\exp\sup_\rho\big\{ \E_{\theta\sim\rho}[U_t(\theta)] - \kl(\rho\|\nu)\big\}.\]
    If $\exp U(\theta)$ is a supermartingale (resp., reverse submartingale), then we claim $(V_t^\mix)$ is a supermartingale (resp., reverse submartingale). Indeed, Lemma~\ref{lem:change-of-measure} gives $V_t^\mix = \E_\nu \exp U_t(\theta)$, so $V_t^\mix$ is a mixture of supermartingales or reverse submartingales, which is itself a supermartingale or reverse submartingale (Lemma~\ref{lem:mixtures}). Applying Ville's inequality (either Lemma~\ref{lem:ville_supermartingales} or~\ref{lem:ville_submartingales}), we obtain 
    \begin{align*}
        & \Pr(\exists t\geq t_0: \exp\sup_\rho\big\{ \E_{\rho}P_t(\theta) - \kl(\rho\|\nu)\big\}\geq 1/\delta) \\
        &\leq \Pr(\exists t\geq t_0: \exp\sup_\rho\big\{ \E_{\rho}U_t (\theta) - \kl(\rho\|\nu)\big\}\geq 1/\delta)\\
        &= 
        \Pr(\exists t\geq t_0: V_t^{\mix} \geq 1/\delta) \leq \E_\dist[V_{t_0}^\mix]\delta\leq \delta,
    \end{align*}
    where the first inequality follows 
since $P_t(\theta)\leq U_t(\theta)$ by assumption. The final inequality follows since $\nu$ is data-free, enabling Fubini's theorem to be applied: $\E[V_{t_0}^\mix] = \E_\dist\E_\nu \exp U_{t_0}(\theta) = \E_\nu \E_\dist \exp U_{t_0}(\theta)\leq 1$.  
Thus, with probability $1-\delta$, for all  $t\geq t_0$, $\exp \sup_{\rho\in\Mspace{\Theta}}\{\E_{\rho}P_t(\theta) - \kl(\rho\|\nu)\} \leq 1/\delta.$
    Taking logarithms gives the desired result.  
\end{proof}

Several remarks are in order. 
\edit{First, the final sentence of Theorem~\ref{thm:pac-bayes-bounded-proc} highlights that the uniformity of time and probability measures implies that the bound holds over all \emph{sequences} of posteriors. This is the form in which we expect the result to be most useful. A concrete example of changing posteriors is given in Section~\ref{sec:gp}, where we apply our result to Gaussian process classification.
}
Second,  it's worth noting that Theorem~\ref{thm:pac-bayes-bounded-proc} posits no distributional assumptions on the underlying data. 
Indeed, it does not even assume that the underlying filtration is the canonical data filtration. 
While our examples in subsequent sections will use either the canonical forward filtration $\F_t=\sigma(Z^t)$ or a particular backward ``exchangeable'' filtration $(\mE_t)$, Theorem~\ref{thm:pac-bayes-bounded-proc} holds for more general processes. Third, we note also that we need not specify that $\rho$ be absolutely continuous with respect to the prior $\nu$ in inequality~\eqref{eq:anytime-bounded-proc} since, if not, then $\kl(\rho\|\nu)=\infty$ and the bound holds trivially. Finally, in addition to bounding $\E_\dist[V_1^\mix]$, the fact that the prior $\nu$ is data free is required by Lemma~\ref{lem:mixtures}. That is, it is required to ensure that $\E_\nu \exp U_t(\theta)$ is a super/submartingale.

\begin{table*}[h!]
    \small
    \centering
    \begin{tabular}{r|c|l}
         \emph{Condition on $(f_t)_{t\geq 1}$} & \emph{Condition on} $(Z_t)_{t\geq 1}$ & \emph{Results}  \\
         \hline 
         SubGaussian or subexponential & No explicit assumption & Corollaries \ref{cor:anytime-subGaussian-losses}, \ref{cor:mixture-subGaussian}\\
         Bounded & No explicit assumption & Corollaries~\ref{cor:bounded-bernstein},~\ref{cor:bounded-bennet}, \edit{\ref{cor:unexpected-bernstein}}\\
         Bernstein & No explicit assumption & Corollary~\ref{cor:anytime-bernstein-cond} \\
         Bounded MGF & No explicit assumption & Corollary~\ref{cor:cgf} \\
         $\E[f_t^2(Z_t,\theta)|\F_{t-1}]<\infty$ & No explicit assumption & Corollaries \ref{cor:anytime-second-moment}, \ref{cor:anytime-sn1} \\ 
         %$\E[|\Delta_t^2(\theta)|^p|\F_{t-1}]$ some $1<p\leq 2$ & No explicit assumption & Corollary~\ref{cor:bounded-p} \\
         Stn. \& MGF of $\vp_t(\theta)$  exists & Exchangeable & Corollaries \ref{cor:anytime-convex}, \ref{cor:anytime-convex-target}, \ref{cor:ipm-convex}, \ref{cor:renyi-convex}\\ 
         Stn. \& bounded in [0,1] & i.i.d.\ & Corollaries \ref{cor:anytime-seeger}, \ref{cor:anytime-mcallester}, \edit{\ref{cor:betting-style-mart}}, %\edit{\ref{cor:u-and-v-stats}, 
         \edit{\ref{cor:tolstikhin}, \ref{cor:anytime-gp}} 
    \end{tabular}
    \caption{A summary of the conditions on the loss and the data required by several bounds. \edit{``Stn'' stands for stationary.} 
    Even though for most rows there is no explicit dependence assumption required of $(Z_t)$, the usefulness of the bounds or the establishment of conditions on $(f_t)$ may sometimes require implicitly making distributional assumptions on the data, but these will often be (much) less restrictive than an i.i.d.\ assumption. See Section~\ref{sec:light-losses} after Corollary~\ref{cor:anytime-subGaussian-losses} for more discussion. 
    As all results require $(f_t)$ to be predictable, this requirement is disregarded above. 
    We omit results from Section~\ref{sec:martingale-differences} (martingale difference sequences) as the setting is slightly different.}  
    \label{tab:conditions}
    \vspace{-1em}
\end{table*}

\section{PAC-Bayes Bounds via Supermartingales}
\label{sec:supermart_bounds}
We first construct PAC-Bayes bounds via supermartingales in light of Theorem~\ref{thm:pac-bayes-bounded-proc}. Our general framework for doing so is based on \emph{sub-$\psi$-processes}~\citep{howard2020time}, which are generalizations of processes amenable to exponential concentration inequalities. 
Many standard concentration inequalities (e.g., Hoeffding, Bennett, Bernstein) implicitly use sub-$\psi$ processes  which, if identified, yield time-uniform Chernoff bounds~\citep{howard2020time}. For our purposes, sub-$\psi$ processes can be used in Theorem~\ref{thm:pac-bayes-bounded-proc} to yield a time-uniform PAC-Bayes bound (Corollary \ref{cor:anytime-pac-bayes-sub-psi}). Many existing PAC-Bayes bounds rely on fixed-time concentration inequalities which can be generalized to sub-$\psi$ processes, thus yielding time-uniform extensions. 
We begin by defining sub-$\psi$ processes and then proceed to give explicit bounds for light-tailed losses (Section~\ref{sec:light-losses}), and then for heavier-tailed losses (Section~\ref{sec:heavy-losses}).

\subsection{The sub-$\psi$ Condition}
\label{sec:sub-psi-condition}

Roughly speaking, a sub-$\psi$ process is a stochastic process which is upper bounded by a supermartingale but takes a particular functional form. They are at the heart of recent progress on time-uniform Chernoff bounds~\citep{howard2020time}. 
 	This section presents a corollary of Theorem~\ref{thm:pac-bayes-bounded-proc} for sub-$\psi$ processes which, in turn, yields many time-uniform extensions of existing PAC-Bayes bounds. We find that many existing bounds are implicitly relying on sub-$\psi$ processes without recognizing it.

	\begin{definition}[Sub-$\psi$ process]
		\label{def:sub-psi}
		Let $(S_t)_{t=1}^\infty\subset \R$ and $(V_t)_{t=1}^\infty\subset \R_{\geq 0}$ be stochastic processes adapted to an underlying filtration $(\F_t)_{t=1}^\infty$. For a function $\psi:[0,\psi_{\text{max}})\to \R$, we say $(S_t,V_t)$ is a sub-$\psi$ process if, for every $\lambda\in[0,\psi_{\text{max}})$, there exists some supermartingale $(L_t(\lambda))_{t=1}^\infty$ with $L_1(\lambda)\leq 1$ such that, for all $t\geq 1$,
		\begin{equation}
			\exp\{\lambda S_t - \psi(\lambda) V_t\}\leq L_t(\lambda),\; \text{a.s.}
		\end{equation}
	\end{definition}
	
	Definition~\ref{def:sub-psi} may appear rather abstract at first glance. Useful intuition comes from considering what happens when $(S_t)$ is a martingale. In this case, $(\exp(\lambda S_t))$ is a submartingale by Jensen's inequality. Thus, $\psi(\lambda)V_t$ must be a process which appropriately ``dominates'' $S_t$ in order to ensure that $\exp(\lambda S_t - \psi(\lambda) V_t)$ decreases in expectation rather than increases. 
	For instance, suppose $X_1,X_2,\dots$ are i.i.d. with mean 0. If $S_t=\sum_{i\leq t}X_t$, then taking $\psi(\lambda)$ to be the log-MGF $\log \E e^{\lambda X_1}$ and $V_t=t$ is sufficient to turn $\exp(\lambda S_t - \psi(\lambda)V_t)$ into a martingale. 
 Indeed, $\E[\exp(\lambda S_t -\psi(\lambda)V_t)|\F_{t-1}] = \prod_{i=1}^t \E[\exp(\lambda X_i - \log \E e^{\lambda X_1})|\F_{t-1}] = \prod_{i=1}^{t-1} \exp(\lambda X_i - \log\E e^{\lambda X_1})$. Corollary~\ref{cor:cgf} gives a PAC-Bayes bound based on this process. Another example comes from supposing the $X_i$ are $\sigma$-subGaussian. In that case we may take $\psi(\lambda)=\lambda^2\sigma^2/2$, keeping $S_t$ and $V_t$ the same. Then $\exp(\lambda S_t-\psi(\lambda )V_t)$ is a supermartingale (as opposed to a martingale). 
 This process is used (albeit in more generality) by Corollary~\ref{cor:anytime-subGaussian-losses}. 
 If, as in the examples above, $S_t$ is a sum then we may let $\lambda=\lambda_t$ change as a function of time. This will be the case in the majority of our bounds. 
 Finally, notice that in these examples, we may simply take $L_t(\lambda)=\exp(\lambda S_t-\psi(\lambda) V_t)$, meaning that the exponential process is itself a supermartingale. This is often the case. 
	We refer the reader to \citet{howard2020time} for a more lengthy discussion and further examples. 
	
	A nonnegative process that is upper bounded by 
 a supermartingale (but may or may not itself be a supermartingale) has recently been termed an ``e-process''~\citep{ramdas2022game}.  
 Theorem~\ref{thm:pac-bayes-bounded-proc} yields bounds for such processes. 
	Instead of working with  more general definitions, however, we prefer to base our discussion on sub-$\psi$ processes specifically because it's helpful to consider particular functions $\psi$ and processes $(V_t)$ which can bound our process $(S_t)$ of interest. More to the point, we will often consider $S_t$ to be the martingale 
 \begin{equation}
 \label{eq:example_st}
     \sum_{i=1}^t \E_\dist[f_i(Z,\theta)|\F_{i-1}] - f_i(Z_i,\theta).
 \end{equation}
	Different assumptions on $f_t$ (e.g., bounded, light-tailed, heavy-tailed) will then lead us to particular selections of $\psi$ and $(V_t)$. Moreover, our PAC-Bayes inequalities will bound $S_t$ in terms of $\psi$ and $V_t$. Consequently, if one finds themselves dealing with a sub-$\psi$ process, then the form of the bound will be immediately apparent. 

    As we did for more general processes, we will consider sub-$\psi$ processes which are indexed by parameters $\theta\in\Theta$ and we will write that $(S_t(\theta),V_t(\theta))$ is a sub-$\psi$ process. This should be taken to mean that, for each fixed $\theta$, $\exp\{\lambda S_t(\theta)-\psi(\lambda)V_t(\theta)\}\leq L_t(\lambda,\theta)$ for an appropriate supermartingale $L_t(\lambda,\theta)$. 
    Since, by construction, sub-$\psi$ processes are nonnegative and upper bounded by a supermartingale with unit initial value, we obtain the following corollary of Theorem~\ref{thm:pac-bayes-bounded-proc}. 

    \begin{corollary}
    \label{cor:anytime-pac-bayes-sub-psi}
    Assume that for each $\theta\in\Theta$, $(S_t(\theta),V_t(\theta))$ is a sub-$\psi$ process. Let $\nu\in\Mspace{\Theta}$ be a data-free prior and let $\lambda\in[0,\psi_{\text{max}})$. Then for any $\delta\in(0,1)$, with probability at least $1-\delta$, we have that
		\begin{equation}
			\Exp_{\theta\sim\rho}[\lambda S_t(\theta) - \psi(\lambda) V_t(\theta)] \leq \kl(\rho\|\nu) + \log(1/\delta),
		\end{equation}
    for all times $t\geq 1$ and $\rho\in\Mspace{\Theta}$. 
    \end{corollary}

% \subsection{A preliminary note on time-dependent $\lambda$s}

\subsection{Light-tailed losses}
\label{sec:light-losses}
	
	Here we return to our main problem setting and consider anytime bounds on the difference between the expected risk and the empirical risk. By choosing particular sub-$\psi$ processes and applying Corollary~\ref{cor:anytime-pac-bayes-sub-psi}, we can develop anytime bounds for light-tailed losses (this section) and more general losses (Section~\ref{sec:heavy-losses}). 
	It will often be useful to consider the quantity
	\begin{equation*}
		\label{eq:delta}
		\Delta_i(\theta) := \mu_i(\theta)  - f_i(Z_i,\theta),
	\end{equation*}
	where $\mu_i(\theta)=\E_\dist [f_i(Z_i,\theta)|\F_{i-1}]$.
	Note that the process $(\sum_{i\leq t}\Delta_i(\theta))_{t \geq 1}$ is a martingale, but it is not nonnegative. Throughout the remainder of this section, the underlying filtration will be the canonical data filtration $\F_t=\sigma(Z_1,\dots,Z_t)$.

\subsubsection{SubGaussian losses} %\paragraph{SubGaussian and Bounded Losses.}
We begin by giving an anytime-valid PAC-Bayes bound for subGaussian losses. 
Recall that a random variable $Y$ is $\sigma$-\emph{subGaussian conditional on} $\F$ if  $\E[\exp(s(Y-\E[Y]))|\F] \leq \exp(s^2\sigma^2/2)$ for all $s\in\R$.  
We will say the loss $f_t$ is $\sigma$-subGaussian if $f_t(Z_t,\theta)$ is $\sigma$-subGaussian for all $\theta\in\Theta$. 

    \begin{corollary}
		\label{cor:anytime-subGaussian-losses}
		Let $(Z_t)$ be a stream of (not necessarily i.i.d.) data. 
		Let $(f_t)_{t=1}^\infty$ be a predictable sequence of loss functions such that $f_i$ is $\sigma_i$-subGaussian conditional on $\F_{i-1}$. Let $(\lambda_t)$ be a nonnegative predictable sequence and consider any data-free prior $\nu\in\Mspace{\Theta}$. Then, for all $\delta\in(0,1)$, with probability at least $1-\delta$ \edit{over the random draw of $(Z_t)$}, for all $t$ and measures $\rho\in\Mspace{\Theta}$,
		\begin{align}
  \label{eq:subgaussian-bound}	\sum_{i=1}^t\lambda_i \E_\rho\Delta_i(\theta) \leq 
    \sum_{i=1}^t\frac{\lambda_i^2\sigma_i^2}{2}  + \kl(\rho\|\nu) + \log(1/\delta).  \notag 
		\end{align}
	\end{corollary}
	The proof is in Appendix~\ref{app:proof-anytime-subGaussian}. 
	Suppose the loss is stationary and bounded in $[0,H]$, implying that it is $H/2$-subGaussian. 
 If $\lambda_i=\lambda$ is constant, then Corollary~\ref{cor:anytime-subGaussian-losses} implies that with probability at least $1-\delta$,
	\begin{equation}
		\label{eq:bounded_loss_stationary}
		\E_\rho\E_\dist f(Z,\theta) \leq \E_\rho \hR_t(\theta)+ \frac{\lambda H^2}{8}  + \frac{\kl(\rho\|\nu) + \log(1/\delta)}{\lambda t}, 
	\end{equation}
	for all times $t$ and measures $\rho\in\Mspace{\Theta}$. 
 For any fixed time $n$ of special interest, setting $\lambda_i=\lambda/n$ for all $\lambda$ recovers~\eqref{eq:catoni-bounded} (Catoni's bound) exactly at time $n$, but still makes a nontrivial claim for all $t \neq n$ at no extra cost.  
	This time-uniform bound for bounded losses was recently also given by \citet{haddouche2022pac}. 
 As noted previously, it generalizes well-known fixed-time bounds of the same flavour \citep{catoni2003pac,catoni2004statistical,catoni2007pac,alquier2016properties}. 
 This phenomenon of exactly recovering a fixed-time Chernoff-style bound by a more general time-uniform bound was a central contribution of the unified ``supermartingale + Ville'' framework of~\cite{howard2020time}.

 \begin{remark}
\label{remark:lambdas}
As in Corollary~\ref{cor:anytime-subGaussian-losses},
the remainder of Section~\ref{sec:supermart_bounds} is concerned with bounding the conditional risk $\frac{1}{t}\sum_{i=1}^t \E_{\theta\sim\rho}\mu_i(\theta)$ where $\mu_i(\theta) = \E[f_i(Z_i,\theta)|\F_{i-1}]$.  However, we will often state bounds on $\sum_{i=1}^t \lambda_i \E_{\theta\sim\rho}\mu_i(\theta)$, where $(\lambda_t)_{t\geq 1}$ is a predictable sequence of positive scalars. Considering such sequences is useful if the conditional risk is constant as a function of $t$, i.e., $\risk(\theta) = \mu_t(\theta)$ for all $t$, as we can then remove $\risk(\theta)$ from the sum and divide by $\sum_{i=1}^t \lambda_i$. Values of $\lambda_t$ can be chosen such that difference between $\E_{\rho}\risk(\theta)$ and $t^{-1}\sum_{i\leq t}\E_\rho f_i(Z_i,\theta)$ --- the \emph{width} of the bounds --- goes asymptotically to zero with $t$. This has been called the method of ``predictable plug-ins'' (see, e.g., \citet{waudby2023estimating}). 

On the other hand, if $\mu_t(\theta)$ is changing with time, then we must select $\lambda_i=\lambda$ to be constant in order to isolate the mean. In this case, we can still achieve bounds with widths that go to zero, but via a different (and more complicated) method of applying different bounds over geometrically spaced epochs. We provide details in Section~\ref{sec:confseq}. 
Otherwise, if such a method is not used, one may still select $\lambda$ as a function of some fixed-time $n$, in which case the bound will be tight at that point but progressively looser as the number of samples $t$ moves away from $n$ (the bound will remain valid at all times, however). 
\end{remark}

	The lack of independence assumptions in Corollary~\ref{cor:anytime-subGaussian-losses} may seem surprising at first, but it is another consequence of the supermartingale approach. The proof of the Corollary is based on the process
	\begin{equation}
    \label{eq:subgaussian-supermart}
		N_t(\theta) := \prod_{i=1}^t \exp\bigg\{\lambda_i (\mu_i(\theta) - f(Z_i,\theta)) - \frac{\lambda_i\sigma_i^2}{2}\bigg\}. 
	\end{equation} 
	Since $N_{t-1}(\theta)$ is $\F_{t-1}$ measurable, 
	\begin{align*}		
 \E[N_t(\theta)|\F_{t-1}] = N_{t-1}(\theta) \cdot \E \exp\bigg\{\lambda_t (\mu_t(\theta) - f_t(Z_t,\theta)) - \frac{\lambda_t\sigma_t^2}{2}\bigg|\F_{t-1}\bigg\}. 
	\end{align*} 
By definition of (conditional) subGaussianity, the expected value term in the above display is at most 1. This demonstrates that $(N_t(\theta))$ is a nonnegative supermartingale, meaning that Theorem~\ref{thm:pac-bayes-bounded-proc} applies. The same reasoning holds for other bounds we will present: if $\exp\{\lambda_t\Delta_t(\theta) - g_t|\F_{t-1}\}$ has expectation at most 1, then $(\exp\{\sum_i\lambda_i\Delta_i(\theta) - \sum_ig_i\})_{t\geq 1}$ is a supermartingale, yielding a time-uniform PAC-Bayes bound with no independence assumptions on the data. However, we feel it important to emphasize that there is no free lunch. Despite there being no such assumptions, the fact that we must have $\E[f_t(Z_t,\theta)|\F_{t-1}]<\infty$ is implicitly relying on a type of dependence between $f_t$ and the past. In some sense, the lack of distributional assumptions places the burden on $(f_t)$ as opposed to $(Z_t)$. Thus, while the mathematics holds with no conditions on $(Z_t)$, the bounds may be meaningless for  very ``ill-behaved'' data and/or losses.

 We now present a novel bound for subGaussian losses based on the method of mixtures. Fix $\lambda_i = \lambda$ above to consider the supermartingale $M_t(\lambda,\theta) := \prod_{i=1}^t \exp\big\{\lambda \Delta_i(\theta) - \frac{\lambda^2}{2}\sigma_i^2\big\}.$
 As discussed in Section~\ref{sec:background} and proven in Appendix~\ref{app:mixtures}, the mixture 
	\begin{equation}
		\label{eq:nsm-mixture}
		M_t(\theta) := \int_{\lambda\in\R} M_t(\lambda,\theta)\d F(\lambda),  
	\end{equation}
	is also a nonnegative supermartingale for an appropriate distribution $F$. By choosing $F$ to be Gaussian with mean 0 and some fixed variance, we can generate the following bound. The proof is in Appendix~\ref{app:proof-mixture-subGaussian}.

	\begin{corollary}[Gaussian-mixture bound for subGaussian losses]
		\label{cor:mixture-subGaussian}
  %Let $(Z_t)$, $(f_t)$ and $\nu$ be as in Prop.~\ref{cor:anytime-subGaussian-losses}. 
		Let $Z_1,Z_2,\dots$ be a stream of (not necessarily i.i.d.) data. 
		Let $(f_t)_{t=1}^\infty$ be a predictable sequence of loss functions such that $f_i$ is $\sigma_i$-subGaussian. Let $\nu\in\Mspace{\Theta}$ be a data-free prior. 
  Then, for all $\delta\in(0,1)$ and $\beta>0$, with probability at least $1-\delta$ \edit{over the random draw of $(Z_t)$}, for all times $t$ and measures $\rho\in\Mspace{\Theta}$,
		\begin{align}
			\label{eq:mixture-subGaussian}
			&\sum_{i=1}^t \E_\rho \Delta_i(\theta) 
   \leq \bigg(\frac{s_t(\beta)}{\beta}\bigg(\kl(\rho\|\nu) 
   + \log\frac{s_t(\beta)}{\delta}\bigg)\bigg)^{1/2},
		\end{align}
		where $s_t(\beta)=1+\beta\sum_{i=1}^t \sigma_i^2$. 
	\end{corollary}
	
	The parameter $\beta$ comes from the variance of the Gaussian mixture in \eqref{eq:nsm-mixture}. %Since $M_t(\theta)$ is a supermartingale for $F=N(0,\beta)$ for all $\beta>0$, we may take the infimum over all such $\beta$. 
 It is worth comparing the above bound to the one from~\citet{mcallester1999pac}. 
Considering stationary loss functions bounded in $[0,1]$, McAllester's fixed time bound reads 
\begin{equation}
 \label{eq:mcallester}
		\E_\rho \risk(\theta) \leq \E_\rho \hR_n(\theta) + \bigg(\frac{\kl(\rho\|\nu) + \log(n/\delta)}{2(n-1)}\bigg)^{1/2}.
	\end{equation}
In our case, $f$ being bounded implies that $\sigma_i^2=1/4$ for all $i$ since $f$ is $1/2$-subGaussian. 
Fix a time $n$ of interest and take $\beta$ such that $s_n(\beta) = n$, i.e., $\beta = 4(n-1)/n$. The Gaussian mixture bound~\eqref{eq:mixture-subGaussian} then yields McAllester's bound, but tighter by a factor of $\sqrt{2}$. Meanwhile, we can achieve a time-uniform version of McAllester's bound by considering $\beta=1$, in which case $s_t(\beta)=1+t/4 \leq t$ for all $t\geq 2$ and the bound becomes 
	\begin{equation}
	\label{eq:mixture-subGaussian-mcal}
		 \E_\rho\risk(\theta) \leq \E_\rho\hR_t(\theta) + \bigg(\frac{\kl(\rho\|\nu) + \log(t/\delta)}{t}\bigg)^{1/2},  
	\end{equation}
 which is looser than McAllester's by a factor of $\sqrt{2}$. We might thus consider \eqref{eq:mixture-subGaussian-mcal} to be a time-uniform generalization of McAllester's bound. However, this was for a particular choice of $\beta$. In general, our bound contains the parameter $\beta$ over which we can optimize. Performing this optimization gives an implicit equation for $\beta$: \[\log(s_t(\beta)) + \frac{1}{\beta} = \log(\delta) - \kl(\rho\|\nu).\]
 (Though note that the result should not depend on $t$ unless it is fixed in advance.) 
This is difficult to solve in closed-form, but after choosing $\nu$ and $\rho$ and computing the KL divergence, we might generate an approximate solution computationally. 
 Section~\ref{sec:submart_bounds} will explore another generalization of McAllester's bound using a separate (reverse submartingale based) technique \edit{and Section~\ref{sec:wealth_proc} will discuss yet another generalization using betting martingales.}

  \begin{remark}
  \label{rem:subexp}
Corollaries~\ref{cor:anytime-subGaussian-losses} and \ref{cor:mixture-subGaussian} may be strengthened to handle \emph{sub-exponential} losses, where we say that $Y$ is subexponential with parameters $(\sigma,c)$ if $\E\exp(s(Y-\E Y)) \leq \exp(s^2\sigma^2/2)$ for all $|s|\leq 1/c$. SubGaussian random variables are subexponential random variables with $c=0$. To extend Corollary~\ref{cor:anytime-subGaussian-losses} to subexponential variables, we take $\lambda_i\leq 1/c_i$ if $f_i$ is subexponential with parameter $(\sigma_i,c_i)$. 
 \end{remark}

 \edit{
 It is worth noting here that the method of mixtures has been previously employed in the PAC-Bayesian literature. For instance, it was used by \citet{kuzborskij2019efron}, who were interested in providing bounds on $h(Z^n,\theta) - \E_\dist [h(Z^n,\theta)]$, where $h:\cZ^n\times\Theta\to\R$ is a measurable function and $n\in\mathbb{N}$. Here $Z^n=(Z_1,\dots,Z_n)$, where we assume all elements $Z_i$ are drawn i.i.d.\ from $\dist$. \citet{kuzborskij2019efron} give bounds based on an Efron-Stein variance proxy: 
\begin{equation}
    V(\theta) = \sum_{i=1}^n V_i(\theta),\text{~ where ~} V_i(\theta) = \E[(h(Z^n,\theta) - h(Z^{(i)},\theta))^2|\F_{i}],
\end{equation}
where $Z^{(i)}$ is the same as $Z^n$ but contains an independent copy of $Z_i$. Equations (17) and (18) in \citet{kuzborskij2019efron} show that the Doob-decomposition of $h(Z^n,\theta) - \E[h(Z^n,\theta)]$ for a fixed $\theta$ obeys a sub-$\psi$ condition with respect to $\{V_t(\theta)\}$. Denoting $D_i(\theta) = \E[h(Z^n,\theta)|\F_i] - \E[h(Z^n,\theta)|\F_{i-1}]$, they showed that 
\[
\E\bigg[\exp\bigg(\lambda D_i(\theta) - \frac{\lambda^2}{2} V_i(\theta)\bigg)\bigg|\F_{i-1}\bigg] \leq 1; \quad 1\leq i\leq n.
\]
The processes $\{(P_t(\theta))_{t\geq 1}: \theta\in \Theta\}$ where $P_t(\theta) = \prod_{i\leq t} \exp\{\lambda D_i(\theta) - \frac{\lambda^2}{2}V_i(\theta)\}$ thus constitute a family of supermartingales, and mixing over the family yields another supermartingale. Theorems 3 and 4 of \citet{kuzborskij2019efron} are based on such a mixture (again using a Gaussian mixture distribution), and provide bounds on $h(Z^n,\theta) - \E_\dist[h(Z^n,\theta)]$ in terms of $V(\theta)$. 
Their result thus follows as a consequence of Corollary~\ref{cor:anytime-pac-bayes-sub-psi} and the method of mixtures.  
Note, however, that time-uniformity apparently does not gain us much in this case because the empirical risk (in this case $h(Z^n,\theta)$) is not computable until time $t=n$.  
 }

\subsubsection{Losses obeying a Bernstein condition.}
The consideration of subexponential random variables in Remark~\ref{rem:subexp} naturally leads us to consider a Bernstein condition on the losses, which implies that they're subexponential. In particular, we say that a random variable $Y$ satisfies \emph{Bernstein's condition} with parameter $c$ if  
\begin{equation*}
    |\E[(Y-\E(Y))^k]| \leq \frac{1}{2}\Var(Y)k!c^{k-2},\quad \forall k\in\N,\;k\geq 2.
\end{equation*}
It is well known that if $Y$ is Bernstein with parameter $c$ then it is subexponential with parameters $(\sqrt{2\Var(Y)},1/2c)$
(see, e.g., \citealp[Theorem 2.10]{boucheron2013concentration} or \citealp[Corollary 2.10]{wainwright2019high}). For bounded random variables, the resulting concentration inequality can be much tighter than Hoeffding's (which is not variance adaptive), especially when the variance of $Y$ is much smaller than its range. It is therefore worth stating the following PAC-Bayes result for Bernstein-type losses, the proof of which is in Appendix~\ref{app:proof-bernstein}.  

\begin{corollary}[Bernstein condition anytime bound]
\label{cor:anytime-bernstein-cond}
Let $(Z_t)$ be a stream of (not necessarily i.i.d.) data. 
		Let $(f_t)$ be a predictable sequence of functions with $\Var(f_t(Z_t,\theta)|\F_{t-1})\leq \sigma_t^2$ and, for all $t$ and integers $k\geq 2$,  $|\E_\dist[(f_t(Z_t,\theta)-\mu_t(\theta))^k|\F_{t-1}]| \leq \frac{1}{2}\sigma_t^2k!c_t^{k-2}$. 
  Let $(\lambda_t)$ be a predictable sequence such that $\lambda_t\in(0,1/c_t)$ for all $t$. Fix a prior $\nu\in\Mspace{\Theta}$. Then, for all $\delta\in(0,1)$, with probability at least $1-\delta$ \edit{over the random draw of $(Z_t)$}, for all times $t\geq 1$ and measures $\rho\in\Mspace{\Theta}$, we have
  \begin{equation*}
      \sum_{i=1}^t \lambda_i \E_\rho \Delta_i(\theta)  \leq 
      \sum_{i=1}^t \frac{\lambda_i^2\sigma_i^2}{2(1-c_i\lambda_i)}
       + \kl(\rho\|\nu) + \log(1/\delta).
  \end{equation*}  
\end{corollary}

To our knowledge, this result (even the implied fixed-time counterpart) is new to the PAC-Bayes literature.

\subsubsection{Bounded losses}
\label{sec:bounded}

The next few results consider bounded loss functions. 
The first relies on a Bernstein-type process 
    \begin{equation}
    \label{eq:bernstein-process}
        B_t(\theta) := \prod_{i=1}^t \exp\bigg\{ \lambda_i \Delta_i(\theta) - \lambda_i^2(e-2) \E[\Delta_i^2(\theta)|\F_{i-1}]\bigg\}.
    \end{equation}
    It is so termed because $(B_t(\theta))$  can be seen to be a supermartingale via the application of Bernstein's inequality. 
    The details are in the proof of the following Proposition, which can be found in Appendix~\ref{app:proof-bounded-bernstein}.
    The resulting bound has been applied to martingale difference sequences (MDSs)~\citep[Theorem 7]{seldin2012pac}. Corollary~\ref{cor:mds-bernstein} gives the precise time-uniform extension of the MDS result.  

    \begin{corollary}[Bernstein-like anytime bound for bounded losses]
    \label{cor:bounded-bernstein}
        Let $(Z_t)$ be a stream of (not necessarily i.i.d.) data. 
		Let $(f_t)$ be a predictable sequence of loss functions such that $\|f_t\|_\infty \leq H_t$ for all $t$ and constants $H_t>0$. Let $(\lambda_t)$ be a predictable sequence such that $\lambda_t\in[0,1/H_t]$ for all $t$. Fix a prior $\nu\in\Mspace{\Theta}$. Then, for all $\delta\in(0,1)$, with probability at least $1-\delta$ \edit{over the random draw of $(Z_t)$}, for all times $t$ and measures $\rho\in\Mspace{\Theta}$, we have
  \begin{equation*}
      \sum_{i=1}^t \lambda_i \E_\rho \Delta_i(\theta)  \leq  (e-2) \sum_{i=1}^t \lambda_i^2 \E_\rho\E_\dist[\Delta_i^2(\theta)|\F_{i-1}] + \kl(\rho\|\nu) + \log(1/\delta).
  \end{equation*}
    \end{corollary}
	
A second result for bounded losses can be obtained via a supermartingale based on a Bennett-like inequality~\citep[Theorem 2.9]{boucheron2013concentration}. It is the first example in this paper of a result where the empirical risk is bounded by the expected risk. That is, it is a bound on $-\Delta_i(\theta)$. The reader can be forgiven for wondering whether such bounds are useful. However, Catoni's MGF-based PAC-Bayes bound \citep{catoni2007pac} is also an example of such a bound and has found various uses, such as in estimating means of random vectors and matrices~\citep{catoni2017dimension}. We therefore opt to include the next result. More discussion can be found in the next section when we present a time-uniform extension of Catoni's bound (Corollary~\ref{cor:cgf}). 

\begin{corollary}[Bennet-like anytime bound for bounded losses]
\label{cor:bounded-bennet}
    Let $(Z_t)$ be a stream of (not necessarily i.i.d.) data. 
		Let $(f_t)$ be a predictable sequence of loss functions such that $\|f_t\|_\infty \leq H_t$ for all $t$ and constants $H_t>0$. Let $(\lambda_t)$ be a predictable sequence of positive values with $\lambda_t < \inf_\theta \{1/\E_\dist[f_t(Z_t,\theta)|\F_{t-1}]\}$. Fix a prior $\nu\in\Mspace{\Theta}$. Then, for all $\delta\in(0,1)$, with probability at least $1-\delta$ \edit{over the random draw of $(Z_t)$}, for all times $t$ and measures $\rho\in\Mspace{\Theta}$, we have
  \begin{equation*}
      \sum_{i=1}^t \lambda_i\E_\rho (f_i(Z_i,\theta) - \mu_i(\theta)) \leq \sum_{i=1}^t \frac{\E_{\rho,\dist} [f_i^2(Z_i,\theta)|\F_{t-1}]}{H_i^2}\psi_P(\lambda_i H_i) + \kl(\rho\|\nu) + \log(1/\delta),
  \end{equation*}
  where $\psi_P(x) = (e^{x} - x -1)$. 
\end{corollary}

% \hrmk{super minor point: why not $-\Delta_i$ on the left hand side? --- actually they can be $+\Delta_i$ right? otherwise we might want to emphasize it is another bound with emp risk ``on the wrong side"}

The proof is in Appendix~\ref{app:proof-bounded-bennett}. 
Both Corollary~\ref{cor:bounded-bernstein} and \ref{cor:bounded-bennet} are based on one-sided concentration inequalities and thus hold in the more general setting when losses are not nonnegative. The subscript in $\psi_P$ references sub-Poisson processes~\citep{howard2020time}. 

\edit{
Our final result for bounded losses comes via an ``unexpected Bernstein inequality'' provided by \citet[Lemma 13]{mhammedi2019pac} and based on an inequality in \citet{fan2015exponential}. 
More specifically, 
\citet[Equation (4.11)]{fan2015exponential} demonstrate that for random variables $X\geq -1$ and $\lambda\in[0,1)$, 
\[\E\exp\{\lambda X + (\lambda + \log(1-\lambda))X^2\}\leq 1.\]
%$\log(1 + \lambda x) \geq \lambda x + x^2(\lambda + \log(1 - \lambda))$ for $x\geq -1$ and $\lambda\in[0,1)$. 
This bound was extended to derive empirical Bernstein concentration inequalities and confidence sequences in~\cite{howard2021time}.
Following this lead, \citet{mhammedi2019pac} use Fan's inequality to show that for a random variable $X\in (-\infty,b)$, for all $0\leq \lambda<1/b$, 
\begin{equation}
\label{eq:unexpected-bernstein}
  \exp\{\lambda(\E[X] - X - cX^2)\}\leq 1, \quad \forall c \geq \lambda \vartheta(\lambda b).  
\end{equation}
where $\vartheta(\alpha) = \frac{-\log(1-\alpha) - \alpha}{\alpha^2}$. We can use such an inequality to construct a supermartingale which furnishes the following result. 

\begin{corollary}[Unexpected Bernstein anytime bound for bounded losses]
\label{cor:unexpected-bernstein}
    Let $(Z_t)$ be a stream of (not necessarily i.i.d.) data. 
		Let $(f_t)$ be predictable sequence of loss functions such that $\|f_t\|_\infty \leq H_t$ for all $t$ and constants $H_t>0$. Let $(\lambda_t)$ be a predictable sequence of positive values such that $0\leq  \lambda_t\leq 1/H_t$. Let $(c_t)$ be a predictable sequence with $c_t \geq \lambda_t\vartheta(\lambda_tH_t)$. Fix a prior $\nu\in\Mspace{\Theta}$. Then, for all $\delta\in(0,1)$, with probability at least $1-\delta$ over the random draw of $(Z_t)$, for all times $t$ and measures $\rho\in\Mspace{\Theta}$, we have
  \begin{equation*}
      \sum_{i=1}^t \lambda_i\E_\rho \Delta_i(\theta) \leq \sum_{i=1}^t \lambda_i c_i\E_\rho f_i^2(Z_i,\theta) + \kl(\rho\|\nu) + \log(1/\delta).
  \end{equation*}
\end{corollary}
A big distinction between Corollary~\ref{cor:unexpected-bernstein} and Corollaries~\ref{cor:bounded-bernstein} and \ref{cor:bounded-bennet} is the lack of an expectation over $\dist$ on the right hand side. Instead, we work directly with the random variables $f_t(Z_t,\theta)$.  
While the process used in the proof of Corollary~\ref{cor:unexpected-bernstein} is at the core of the result of \citet{mhammedi2019pac}, our result is not a time-uniform version of theirs. Indeed, they employ several tools which make an anytime-valid extension challenging, such as the use of data-dependent priors. We discuss such priors more in Section~\ref{sec:data-dependent-priors}.

}

\bigskip

\edit{

\subsubsection{Interlude: Implicit Bounds via Wealth Processes}
\label{sec:wealth_proc}

There has been a recent surge of interest in so-called game-theoretic probability and statistics~\citep{shafer2019game,ramdas2022game} owing to its fresh perspective on sequential, anytime-valid inference. Here we demonstrate how some of the ideas may be employed to generate PAC-Bayes bounds, and how this perspective recovers some concurrent work by \citet{jang2023tighter}. 

Central to game-theoretic statistics is the idea of a fictitious bettor playing an iterated game against nature. The game is structured as follows. The bettor begins with an initial wealth of $\K_0=1$. At time $t$, the bettor chooses a $\F_{t-1}$-measurable \emph{payoff function} $S_t : \mathcal{V}\to[0,\infty]$ obeying $\E[S_t(V)|\F_{t-1}]\leq 1$, where the expectation is taken with respect to some strategically chosen distribution(s) $P$. For instance, in hypothesis testing problems, $P$ is chosen to be the set of distributions comprising the null. See \citet{ramdas2022game} for more details. Nature then reveals a value $V_t\in \mathcal{V}$ and the bettor updates his wealth as $\K_t = \K_{t-1}\cdot S_t(V_t)$. The total wealth of the bettor at time $t$ is therefore $\K_t = \prod_{i=1}^t S_i(V_i),$
and the process $(\K_t)_{t\geq 0}$ is guaranteed to be a supermartingale (on $P$) due to the assumption on the payoff function.

We apply this to the PAC-Bayes setting as follows. Assume that the data are i.i.d.\ and that  the losses are stationary and bounded in $[0,1]$. 
We will consider playing a game for each parameter $\theta\in \Theta$, and will thus have a family of wealth processes $\{(\K_t(\theta))_{t\geq 0}:\theta\in \Theta\}$. 
We take the values $V_t$ to be the losses $f(Z_t,\theta)$. 
Following recent work in this area, 
suppose we use the following payoff function for each $\theta$: $S_t(f(Z_t,\theta)) = 1 + \lambda_t(\theta) (f(Z_t,\theta) - \mu(\theta))$ where $(\lambda_t(\theta))_{t\geq 0}$ is a predictable sequence (often called a \emph{betting strategy}) and  we enforce that $\lambda_t(\theta)\in[-1/(1 - \mu(\theta)),1/\mu(\theta)]$ to ensure that $S_t$ is nonnegative.  
Recalling that $\mu(\theta) = \E[f(Z_t,\theta)|\F_{t-1}]$, 
it's easy to verify that the resulting wealth process defined by 
\begin{equation}
\label{eq:wealth}
    \K_t(\theta) = \prod_{i=1}^t \big\{1 + \lambda_i(\theta)(f(Z_i,\theta) - \mu(\theta))\big\},
\end{equation}
is a nonnegative martingale. Consequently, it may be employed in Theorem~\ref{thm:pac-bayes-bounded-proc} where we take $\exp U_t(\theta)$ to be $\K_t(\theta)$. However, the results of \citet{jang2023tighter} concern not only the wealth, but the \emph{optimal} wealth, which is defined as 
\begin{equation}
    \K_t^*(\theta) := \max_{\lambda \in C(\mu(\theta))} \prod_{i=1}^t \big\{ 1 + \lambda (f(Z_i,\theta) - \mu(\theta))\big\},\quad C(x) := \bigg[\frac{-1}{1 - x}, \frac{1}{x}\bigg]. 
\end{equation}
\citet{orabona2021tight} show that there exists a betting strategy such that the wealth and the optimal wealth are related as 
\begin{equation}
\label{eq:opt_wealth_bound}
    \log \K_t^*(\theta)  - \log \K_t(\theta) \leq \log\bigg(\frac{\pi\Gamma(t+1)}{\Gamma(t + 1/2)}\bigg).
\end{equation}

This yields the following result, which is Theorem 1 of \citet{jang2023tighter}. 

\begin{corollary}[Betting-based anytime bound]
    \label{cor:betting-style-mart}
    Let $(Z_t)$ be i.i.d.\ and $f$ a stationary loss function bounded in [0,1]. Fix a data-free prior $\nu\in\Mspace{\Theta}$. Then, for all $\delta\in(0,1)$, with probability at least $1-\delta$ over the random draw of $(Z_t)$, for all times $t$ and measures $\rho\in\Mspace{\Theta}$,
        \begin{equation}
            \E_\rho \log \K_t^*(\theta) \leq \kl(\rho\|\nu) + \log(1/\delta) + \log\bigg(\frac{\pi\Gamma(t+1)}{\Gamma(t + 1/2)}\bigg). 
        \end{equation}
\end{corollary}
\begin{proof}
Since $(\K_t(\theta))$ is a nonnegative martingale, 
Theorem~\ref{thm:pac-bayes-bounded-proc} applied to the wealth process immediately yields that with probability $1-\delta$ over $(Z_t)$, $\E_\rho \log \K_t(\theta) \leq \kl(\rho\|\nu) + \log(1/\delta)$. Applying \eqref{eq:opt_wealth_bound} then finishes the proof. 
\end{proof}

By applying various inequalities to the term $\log(1 + \lambda (f(Z_i,\theta) - \mu(\theta)))$, \citet{jang2023tighter} are able to recover (up to constants), Theorems of \citet{mcallester1999pac}, \citet{maurer2004note}, and \citet{tolstikhin2013pac}. We omit the details here and refer the reader to Propositions 2, 3, and 4 in \citet{jang2023tighter}. 
}

\subsubsection{Losses with bounded MGF}
Finally, we consider losses which may not be bounded or subGaussian but which have bounded moment generating functions (MGFs). 
The following bound is an anytime-valid version of Catoni's bound based on the log-MGF of the loss \citep{catoni2007pac}. Like Corollary~\ref{cor:bounded-bennet}, it is somewhat of an unusual bound seeing as the empirical risk is ``on the wrong side'', i.e., we bound the empirical risk in terms of the log-MGF of the expected risk. However, as discussed above, the bound has proven useful in various estimation problems~\citep{catoni2017dimension,catoni2018dimension}. 
 \edit{It reads as follows. Suppose $f$ is stationary and the data are i.i.d. Fix $n\in\N$ and a prior $\nu$.  Then, with probability at least $1-\delta$, for all $\rho$, 
 \begin{equation}
 \label{eq:fixed-time-mgf}
    \frac{1}{n}\sum_{i=1}^n \E_\rho f(Z_i,\theta) \leq \log \E_\rho\E_\dist [\exp( f(Z,\theta))] + \frac{\kl(\rho\|\nu) + \log(1/\delta)}{n}.
\end{equation}
 }
Our time-uniform extension is given by Corollary~\ref{cor:cgf}. It recovers \eqref{eq:fixed-time-mgf} exactly by taking $t=n$, $\lambda_i = 1$ for all $i$, and then dividing both sides by $n$ (and assuming that the losses are stationary and the data i.i.d.).   
	
	\begin{corollary}[Losses with bounded MGF]
		\label{cor:cgf}
		Let $Z_1,Z_2,\dots$ be a stream of (not necessarily i.i.d.) data. 
		Let $(f_t)$ be a predictable sequence of loss functions.  Let $(\lambda_t)$ be a nonnegative predictable sequence and consider any data-free prior $\nu\in\Mspace{\Theta}$. Then, for all $\delta\in(0,1)$, with probability at least $1-\delta$ \edit{over the random draw of $(Z_t)$}, for all times $t$ and measures $\rho\in\Mspace{\Theta}$,
		\begin{equation}
			\sum_{i=1}^t \lambda_i \E_\rho f_i(Z_i,\theta) \leq \sum_{i=1}^t \log \E_\rho\E_\dist [\exp(\lambda_i f_i(Z,\theta))|\F_{i-1}] + \kl(\rho\|\nu) + \log(1/\delta).
		\end{equation}
	\end{corollary}

 The proof  may be found in Appendix~\ref{app:proof-anytime-cgf}.

	\subsection{More General Losses}
	\label{sec:heavy-losses}
	
	Now we consider less well-behaved losses.

\subsubsection{Losses with Bounded Second Moment.}
Our second bound in this section
assumes only that the conditional second moment of the loss is finite, i.e., $\E_\dist[f_t^2(Z,\theta)|\F_{t-1}]<\infty$ for all $\theta\in\Theta$, and relies on the nonnegative process
	\begin{equation*}
		M_t(\theta) := \prod_{i=1}^t\exp\bigg\{\lambda_i\Delta_i(\theta) - \frac{\lambda_i^2}{2}\E_\dist [f_i(Z_i,\theta)^2|\F_{i-1}]\bigg\}, 
	\end{equation*}
	which can be seen to be a supermartingale via an application of a one-sided Bernstein inequality. Lemma~\ref{lem:second-moment-sub-psi} gives the relevant statement and proof of this result. 
 As far as we are aware, the resulting PAC-Bayes bound is novel. 

\begin{corollary}[Losses with bounded conditional second moment]
		\label{cor:anytime-second-moment}
  Let $Z_1,Z_2,\dots$ be a stream of (not necessarily i.i.d.) data. 
		Let $(\lambda_t)$ be a nonnegative predictable sequence and consider any data-free prior $\nu\in\Mspace{\Theta}$.
		Let $(f_t)$ be a sequence of predictable loss functions such that $\sigma_t^2(\theta)=\E_\dist [f_t^2(Z,\theta)|\F_{t-1}]<\infty$. Then, for all $\delta\in(0,1)$, with probability at least $1-\delta$ \edit{over the random draw of $(Z_t)$}, for all  $t$ and  $\rho\in\Mspace{\Theta}$,
	\begin{equation}
 \label{eq:anytime-second-moment}
			\sum_{i=1}^t \lambda_i \E_\rho \Delta_i(\theta) \leq \sum_{i=1}^t \frac{\lambda_i^2}{2}\E_\rho\sigma_i^2(\theta) + \kl(\rho\|\nu) + \log(1/\delta).
		\end{equation}
	\end{corollary}

We now give another bound assuming only the second moment is finite.  It is based on a supermartingale discovered by \citet{bercu2008exponential} and the resulting bound (for stationary losses $f_t$ and constant $\lambda=\lambda_i$) was given by \citet{haddouche2022pac}. Let
 \[M_t(\theta) = \sum_{i=1}^t \Delta_i(\theta) = \sum_{i=1}^t (\mu_i(\theta) - f_i(Z_i,\theta)). \]
 The \emph{quadratic variation} of $M_t(\theta)$ is $[M(\theta)]_t := \sum_{i=1}^t  \Delta_i^2(\theta)$ 
% \begin{equation*}
% 	[M(\theta)]_t := \sum_{i=1}^t  (\mu_i(\theta) - f_i(Z_i,\theta))^2, 
% \end{equation*}
and its conditional quadratic variation is $\la M(\theta)\ra_t := \sum_{i=1}^t \E[\Delta_i^2(\theta)|\F_{i-1}].$
% \begin{equation*}
% 	\la M(\theta)\ra_t := \sum_{i=1}^t \E[(\mu_i(\theta) - f_i(Z_i,\theta))^2|\F_{i-1}].
% \end{equation*}
 \citet{bercu2008exponential} (see also \citealp[Table 3]{howard2020time}) demonstrate that the process 
 \[L_t(\theta) = \exp\bigg\{\lambda M_t(\theta) - \frac{\lambda^2}{6}\bigg([M(\theta)]_t + 2\la M(\theta)\ra_t\bigg) \bigg\},\]
 is a supermartingale for all $\lambda \in \R$.  
 Our unified proof technique
 leads us immediately to the following result. 
	
	\begin{corollary}[Anytime bound finite second moment]
		\label{cor:anytime-sn1}
		Let $Z_1,Z_2,\dots$ be a stream of (not necessarily i.i.d.) data and  
		$(f_t)$ be a sequence of predictable loss functions such that $\E_\dist[f_t^2(Z,\theta)|\F_{t-1}]<\infty$. Let $(\lambda_t)$ be a nonnegative predictable sequence and consider any data-free prior $\nu\in\Mspace{\Theta}$. Then, for all $\delta\in(0,1)$, with probability at least $1-\delta$ \edit{over the random draw of $(Z_t)$}, for all times $t$ and measures $\rho\in\Mspace{\Theta}$,
    \begin{equation}
    \label{eq:sn1}
     \sum_{i\leq t}\lambda_i\E_\rho\Delta_i(\theta) \leq  \frac{1}{6}\sum_{i\leq t}\lambda_i^2\E_\rho\bigg(\Delta_i^2(\theta) + 2\E_\dist[\Delta_i^2(\theta)|\F_{i-1}]\bigg) + \log(1/\delta) + \kl(\rho\|\nu).
 \end{equation}
\end{corollary}

 The proof of this result (including that $L_t(\theta)$ above forms a supermartingale) can be found in Appendix~\ref{app:proof-anytime-sn1}. 
 The right hand side of \eqref{eq:sn1} can be upper bounded to give a more interpretable result. In particular, if we consider stationary losses and i.i.d.\ data, then we can replace \eqref{eq:sn1} with the following: 
 %  \begin{equation}
 %  \label{eq:sn1-simplified}
 %     \sum_{i\leq t}\lambda_i\E_\rho\Delta_i(\theta) \leq  \frac{1}{6}\sum_{i\leq t}\lambda_i^2\bigg(f_i^2(Z_i,\theta) + 2\E[f_i^2(Z,\theta)|\F_{i-1}]\bigg) + \log(1/\delta) + \kl(\rho\|\nu).
 % \end{equation}
 \edit{
   \begin{equation}
  \label{eq:sn1-simplified}
     \E_\rho\risk(\theta) \leq \E_\rho \hR_t(\theta) + \frac{\lambda}{6t}\sum_{i\leq t}f^2(Z_i,\theta) + \frac{\lambda}{3}\E_{\rho,\dist}[f^2(Z,\theta)] + \frac{\log(1/\delta) + \kl(\rho\|\nu)}{\lambda t}.
 \end{equation}
 }
 This recovers (with slightly tighter constants), Theorem 2.3 of \citet{haddouche2022pac}. 
  \edit{The relationship between \eqref{eq:sn1-simplified} and Corollary~\ref{cor:anytime-second-moment} is worth investigating. For i.i.d.\ data and fixed $\lambda_i=\lambda>0$, \eqref{eq:anytime-second-moment} can be rearranged to read 
  \begin{equation}
  \label{eq:anytime-second-moment-simplified}
      \E_\rho \risk(\theta) \leq \E_\rho \hR_t(\theta) + \frac{\lambda}{2}\E_{\rho,\dist} [f^2(Z,\theta)] + \frac{\kl(\rho\|\nu) + \log(1/\delta)}{\lambda t}.
  \end{equation}
  Subtracting the right hand side of \eqref{eq:sn1-simplified} from \eqref{eq:anytime-second-moment-simplified} gives 
  \[D:= \frac{\lambda}{6t}\sum_{i\leq t}f^2(Z_i,\theta) - \frac{\lambda}{6}\E_{\rho,\dist}[f^2(Z,\theta)],\]
  which converges to zero almost surely via the LLN. Because \eqref{eq:sn1-simplified} is looser than \eqref{eq:sn1}, this implies that Corollary~\ref{cor:anytime-second-moment} is looser than Corollary~\ref{cor:anytime-sn1}.  Corollary~\ref{cor:anytime-second-moment} is, however, a cleaner result, and one we felt was worth stating. 
}
Let also note that using \eqref{eq:sn1-simplified}, \citet{haddouche2022pac} are able to generalize previous work of \citet{haddouche2021pac} on unbounded losses under the Hypothesis Dependent Range Condition (HYPE). The same discussion and generalization thus applies here. 

\edit{We end this section with an open problem: Can we obtain a time-uniform PAC-Bayes bound for losses under the sole assumption of a bounded $p$-th moment, $1<p\leq 2$? 
\citet{wang2022catoni}, based on previous work of \citet{chen2021generalized} and \citet{catoni2012challenging}, have provided nonnegative supermartingales under such conditions. They do not, however, result in closed form expressions of the risk, making the resulting PAC-Bayes bound difficult to use. 

}

\section{PAC-Bayes Bounds via Submartingales}
\label{sec:submart_bounds}

    While Section~\ref{sec:supermart_bounds} was able to generalize several fixed-time PAC-Bayes bounds, the sub-$\psi$ approach explored therein does not cover all existing PAC-Bayes bounds. Here we explore the other half of Theorem~\ref{thm:pac-bayes-bounded-proc}, giving bounds based on reverse-time submartingales.

    Throughout this section, for reasons that will become clear later, we will require that the loss is stationary ($f_t=f$) and that the data $(Z_t)$ are \emph{exchangeable}. In particular, for all $t\geq 1$, and permutations $g:[t]\to[t]$, $(Z_1,\dots,Z_t)\stackrel{d}{=} (Z_{g(1)},\dots,Z_{g(t)})$. Exchangeability is slightly weaker than the i.i.d.\ assumption. For instance, sampling without replacement gives rise to exchangeable sequences which are not i.i.d.\ Another example comes from considering $X_1+Y,\dots,X_n+Y$ for some random variable $Y$ and i.i.d.\ $X_1,\dots,X_n$. 
    Observe that exchangeability implies a common mean, 
    so throughout this section we set $\risk(\theta) = \risk_t(\theta) = \E_\dist[f(Z,\theta)]$ for all $t$.

    The bounds in the previous section were based on the process
    $S_t = \sum_{i=1}^t (\mu_i(\theta) - f_i(Z_i,\theta))$, while those in this section will be based on the process $(S_t/t)$. This is because, while the partial sums $(S_t)$ form a martingale, only the \emph{partial means} $(S_t/t)$ form a reverse submartingale. 
    We'll see that while PAC-bounds based on reverse submartingales can capture a larger variety of relationships between $R_t(\theta)$ and $\hR_t(\theta)$, this comes at the expense of slightly looser bounds in addition to stronger distributional assumptions.

	A formidable example of a bound which is not recovered by appealing to supermartingales is that of \citet{germain2015risk} (a similar bound was stated by \citet{lever2010distribution}; Theorem 1). This generalizes a class of bounds which consider convex functions acting on the risk and empirical risk. In particular, this recovers earlier bounds of \citet{seeger2002pac,seeger2003bayesian,germain2009pac,mcallester1998some,mcallester2003simplified}. 
A similar bound was given recently by \citet{rivasplata2020pac} when considering PAC-Bayes bounds for stochastic kernels.

	\begin{proposition}[\citealp{germain2015risk}]
		\label{prop:fixed-time-convex}
  Let $Z_1,\dots,Z_n$ be i.i.d.,
	 $\vp:[0,1]^2\to\R$ be convex and $f=f_t$ be stationary and bounded in $[0,1]$. \edit{Let $\nu$ be a data-free prior.} For all $n$ and $\lambda>0$, with probability at least $1-\delta$ \edit{over the random draw of $(Z_t)$}, for all $\rho\in\Mspace{\Theta}$,
  	\begin{align*}
			\vp(\E_{\rho} \hR_n(\theta),\E_{\rho}\risk(\theta)) &\leq \frac{1}{\lambda}\log \E_{\edit{\nu}}\E_\dist\exp(\lambda\vp(\hR_n(\theta),\risk(\theta)) + \frac{\kl(\rho\|\nu) + \log(1/\delta)}{\lambda}.
		\end{align*}
	\end{proposition}
	
	Let us consider for a moment attempting to give an anytime-valid version of the above result using the machinery from Section~\ref{sec:supermart_bounds}. One would need to guarantee that the nonnegative process
	$P_t(\theta) = \exp\big\{\lambda\vp(\E_{\rho} \hR_n(\theta),\E_{\rho}\risk(\theta)) - \log \E_{\edit{\nu}}\E_\dist\exp(\lambda\vp(\hR_t(\theta),\risk(\theta))\big\}$
	is upper bounded by a supermartingale. Since $\vp$ may not be linear, however, one cannot write this as a product of exponential terms, thereby making it difficult to write $\E[P_t(\theta)|\F_{t-1}]$ in terms of $P_{t-1}(\theta)$. We thus require a different approach.  Interestingly, one can show that convex functions acting on the empirical risk are reverse submartingales with respect to an appropriate filtration, which we define below. From here, Ville's inequality for reverse submartingales (Lemma~\ref{lem:ville_submartingales}) will provide us with an anytime version of Proposition~\ref{prop:fixed-time-convex}.

	Given a sequence of data $Z_1,Z_2,\dots$, 
	the \emph{exchangeable reverse filtration} $(\mE_t)_{t=1}^\infty$ is the reverse filtration where $\mE_t$ is the $\sigma$-algebra generated by all (Borel) measurable functions of the data which are permutation symmetric in their first $t$ arguments. We say a function $s$ is permutation symmetric if $s(Z_1,\dots,Z_t)=s(Z_{g(1)},\dots,Z_{g(t)})$ for all permutations $g:[t]\to[t]$. Formally, $\mE_t$ is written 
	\begin{equation}
		\label{eq:exch_reverse_filt}
		\mE_t = \sigma\bigg(\big\{s(Z_1,\dots,Z_t):s\text{ is permutation symmetric }\big\}\cup\{Z_j\}_{j>t}\bigg).
	\end{equation}
 We find the following intuition from \citet{manole2021sequential} helpful when thinking about $\mE_t$. $\mE_1$ might be viewed as an omniscient oracle with access to all information over the whole future. As time goes on, her memory of the past decays but she retains perfect knowledge of the future. Importantly, she does not forget what happened in the past, only the \emph{order} in which events occurred. That is, the oracle $\mE_t$ is omniscient with respect to $Z_{t+1},Z_{t+2},\dots$, but forgets the order of $Z_1,\dots,Z_t$. 
	\citet{manole2021sequential} also give a sufficient condition for a  process to be a reverse submartingale with respect to $(\mE_t)$.

	\begin{lemma}
 [Leave-one-out, \citealp{manole2021sequential}, Corollary 5]
		\label{lem:leave-one-out}
		If a sequence of permutation invariant functions $\{h_t:\cZ^t \to\R\}$ satisfies the ``leave-one-out'' property, namely, $h_t(Z^t) \leq \frac{1}{t}\sum_{i=1}^t h_{t-1}(Z^t_{-i})$ (where $Z^t_{-i}$ omits $Z_i$),
		then $(h_t(Z^t))_{t=0}^\infty$ is a  reverse submartingale with respect to  $(\mE_t)$. 
  Moreover, if the expression above holds with equality then $(h_t(Z^t))$ is a reverse martingale with respect to $(\mE_t)$. 
	\end{lemma}

	To reduce notational clutter, given a convex function $\vp:\R_{\geq 0}\times \R_{\geq 0}\to \R$, define 
	\begin{equation}
 \label{eq:convex-short}
		\vp_t(\theta) :=  \vp(\hR_t(\theta),\risk(\theta)).
	\end{equation}
	That is, $\vp_t$ simply fixes the second argument of $\vp$ as $\risk(\theta)$ and sets the first as the empirical risk at time $t$. 
 Considering $\vp_t$ is useful because the stochastic process 	$(\vp_t(\theta))_{t=1}^\infty$ for fixed  $\theta$ is a reverse submartingale with respect to $(\mE_t)$. This holds by Lemma~\ref{lem:leave-one-out}, since the empirical risk $\hR_t(\theta)$ is permutation invariant and the convexity of $\vp$ ensures that the leave-one-out property holds, as proven below. 
	
	\begin{lemma}
		\label{lem:convex-submart} 
		For an exchangeable sequence $(Z_t)$, $(\vp_t(\theta))$ is a reverse submartingale with respect to $(\mE_t)$. 
	\end{lemma}
	\begin{proof}
		First note that $\vp_t(\theta)$ is permutation invariant by construction. Thus, by Lemma~\ref{lem:leave-one-out}, we need only show that it satisfies the leave-one-out property. 
		For each $i\in[t]$, define \[\hR_t^{(-i)}(\theta):=\frac{1}{t-1}\sum_{j\neq i} f(Z_j,\theta),\]
		and observe that 
		\[\sum_{i=1}^t \hR_t^{(-i)}(\theta) = \frac{1}{t-1}\sum_{i=1}^t \sum_{j\neq i}f(Z_j,\theta) = \sum_{i=1}^t f(Z_j,\theta) = t\hR_t(\theta).\]
		Consequently, by the convexity of $\vp$ and Jensen's inequality,
		\begin{align*}
			\vp_t(\theta) &= \vp(\hR_t(\theta),R(\theta)) =   \vp\bigg(\frac{1}{t}\sum_{i=1}^t \hR_t^{(-i)}(\theta),R(\theta) \bigg) \\
			&\leq \frac{1}{t}\sum_{i=1}^t \vp(\hR_t^{(-i)}(\theta),R(\theta)) = \frac{1}{t}\sum_{i=1}^t \vp_0((Z_1,\dots,Z_{i-1},Z_{i+1},\dots,Z_t),\theta),
		\end{align*}
        which is precisely the leave-one-out property.
	\end{proof}

    Our reliance on Lemma~\ref{lem:leave-one-out} is the reason that this section considers only stationary loss functions (but so do the bounds we generalize). More specifically, stationary losses are required for $\vp_t(\theta)$ to be permutation invariant. 
    We cannot in general swap $Z_i$ and $Z_k$ if $f_i$ and $f_k$ are different. 

\subsection{A Time-Uniform Bound for Convex Functions}
\label{sec:bound-convex-functional}

As we alluded to in the introduction, while the supermartingale approach of Section~\ref{sec:supermart_bounds} was able to generalize fixed-time bounds at no cost, this is not true for the bounds presented in this section. Roughly speaking, this is because even though the process $(\vp_t(\theta))_{t\geq 1}$ is a reverse submartingale with respect to $(\mE_t)$ (and therefore so is $(\exp(\lambda \vp_t(\theta)))_{t\geq 1}$), the process $(\exp\{\lambda \vp_t(\theta) - \log\E_{\edit{\nu},\dist} \exp(\lambda \vp_t(\theta))\})_{t\geq 1}$ may not be. Thus, we cannot use such a process in Theorem~\ref{thm:pac-bayes-bounded-proc} to recover (a time-uniform version of) Proposition~\ref{prop:fixed-time-convex} exactly. 

Instead, we rely on a ``stitching'' argument in a similar vein to \citet{howard2021time} and \citet{manole2021sequential}. 
This entails considering a series of submartingales over geometrically spaced epochs $[2^{t-1},2^{t})$, $t\geq 0$, each holding with a precise probability such that we may take the union bound over all such intervals to obtain our result. As we'll see, the resulting bounds will suffer at most a small constant factor plus an iterated logarithm factor over the originals. 

Formally, we consider a ``stitching function" function $\ell: \mathbb N_{> 0} \to (1,\infty)$ such that $\sum_{k=1}^\infty \frac{1}{\ell(k)}\leq 1$.  Different choices will leads to different shapes of the resulting bounds. For clarity and concreteness, however, we will consider the following stitching function for the remainder of this manuscript:
\[
\ell(k) = k^2 \zeta(2), \quad \text{where }\quad \zeta(2) = \sum_{j=1}^\infty j^{-2} \approx 1.645.
\]
We also introduce the following ``iterated logarithm" factor that captures the small excess error inherent to our anytime-valid bounds: 
\begin{equation}
\label{eq:stitching-error}
\il_t :=  \log(\ell(\log_2(2t)))
%< \log(2\log_2^2(t)) this is bigger than 2loglogt + 1.3
< 2 \log \log 2t + 1.3.  
\end{equation}
Additionally, throughout this section we set 
\begin{equation}
\label{eq:stitch-time}
  \bar{t} := 2^{\lfloor \log_2(t)\rfloor}.  
\end{equation}
With these definitions in hand, we now state our time-uniform version of Proposition~\ref{prop:fixed-time-convex}. 

\begin{corollary}[General anytime bound for convex functions]
\label{cor:anytime-convex}
    Let $(Z_t)$ be exchangeable. 
		Let $\vp:\R_{\geq 0}\times \R_{\geq 0}\to\R$ be convex and $\nu\in\Mspace{\Theta}$ be a prior. Let $(\lambda_t)$  be a sequence of positive values. 
		Then, for all $\delta\in(0,1)$, with probability at least $1-\delta$ \edit{over the random draw of $(Z_t)$}, for all $\rho\in\Mspace{\Theta}$ and at all times $t\geq 1$, 
  		\begin{align}
  \label{eq:anytime-convex}
			\E_\rho \vp_t(\theta) 
  &\leq \frac{\log \E_{\edit{\nu}} \E_\dist \exp\big\{\lambda_{\bar{t}} \vp_{\bar{t}}(\theta)\big\}}{\lambda_{\bar{t}}} + \frac{\kl(\rho\|\nu) + \log(1/\delta) + \il_t}{\lambda_{\bar{t}}}, 
		\end{align}
  for $\il_t$ as in \eqref{eq:stitching-error}, $\bar{t}$ as in \eqref{eq:stitch-time}, and $\vp_t(\theta) = \vp(\hR_t(\theta),R(\theta))$. 
\end{corollary}

Ideally, the subscripts $\bar{t}$ above would have been equal to $t$, but we were not able to prove such a result. Since $t/2 \leq \bar{t} \leq t$, this results in a slight looseness; see Remark~\ref{remark:eta(t)}.

\bigskip

\begin{proof}
Recall the shorthand $\vp_t(\theta) = \vp(\hR_t(\theta),\risk(\theta))$. 
For $j\in\N$, define 
\begin{equation}
  M_t^j(\theta) = \lambda_j \vp_t(\theta) - \log \E_{\edit{\nu},\dist} \exp (\lambda_j \vp_j(\theta)).  
\end{equation}
the second term on the right hand side is deterministic, 
so Lemma~\ref{lem:convex-submart} implies that $(M_t^j(\theta))_{t\geq 1}$ is a reverse submartingale with respect to $(\mE_t)$. 
Hence, by Jensen's inequality, so is the process $(\exp M_t^j(\theta))_{t\geq 1}$. 
Moreover, note that $\E_\dist\exp(M_{j}^j(\theta))=1$. 
Therefore, Theorem~\ref{thm:pac-bayes-bounded-proc} implies that, for all $\rho$,
\begin{align}
    \Pr(\exists t\geq j: \E_\rho M_t^j(\theta) - \kl(\rho\|\nu) \geq \log(u/\delta)) \leq \delta/u, \label{eq:anytime-convex-1}
\end{align}
for $u>0$. 
%Fix $\rho\in\Mspace{\Theta}$.
Suppose that for some $t^*\geq 1$ \edit{and some $\rho\in\Mspace{\theta}$} we have the inequality $\E_\rho M_{t^*}^{\bar{t^*}}(\theta) - \kl(\rho\|\nu)\geq \log(\ell(\log_2(2 t^*))/\delta)$. By construction, $\bar{t^*}=2^{k^*}$ where $k^*=\lfloor \log_2(t^*)\rfloor\in\N$. Therefore, 
\begin{align*}
  \E_\rho M_{t^*}^{2^{k^*}}(\theta) - \kl(\rho\|\nu) &= \E_\rho M_{t^*}^{\bar{t^*}}(\theta) - \kl(\rho\|\nu) \\
  &\geq \log(\ell(\log_2( 2 t^*))/\delta) \geq \log(\ell(k^*+1)/\delta),  
\end{align*}
where the final inequality follows since 
$\log_2(2t^*) = \log_2(t^*) + 1 \geq \lfloor \log_2(t^*)\rfloor + 1 = k^*+1$, and $\ell$ is an increasing function. 
%We have thus shown that the event $\{\exists t\geq 1: \E_\rho M_t^{\bar{t}}(\theta) - \kl(\rho\|\nu) \geq \log(\ell(\log_2(2t))/\delta)\}$ is contained in the event 
%\[\bigcup_{k=0}^\infty \big\{\exists t\geq 2^k: \E_\rho M_t^{2^k}(\theta) - \kl(\rho\|\nu) \geq \log(\ell(k+1)/\delta))\big\},\]
\edit{
We have thus shown that 
the event \[\{\exists \rho,  \exists t\geq 1: \E_\rho M_t^{\bar{t}}(\theta) - \kl(\rho\|\nu) \geq \log(\ell(\log_2(2t))/\delta)\},\] is contained in \[\bigcup_{k=0}^\infty \big\{\exists \rho, \exists t\geq 2^k: \E_\rho M_t^{2^k}(\theta) - \kl(\rho\|\nu) \geq \log(\ell(k+1)/\delta))\big\},\]
} implying that 
\begin{align*}
    &\quad\ \Pr\bigg(\edit{\exists \rho}, \exists t\geq 1: \E_\rho M_t^{\bar{t}}(\theta) - \kl(\rho\|\nu) \geq \log(\ell(\log_2(2t))/\delta)\bigg) \\
    &\leq 
    \Pr\bigg(\bigcup_{k=0}^\infty \big\{\edit{\exists \rho}, \exists t\geq 2^k: \E_\rho M_t^{2^k}(\theta) - \kl(\rho\|\nu) \geq \log(\ell(k+1)/\delta))\big\}\bigg) \\
    &\leq \sum_{k=0}^\infty \Pr(\edit{\exists\rho}, \exists t\geq 2^k:  \E_\rho M_t^{2^k}(\theta)  - \kl(\rho\|\nu) \geq \log(\ell(k+1)/\delta)) \leq \sum_{k=0}^\infty \frac{\delta}{\ell(k+1)} = \delta,
\end{align*}
where we've applied \eqref{eq:anytime-convex-1} with $u=\ell(k+1)$. 
In other words, expanding $M_t^{\bar{t}}(\theta)$, we have that with probability at least $1-\delta$, for all $\rho$ and $t\geq 1$,
\begin{align*}
  \lambda_{\bar{t}} \E_\rho \vp_t(\theta) &\leq \E_\rho \log \E_{\edit{\nu},\dist} \exp(\lambda_{\bar{t}}\vp_{\bar{t}}(\theta)) + \kl(\rho\|\nu) + \log(\ell(\log_2(2t))/\delta) \\
  &\leq \log \E_{\edit{\nu},\dist}\exp(\lambda_{\bar{t}} \vp_{\bar{t}}(\theta)) + \kl(\rho\|\nu) + \log(1/\delta) + \il_t, 
\end{align*}
using the definition of $\ell$ and the fact that $\log$ is concave. Dividing by $\lambda_{\bar{t}}$ completes the argument. 
\end{proof}

\edit{
\begin{remark}
\label{remark:eta(t)}
Notice that our choice of $\bar{t}$ may leave us with a bound that is approximately twice as big as the fixed time counterpart. Indeed, if $t=1023$ then $\bar{t} = 512\approx t/2$. Thus if $\lambda_j = j$ (say), then $\lambda_{\bar{1023}} = 512 \approx \lambda_{1023}/2$.  The culprit is our choice of $\bar{t}$ and the constant  2 therein. We chose 2 simply as a matter of convenience, but the analysis can be modified. In particular, for any fixed $s>1$, we may consider $\bar{t} = s^{\lfloor \log_s(t)\rfloor}$, which obeys $t/s\leq \bar{t}\leq t$. As $s\to 1$, $\bar{t}$ thus lags behind $t$ less and less. The price we pay is that the iterated logarithm error term must be modified to $\il_t=\log(\ell(\log_s(st)))$ which grows as $s\to 1$. We leave the choice of $s$ and the appropriate trade-off between the lag and the additive error to the practitioner. 
\end{remark}

}

Comparing Corollary~\ref{cor:anytime-convex} and Proposition~\ref{prop:fixed-time-convex}, we see there are several differences aside from $\il_t$.  For one, our expectation is on the outside of $\vp_t$ on the left hand side. Of course, because $\vp$ is convex, $\E_\rho \vp(\hR_t(\theta),\risk(\theta))\geq \vp(\E_\rho \hR_t(\theta),\E_\rho \risk(\theta))$, so our result implies a bound on the latter term.   
Second, as noted in the remark above, our log-MGF term is based on $\bar{t}\in[t/2,t]$ instead of $t$, thus ``lags behind'' the fixed-time result. This is a consequence of stitching. However, if there is a fixed time $n$ of special interest, we can obtain the following time-uniform bound for all $t\geq n$, which is just as tight as Proposition~\ref{prop:fixed-time-convex}. 

\begin{corollary}
\label{cor:anytime-convex-target}
    Let $(Z_t)$ be exchangeable. 
		Let $\vp:\R_{\geq 0}\times \R_{\geq 0}\to\R$ be convex and $\nu\in\Mspace{\Theta}$ be a prior. Fix $\lambda>0$ and $n\in\mathbb{N}$.  
		Then, for all $\delta\in(0,1)$, with probability at least $1-\delta$ \edit{over the random draw of $(Z_t)$}, for all $\rho\in\Mspace{\Theta}$ and at all times $t\geq n$, 
  		\begin{align}
  \label{eq:anytime-convex-target}
			\E_\rho \vp(\hR_t(\theta),\risk(\theta)) 
  &\leq \frac{\log \E_{\edit{\nu}} \E_\dist \exp\big\{\lambda \vp(\hR_{n}(\theta),\risk(\theta))\big\}}{\lambda} + \frac{\kl(\rho\|\nu) + \log(1/\delta)}{\lambda}.
		\end{align}
\end{corollary}
\begin{proof}
    Similarly to the proof of Corollary~\ref{cor:anytime-convex}, set 
    \[M_t^n(\theta) = \exp\big\{\lambda \vp_t(\theta) - \log\E_{\edit{\nu},\dist} \exp(\lambda \vp_n(\theta))\big\}.\]
    Then $(M_t^n(\theta))_{t\geq n}$ is 
     a reverse submartingale with respect to $(\mE_t)_{t\geq n}$ with $\E_\dist M_n^n(\theta)=1$. Therefore, Theorem~\ref{thm:pac-bayes-bounded-proc} gives 
     %\hrmk{here, and in proof of Theorem 28, etc., $\exists \rho$ is missing in $\Pr$ (whenever we use $\Pr$ we need to include it)}
    \[\Pr(\exists \rho, \exists t\geq n: \E_\rho M_t^n(\theta) - \kl(\rho\|\nu) \geq \log(1/\delta))\leq \delta,
    \]
    which rearranges to the claimed result.
\end{proof}

Corollary~\ref{cor:anytime-convex-target} requires some interpretation. The right hand side of \eqref{eq:anytime-convex-target} is constant with respect to $t$. 
% One might therefore suspect that it is vacuous: If more training data is added, surely the distance between $\hR_t(\theta)$ and $\risk(\theta)$ will not increase with $t$? Indeed, suppose we take $\vp(x,y) = |x-y|$ (or any other convex norm).  Corollary~\ref{cor:anytime-convex-target} says that $|\hR_t(\theta) - \risk(\theta)| \leq C$ for all $t\geq n$ and some constant $C$. 
% While this might seem at first obvious, it is known that the risk of empirical risk minimizers is not always monotonic~\citep{mhammedi2021risk,loog2019minimizers}.  
% It is thus comforting to know that, in the above PAC-Bayes setting at least, we can obtain a high probability result that the expected discrepancy between $\hR_t(\theta)$ and $\risk(\theta)$ will not increase as more data is added. 
While such a bound might be more straightforward for a fixed $t \geq n$, our bound shows that it holds simultaneously for all $t \geq n$. These bounds are in some sense analogous to Freedman-style deviation inequalities (which hold for all $t \leq n$, but with tightness only depending on $n$ and not improving for $t \ll n$) and perhaps even more analogous to de la Peña-style deviation inequalities (which hold for all $t\geq n$, but with tightness only depending on time $n$ and not improving for $t \gg n$) --- see~\cite{howard2020time} for a detailed discussion and a unification of the two types of boundaries (in particular Figures 1, 4 and 5 for intuition). 

\subsection{A Time-Uniform Seeger Bound}
 \label{sec:submart-seeger}

    	By choosing particular convex functions $\vp$ and applying Corollary~\ref{cor:anytime-convex} (or \ref{cor:anytime-convex-target}), we  recover time-uniform versions of several classical PAC-Bayes inequalities. We present several of them here, but refer the reader to resources such as \citet{alquier2021user} and \citet{germain2009pac} for more comprehensive discussions. 
        A particularly famous result is that of \citet{langford2001bounds} and \citet{maurer2004note}. 
     To state it, let us define, for any $p,q\in(0,1)$, 
 \[\klsf(p\|q) := p\log\bigg(\frac{p}{q}\bigg) + (1-p)\log\bigg(\frac{1-p}{1-q}\bigg),\] 
 which is the KL-divergence between Bernoulli distributions with means $p$ and $q$, respectively. That $\klsf$ is convex (in each argument, $p$ and $q$) thus follows from the fact that $\kl(\cdot\|\cdot)$ is convex (in distribution space). Indeed, 
 \begin{align*}
        \klsf(\lambda p_1 +  (1-\lambda)p_2\|q) &= \kl(\lambda \ber(p_1) + (1-\lambda)\ber(p_2) \|\ber(q)) \\
        &\leq \lambda \kl(\ber(p_1)\|\ber(q)) + (1-\lambda) \kl(\ber(p_2)\|\ber(q)) \\
        &= \lambda \klsf(p_1\|q) + (1-\lambda) \klsf(p_2\|q),
    \end{align*}
    for any $\lambda\in[0,1]$, where $\ber(p)$ is a Bernoulli distribution with mean $p$. An identical argument holds for the second argument of $\klsf$.  
    Now, for $k\in\N$, define 
	\begin{equation*}
		\xi(k) := \sum_{\ell=0}^{k} \Pr_{Y\sim \text{Bin}(k,\ell/k)}(Y=\ell) = \sum_{\ell=0}^k {k\choose \ell} (\ell/k)^\ell (1 - \ell/k)^{k-\ell}. 
	\end{equation*}
 As noted by \citet{maurer2004note,germain2015risk}, $\sqrt{k}\leq \xi(k)\leq 2\sqrt{k}$ for all $k\in\N$. Employing Corollary~\ref{cor:anytime-convex} leads to the following bound, which relates $\xi(k)$ to the log-MGF. Recall that $\bar{t} = 2^{\lfloor \log_2(t)\rfloor}$ and $\il_t < 2\log\log 2t + 1.3$. The proof of the following bound can be found in Appendix~\ref{app:proof-anytime-seeger}.
 
 \begin{corollary}[Anytime-valid Langford-Seeger Bound]
  \label{cor:anytime-seeger}   
  Let $(Z_t)$ be i.i.d.\ and  
    consider stationary losses bounded in $[0,1]$. 
    Let $\nu\in\Mspace{\Theta}$ be a data-free prior. 
    Then, for all $\delta\in(0,1)$, with probability at least $1-\delta$ \edit{over the random draw of $(Z_t)$}, for all $\rho\in\Mspace{\Theta}$ and at all times $t\geq 1$, 
    \begin{align}
    \label{eq:anytime-seeger}
        \E_\rho \klsf(\hR_t(\theta)\|\risk (\theta)) &\leq \frac{\kl(\rho\|\nu) + \log(\xi(\bar{t})/\delta) + \il_t}{\bar{t}}.
    \end{align}
    Moreover, for any fixed $n$, we obtain that for  all $\delta\in(0,1)$, with probability at least $1-\delta$ \edit{over the random draw of $(Z_t)$}, for all $\rho\in\Mspace{\Theta}$ and at all times $t\geq n$, 
    \begin{equation}
    \label{eq:anytime-seeger-target}
        \E_\rho \klsf(\hR_t(\theta)\|\risk(\theta)) \leq \frac{\kl(\rho\|\nu) + \log(\xi(n)/\delta)}{n}. 
    \end{equation}
 \end{corollary}
 \edit{For ease of comparison, let us recall the usual fixed-time version of this bound, which reads: For all $n\in\N$, with probability at least $1-\delta$, for all $\rho$, 
 \begin{equation}
     \label{eq:fixed-time-seeger}
     \klsf(\E_\rho \hR_n(\theta)\|\E_\rho\risk(\theta)) \leq \frac{\kl(\rho\|\nu) + \log(\xi(n)/\delta)}{n}. 
 \end{equation}
 
 }
 Noting that $\klsf(\E_\rho \hR_n(\theta)\|\E_\rho\risk(\theta)) \leq \E_\rho \klsf(\hR_n(\theta)\|\risk(\theta))$ due to Jensen's inequality, 
 we see that at time $t=n$, \eqref{eq:anytime-seeger-target} recovers the fixed-time Seeger bound. 
 Moreover, by noting that $\bar{t}\in[t/2,t]$, \eqref{eq:anytime-seeger} provides a guarantee for all $t\geq 1$, that is at most a constant factor worse than \eqref{eq:anytime-seeger-target}. We emphasize that this constant factor can be changed by altering the definition of $\bar{t}$; see Remark~\ref{remark:eta(t)}. 
 \edit{As a brief historical note, \eqref{eq:fixed-time-seeger} was first stated in that form by \citet[Lemma 20]{germain2015risk}. \citet[Theorem 3]{langford2001bounds} use $n$ in place of $\xi(n)$ and \citet[Theorem 5]{maurer2004note} then tightens this to $2\sqrt{n}$. }

 A time-uniform McAllester bound \citep{mcallester1998some,mcallester2003simplified} --- distinct from that derived in Section~\ref{sec:supermart_bounds} --- follows immediately by applying Jensen's inequality and 
 Pinsker's inequality: For all $x,y\in(0,1)$, $2(x-y)^2 \leq \klsf(x\|y)$. This implies that 
 $2[\E_\rho (\hR_t(\theta) - \risk(\theta))]^2\leq 2 \E_\rho (\hR_t(\theta) - \risk(\theta))^2 \leq \E_\rho \klsf(\hR_t(\theta)\|\risk(\theta))$. Using this in conjunction with the fact that $\xi(k)\leq 2\sqrt{k}$  yields the following.

	\begin{corollary}[Anytime-valid McAllester Bound]
		\label{cor:anytime-mcallester}
    Let $(Z_t)$ be i.i.d.\ and  
    consider stationary losses bounded in $[0,1]$. 
    Let $\nu\in\Mspace{\Theta}$ be a data-free prior. 
    Then, for all $\delta\in(0,1)$, with probability at least $1-\delta$ \edit{over the random draw of $(Z_t)$}, for all $\rho\in\Mspace{\Theta}$ and at all times $t\geq 1$, we have
		\begin{equation*}
			\E_\rho \risk(\theta) \leq \E_\rho \hR_t(\theta) +  \bigg(\frac{\kl(\rho\|\nu) + \log(2\sqrt{\bar{t}}/\delta) + \il_t}{2\bar{t}}\bigg)^{1/2}.
		\end{equation*}
      Moreover, for any fixed $n$, we obtain that for  all $\delta\in(0,1)$, with probability at least $1-\delta$ \edit{over the random draw of $(Z_t)$}, for all $\rho\in\Mspace{\Theta}$ and at all times $t\geq n$, 
    \begin{equation}
    \label{eq:anytime-macallester-target}
    \E_\rho \risk(\theta) \leq \E_\rho \hR_t(\theta) +  \bigg(\frac{\kl(\rho\|\nu) + \log(2n/\delta) }{2n}\bigg)^{1/2}.
    \end{equation}
	\end{corollary}
    
 As above, at the fixed time $t=n$, \eqref{eq:anytime-macallester-target} recovers McAllester's bound in \eqref{eq:mcallester}.  
Other bounds follow from other choices of $\vp$. \citet{begin2016pac} note that $\vp(x,y) = -cx - \log(1 - y(1-e^{-c})$ leads to Theorem 1.2.6 of \citet{catoni2007pac}. 
 Meanwhile, as pointed out by \citet{alquier2021user} and \citet{perez2021tighter} we can also generate the bounds of  \citet{tolstikhin2013pac} and \citet{thiemann2017strongly} by using other inequalities involving $\klsf$. 

\edit{
 \subsection{U- and V-statistics}
 \label{sec:u-statistics}

 We end this section by discussing the representation of U- and V-statistics as reverse submartingales, thus opening the door for bounds based on these quantities. We will assume the data $(Z_t)$ are drawn i.i.d.\ and that the loss function is stationary. We consider functionals of the form $\Phi:\mathcal{P}(\cZ) \times \Theta \to \R$, where 
 \begin{equation*}
     \Phi(P,\theta) = \iint h(f(z_1,\theta),f(z_2,\theta)) \d P(z_1)\d P(z_2).
 \end{equation*}
 Here, $P$ is a distribution over the data, and $h:\R \times \R \to \R_{>0}$ is symmetric, continuous, and positive semi-definite. 
 For instance, if we take $h(a,b) = (a-b)^2/2$, then $\Phi(P,\theta) = \Var_{Z\sim P} [f(Z,\theta)]$. 
 The U- and V-statistics of $\Phi$ are, respectively, 
 \begin{align}
     U_t(\theta) &:= \frac{2}{t(t-1)} \sum_{1\leq i<j\leq t} h(f(Z_i,\theta),f(Z_j,\theta)), \\ 
     V_t(\theta) &:= \frac{1}{t^2}\sum_{1\leq i,j\leq t} h(f(Z_i,\theta),f(Z_j,\theta)).
 \end{align}

 \citet{berk1966limiting} and, more recently, \citet{manole2021sequential} observe that $U_t(\theta)$ is a reverse martingale with respect to $(\mE_t)$.  Indeed, this can be seen by appealing to the leave-one-property (Lemma~\ref{lem:leave-one-out}). As for $V_t(\theta)$, it can be seen to be a reverse submartingale with respect to $(\mE_t)$ if, in addition to the conditions on $h$ above, the range of $f$ is a compact subset of $\R$ \citep[Proposition 16]{manole2021sequential}. 

 In conjunction with Theorem~\ref{thm:pac-bayes-bounded-proc}, these properties enable us to give time-uniform PAC-Bayes bounds involving functionals and their U- and V-statistics. 
 To illustrate, we recover time-uniform versions of Theorems 3 and 4 in \citet{tolstikhin2013pac}. Recalling Remark~\ref{remark:eta(t)}, this result will employ stitching and thus lags behind the fixed-time result by a small constant. As was discussed in Section~\ref{sec:wealth_proc}, a tighter time-uniform bound (i.e., one that uses forward supermartingales and thus does not have an iterated logarithm factor) can be obtained by using a betting-style martingale. However, we  state and prove the following bound as an example of how U-statistics and their properties can be applied in the PAC-Bayes setting.

\begin{corollary}
\label{cor:tolstikhin}
Let $(Z_t)$ be drawn i.i.d.\ and $f$ be stationary.  Let $\Var_t(\theta)$ be the unbiased empirical variance, i.e., $\Var_t(\theta) := \frac{1}{t(t-1)}\sum_{1\leq i<j\leq t} (f(Z_i,\theta) - f(Z_j,\theta))^2$. 
Let $(\lambda_t)$ be a sequence of positive scalars. 
Denote the true variance as $\Var(\theta) = \Var_{Z\sim P}[f(Z,\theta)]$. 
Then, for any $\delta\in(0,1)$, with probability at least $1-\delta$ over $(Z_t)$, for all $t\geq 1$ and $\rho\in\Mspace{\Theta}$,
\begin{equation}
\label{eq:tolstikhin-bound}
    \E_\rho[\Var(\theta) - \Var_t(\theta)] \leq \frac{\kl(\rho\|\nu) + \log(1/\delta) + \il_t + \frac{\lambda_{\bar{t}}^2}{2}\frac{\bar{t}^2}{\bar{t}-1}\E_\rho \Var(\theta)}{\bar{t}\lambda_{\bar{t}}},
\end{equation}
    where $\bar{t} = 2^{\lfloor \log_2(t)\rfloor}$ and $\il_t < 2\log\log 2t + 1.3$. 
\end{corollary}

The proof may be found in Appendix \ref{app:proof-tolstikhin}. 
Theorem 3 of \citet{tolstikhin2013pac} follows from optimizing over $\lambda$ with a union bound (for a fixed $t=n$). 
In particular, they consider a grid of $\lambda$s designed as a geometric progression. Then, for each $\rho$, they find the optimal value of $\lambda$ from the grid. Finally, they apply a union bound over the grid, a technique inspired by \citet{seldin2012pac} (and which we will see again in Section~\ref{sec:extensions}). 
Doing so yields a bound with dependence $\widetilde{O}(\sqrt{ \frac{\Var_t(\theta)(\kl(\rho\|\nu) + \log(1/\delta))}{n}} + \frac{\kl(\rho\|\nu) + \log(1/\delta)}{n})$, where $\widetilde{O}$ ignores iterated logarithm factors. If $\Var_t(\theta)$ is sufficiently  small, the bound then scales at a rate of roughly $1/n$. In the anytime setting meanwhile, we might consider taking $\lambda_j = \frac{1}{\sqrt{j}}$, in which case the right hand side of \eqref{eq:tolstikhin-bound} becomes proportional to $\frac{\kl(\rho\|\nu) + \log(1/\delta) + \E_\rho \Var(\theta)}{\sqrt{\bar{t}}}$, thus shrinking at a rate of roughly $1/\sqrt{\bar{t}}$. 
Similary to the fixed-time setting, Theorem 4 of \citet{tolstikhin2013pac} follows from combining the above result with Corollary \ref{cor:mds-bernstein}, a time-uniform extension of results in \citet{seldin2012pac}.  

 }
\section{Extensions}
\label{sec:extensions}	

Owing in part to the ability for PAC-Bayes bounds to provide insight into the performance of neural networks~\citep{dziugaite2017computing,biggs2022non}, recent years have seen a surge of interest in and progress on the topic. In this section, we provide some comments on the ability of our unified framework to incorporate some of these advances. In particular, we discuss replacing the KL divergence with integral probability metrics, $\phi$-divergences, and R\'{e}nyi divergences, in addition to how Theorem~\ref{thm:pac-bayes-bounded-proc} enables us to replace the loss function with martingale difference sequences. 
We also discuss how many of the bounds in the two previous sections give rise to confidence sequences (i.e., time-uniform confidence intervals), and provide some general advice on choosing $(\lambda_t)$ in the supermartingale bounds.

 \subsection{Replacing the KL Divergence with IPMs}
 \label{sec:beyond-kl}

Given that all the bounds provided thus far rely on the KL divergence between $\rho$ and $\nu$, a natural question is whether we can replace this term with an alternative distributional metric? Here we answer in the affirmative and demonstrate that recent work by \cite{amit2022integral}, which replaces the KL divergence with a variety of \emph{Integral Probability Metrics} (IPMs), can be made time-uniform.

\begin{definition}
    \label{def:ipm}
    Let $\cG$ be a family of functions which map $\Theta$ to $\R$. The Integral Probability Metric with respect to $\cG$ between two distributions $\rho$ and $\nu$ over $\Theta$ is 
\begin{equation}
    \label{eq:ipm}
    \gamma_\cG(\rho,\nu) := \sup_{g\in \cG} \big |\E_{\theta\sim\rho} g(\theta) - \E_{\theta\sim\nu}g(\theta)\big|.
\end{equation}
\end{definition}
IPMs are a large class of divergences. By choosing the appropriate family $\cG$, one can recover the Total Variation distance, the Wasserstein distance, the Dudley metric, and the Maximum Mean Discrepancy~\citep{sriperumbudur2009integral}. We note that the KL divergence is not a special case of an IPM.

The following theorem is our main result for IPMs. 
Just as Theorem~\ref{thm:pac-bayes-bounded-proc} provided a general framework for generating PAC-Bayes bounds with a KL-divergence term, Theorem~\ref{thm:anytime-ipm} provides a  framework for generating PAC-Bayes bounds with an IPM. The main idea is to replace the use of the Donsker-Varadhan formula with an assumption on the family of functions $\cG:\Theta\to\R$ (or, more precisely, \emph{families} of functions). 

\begin{theorem}
\label{thm:anytime-ipm}
    Let $(\cG_t)_{t\geq 1}$ be a predictable sequence, where each $\cG_t$ is a family of functions from $\Theta\to\R$. Let $(h_t)$ be a sequence of functions such that $h_t\in \cG_t$ for all $t\geq t_0$. Suppose that 
    $(\exp h_t(\theta))_{t\geq t_0}$ is a supermartingale or reverse submartingale 
    %\edit{on a filtered probability space $(\Omega, \mathcal{D},(\mathcal{G}_t),\Pr)$}
    (adapted to some filtration) 
    for all $\theta\in\Theta$ such that $\E_\dist \exp h_{t_0}(\theta)\leq 1$.  
    Then, for any $\delta\in(0,1)$ and prior $\nu\in\Mspace{\Theta}$, with probability at least $1-\delta$, %\edit{with respect to $\mathcal{D}$}, 
    \begin{equation}
    \label{eq:ipm-general-bound}
        \E_{\theta\sim\rho} h_t(\theta) \leq \gamma_{\cG_t}(\rho,\nu) + \log(1/\delta),
    \end{equation}
    for all $\rho\in\Mspace{\Theta}$ and times $t\geq t_0$. 
\end{theorem}
\begin{proof}
    By assumption, $h_t\in\cG_t$ for all $t$. Hence $\gamma_{\cG_t}(\rho,\nu)\geq \E_\rho h_t(\theta) - \E_\nu h_t(\theta)$. Rearranging and exponentiating gives 
    \begin{equation*}
        \exp(\E_\rho h_t(\theta) - \gamma_{\cG_t}(\rho,\nu)) \leq \exp \E_\nu h_t(\theta) \leq \E_\nu \exp h_t(\theta).
    \end{equation*}
    Since $(\exp h_t(\theta))_{t\geq t_0}$ is a super or submartingale by assumption and $\nu$ is data-free, the process $(\E_\nu \exp h_t(\theta))_{t\geq t_0}$ is also a super or submartingale by Lemma~\ref{lem:mixtures}. Therefore, Ville's inequality gives 
    \begin{align*}
        \Pr\big(\exists t\geq t_0: \exp\big\{\E_\rho h_t(\theta) - \gamma_{\cG_t}(\rho,\nu)\big\} \geq 1/\delta\big) \leq \Pr\big(\exists t\geq 1: \E_\nu \exp h_t(\theta) \geq 1/\delta\big) \leq \delta. 
    \end{align*}
    Since $\rho$ was arbitrary, this yields that with probability at least $1-\delta$,
    \[\exp\big\{\E_\rho h_t(\theta) - \gamma_{\cG_t}(\rho,\nu)\big\}\leq 1/\delta,\]
    for all $t\geq t_0$ and $\rho$. Rearranging gives the desired result. 
\end{proof}

Following \citet{amit2022integral}, we let the family of functions $\cG_t$ be a function of the timestep (hence possibly dependent on data $Z_1,\dots,Z_t$).  Sections~\ref{sec:supermart_bounds} and \ref{sec:submart_bounds} are replete with processes $(\exp h_t(\theta))$ which are super and submartingales, each of which furnishes a separate bound after applying Theorem~\ref{thm:anytime-ipm}. We will not list them all here, trusting that practitioners can combine results as befits their problem of interest. We will, however, state the following consequence of Theorem~\ref{thm:anytime-ipm} in order to compare our results with those of \citet{amit2022integral}. In what follows, we use notation and concepts introduced in Section~\ref{sec:submart_bounds}, such as 
$\bar{t} = 2^{\lfloor \log_2(t)\rfloor}$, $\il_t = \log(\log_2(2t)\zeta(2))$, 
$\vp_t(\theta) = \vp(\hR_t(\theta),\risk(\theta))$, 
 and the exchangeable reverse filtration $(\mE_t)$. 
 We also assume a stationary loss function. 

\begin{corollary}
\label{cor:ipm-convex}
    Let $(Z_t)$ be exchangeable, and let $\vp:\R_{\geq 0}\times \R_{\geq 0}\to\R$ be a convex function. Fix a prior $\nu\in\Mspace{\Theta}$. Consider a family of functions $(\cG_t)$ with $\cG_t:\Theta\to\R$ and let $(\lambda_t)$ be a positive sequence such that, for all natural numbers $k\geq 0$,
    \[\lambda_{2^k} \vp_t(\theta) - \log \E_\dist\exp(\lambda_{2^k} \vp_{2^k} (\theta)) \in \cG_t, \quad \text{for all } t\geq 1.\]
    Then, for all $\delta\in(0,1)$, with probability at least $1-\delta$ \edit{over the random draw of $(Z_t)$}, for all $\rho\in\Mspace{\Theta}$ and at all times $t\geq 1$, 
    \begin{equation}
    \label{eq:ipm-convex}
        \E_{\rho} \vp_t(\theta) \leq \frac{\log\E_{\edit{\nu}}\E_\dist \exp(\lambda_{\bar{t}} \vp_{\bar{t}}(\theta))}{\lambda_{\bar{t}} } + \frac{\gamma_{\cG_t}(\rho,\nu) + \log(1/\delta) + \il_t}{\lambda_{\bar{t}}}.
    \end{equation}
    Moreover, suppose $n$ is some fixed time of interest, and that $\lambda \vp_t(\theta) - \log \E_\dist \exp(\lambda \vp_n(\theta))\in \cG_t$ for all $t\geq n$ and some $\lambda>0$. Then, for all $\delta\in(0,1)$, with probability at least $1-\delta$ \edit{over the random draw of $(Z_t)$}, for all $t\geq n$: 
    \begin{equation}
    \label{eq:ipm-convex-target}
        \E_{\rho} \vp_t(\theta) \leq \frac{\log\E_{\edit{\nu}}\E_\dist \exp(\lambda\vp_n(\theta))}{\lambda} + \frac{\gamma_{\cG_t}(\rho,\nu) + \log(1/\delta)}{\lambda}.
    \end{equation}
\end{corollary}

A proof sketch is provided in Appendix~\ref{app:proof-ipm-convex}. 
The previous result parallels Corollaries \ref{cor:anytime-convex} and \ref{cor:anytime-convex-target} but using IPMs instead of the KL divergence. 
The reliance of \eqref{eq:ipm-convex} on $\bar{t}$ and $\il_t$ once again arises from stitching. For the fixed time $t=n$, \eqref{eq:ipm-convex-target} gives a generalized version of Corollaries 4 and 5 in \cite{amit2022integral}. Those results are obtained by considering particular functions $\vp$, as was done in Section~\ref{sec:submart-seeger}. As noted by \citet{amit2022integral}, the above bounds are merely ``templates'' in the sense that, to be insightful, one  must choose a family of functions $\cG_t$.  A bound based on the total variation distance can be achieved by considering the family $\cG_t = \{g:\Theta\to[0,\infty):\|g\|_\infty\leq 1\}$, and one based on the Wasserstein distance can be achieved by appealing to Kantorovich-Rubinstein duality. We refer the reader to \citet{amit2022integral} for the details of these bounds.

\subsection{$\phi$-divergences and R\'{e}nyi divergences}
\label{sec:phi-div}
The KL divergence is a member of a more general family of divergences termed $\phi$-divergences \citep{ali1966general} (often called $f$-divergences, but we have reserved $f$ for our loss). For a convex function $\phi:\R\to\R$, the $\phi$-\emph{divergence} between measures $\rho$ and $\nu$ over $\Theta$ such that $\rho\ll \nu$  is 
\begin{equation}
\label{eq:phi-div}
    D_\phi(\rho\|\nu) = \int_\Theta  \phi\bigg(\frac{\d\rho}{\d\nu}\bigg)\d\nu =  \E_{\theta\sim\nu}\bigg[\frac{\d\rho}{\d\nu}(\theta)\bigg].
\end{equation}
The KL divergence is recovered by considering $\phi(x) = x\log x$. $\phi$-divergences are a nearly orthogonal set of divergences from IPMs, considered in the previous section. Indeed, the total variation distance is the only (non-trivial) divergence which is both an IPM and a $\phi$-divergence~\citep{sriperumbudur2012empirical}. 

The Donsker-Varadhan formula for the KL divergence is an improvement on a more general variational representation of $\phi$-divergences (e.g., \citealp{sriperumbudur2009integral}), which states the following. For any measures $\rho$ and $\nu$ and any convex function $\phi:\R\to\R$,  
 \begin{equation}
 \label{eq:phi-div-variational}
     D_\phi(\rho\|\nu) \geq \E_\rho [h(\theta)] - \E_\nu[\phi^*(h(\theta))], 
 \end{equation}
 where $\phi^*$ is the convex conjugate of $\phi$, i.e., 
 \begin{equation*}
     \phi^*(y) = \sup_{x\in\R}\{xy - \phi(x)\}.
 \end{equation*}
 We can use \eqref{eq:phi-div-variational} in place of the Donsker-Varadhan formula in Theorem~\ref{thm:pac-bayes-bounded-proc}, where the term $\E_\nu \phi^*(h(\theta)) $ replaces $\log \E_\nu \exp h(\theta)$.

\begin{theorem}[Anytime PAC-Bayes with $\phi$-divergences]
\label{thm:pac-bayes-phi-div}
Let $\phi:\R\to\R$ be a convex function. 
Let $P(\theta) = (P_t(\theta))_{t=1}^\infty$ be a stochastic process such that, for all $\theta\in\Theta$, $\exp \E_\nu [\phi^*(P_t(\theta))]$ is a supermartingale or reverse submartingale 
%\edit{on a filtered probability space $(\Omega,\mathcal{D},(\mathcal{G}_t),\Pr)$}.
(adapted to some underlying filtration). 
Suppose that $\exp \E_\nu [\phi^*(P_1(\theta))]\leq 1$. Then, for any $\delta\in(0,1)$ and prior $\nu\in\Mspace{\Theta}$, with probability at least $1-\delta$, %\edit{with respect to $\mathcal{D}$}, 
\begin{equation}
    \E_\rho P_t(\theta) \leq D_\phi(\rho\|\nu) + \log(1/\delta),
\end{equation}
for all times $t\geq 1$ and $\rho\in\Mspace{\Theta}$. 
\end{theorem}
\begin{proof} 
    Set  
    $V_t^{\mix}:=\exp\sup_\rho\big\{ \E_{\theta\sim\rho}[P_t(\theta)] - D_\phi(\rho\|\nu)\big\}$. 
    The variational formula for $D_\phi$ gives 
    $V_t^{\mix} \leq \exp \E_\nu [\phi^*(P_t(\theta))]$,
    so by assumption, $V_t^{\mix}$ is upper bounded by a nonnegative supermartingale or reverse submartingale. 
    From here, the proof follows that of Theorem~\ref{thm:pac-bayes-bounded-proc}. 
\end{proof}

The key distinction between this result and Theorem~\ref{thm:pac-bayes-bounded-proc} is that while the latter posits that $\exp P(\theta)$ is (upper bounded by) a nonnegative super/submartingale, here we assume that $\exp \E_\nu [\phi^*(P(\theta)]$ plays this role. 
We consider establishing functions $\phi$ and processes $P(\theta)$ such that $\exp \E_\nu\phi^*(P(\theta))$ has this property to be an interesting line of future research. 
We note that Theorem~\ref{thm:pac-bayes-phi-div} cannot strictly be called a generalization of Theorem~\ref{thm:pac-bayes-bounded-proc} as the latter relies on the Donsker-Varadhan formula which is tighter than the variational formula for the KL divergence given by \eqref{eq:phi-div-variational}.

Another (related) family of distances is the \emph{R\'{e}nyi divergence}. Here, for measures $\rho\ll\nu$ and any $\alpha\in(0,1)\cup(1,\infty)$, we define 
\begin{equation*}
    D_\alpha(\rho\|\nu) := \frac{1}{1-\alpha}\E_{\theta\sim\rho} \bigg[\bigg(\frac{\rho(\theta)}{\nu(\theta)}\bigg)^\alpha\bigg].
\end{equation*}
As $\alpha\to1$, $D_\alpha(\rho\|\nu)\to\kl(\rho\|\nu)$, so by continuity we define $D_1(\rho\|\nu) = \kl(\rho\|\nu)$. The R\'{e}nyi divergence yields the following variational formula, which can be seen as an extension of the Donsker-Varadhan formula (Lemma~\ref{lem:change-of-measure}). It was given by \citet{begin2016pac}. 

\begin{lemma}
    \label{lem:renyi-variational}
    Let $h:\Theta\to\R$ be measurable. For any measures $\rho$ and $\nu$, with $\rho\ll\nu$, we have 
    \begin{equation*}
        \log \E_\nu [h(\theta)^{\frac{\alpha}{\alpha-1}}] \geq \frac{\alpha}{\alpha-1}\log \E_\rho h(\theta) - D_\alpha(\rho\|\nu),
    \end{equation*}
    for all $\alpha\in(0,1)\cup(1,\infty)$. 
\end{lemma}

Using this formula, one can give a Theorem in the style of Theorem~\ref{thm:pac-bayes-bounded-proc} and \ref{thm:pac-bayes-phi-div} for $\alpha$-divergences.  

\begin{theorem}
    \label{thm:anytime-alpha-div}
    Set $\alpha>1$. 
Let $P(\theta) = (P_t(\theta))_{t\geq t_0}$ be a stochastic process such that, for all $\theta\in\Theta$, $\exp(P^{\frac{\alpha}{\alpha-1}}(\theta))$ is a supermartingale or reverse submartingale 
%\edit{on a filtered probability space $(\Omega, \mathcal{D},(\mathcal{G}_t),\Pr)$},
(adapted to some underlying filtration) 
obeying $\E_\dist \exp P_{t_0}^{\frac{\alpha}{\alpha-1}}(\theta)\leq 1$. Then, for any $\delta\in(0,1)$ and prior $\nu\in\Mspace{\Theta}$, with probability at least $1-\delta$, 
%\edit{with respect to $\mathcal{D}$}, 
\begin{equation}
  \E_\rho [P_t(\theta)] \leq \frac{\alpha-1}{\alpha}(D_\alpha(\rho\|\nu) + \log(1/\delta)),
\end{equation}
for all times $t\geq t_0$ and $\rho\in\Mspace{\Theta}$.  
\end{theorem}
\begin{proof}
    Following Theorems~\ref{thm:pac-bayes-bounded-proc} and \ref{thm:pac-bayes-phi-div}, put $V_t^\mix = \exp \sup_\rho \{\frac{\alpha}{\alpha-1} \log \E_\rho \exp P_t(\theta) - D_\alpha(\rho\|\nu)\}$. Then $V_t^\mix \leq \E_\nu \exp P_t^{\frac{\alpha}{\alpha-1}}(\theta)$ by Lemma~\ref{lem:renyi-variational}, where the process $(\E_\nu \exp P_t^{\frac{\alpha}{\alpha-1}}(\theta))_{t\geq t_0}$ is a nonnegative supermartingale or reverse submartingale by assumption and Lemma~\ref{lem:mixtures}. It also has initial expected value at most 1 by assumption. Therefore, $\Pr(\exists t\geq t_0: V_t^\mix \geq 1/\delta) \leq \Pr(\exists t\geq t_0: \E_\nu \exp P_t^{\frac{\alpha}{\alpha-1}}(\theta)\geq 1/\delta)\leq \delta$ by Ville's inequality. Rearranging the inequality $V_t^\mix \leq 1/\delta$, we obtain that with probability at least $1-\delta$, 
    \[\E_\rho P_t(\theta) \leq \log \E_\rho \exp P_t(\theta) \leq \frac{\alpha-1}{\alpha}\bigg(D_\alpha(\rho\|\nu) + \log(1/\delta)\bigg),\]
    for all $\rho$ and $t\geq t_0$, as claimed. 
\end{proof}

Theorem~\ref{thm:anytime-alpha-div} suggests the question: When is $\exp(P^{\frac{\alpha}{\alpha-1}}(\theta))$ a supermartingale or reverse submartingale? 
There are several candidates. 
By Jensen's inequality, a sufficient condition for this quantity to be a reverse submartingale is for $P(\theta)$ to also be a reverse submartingale. Indeed, if $(N_t)$ is a reverse submartingale with respect to $(\cR_t)$, then 
\begin{equation}
\label{eq:alpha-submart}
  \E[N_t^{\frac{\alpha}{\alpha-1}}|\cR_{t+1}] \geq \E[N_t|\cR_{t+1}]^{\frac{\alpha}{\alpha-1}} \geq N_{t+1}^{\frac{\alpha}{\alpha-1}},  
\end{equation}
since $x\mapsto x^{\frac{\alpha}{\alpha-1}}$ is convex. However, to apply Ville's inequality, one would also need to control $\E_\dist N_1^{\frac{\alpha}{\alpha-1}}$ which is less easily done, even if $\E_\dist N_1\leq 1$. One might also  consider 
using the processes employed in the proof of Corollary~\ref{cor:anytime-convex}, but raised to the $(\alpha-1)/\alpha$. In that case, of course, raising the result to the $\alpha/(\alpha-1)$ power would result in the original process. 
However, in this case we achieve the same bound as Corollary~\ref{cor:anytime-convex}, but with $\kl(\rho\|\nu)$ replaced by $D_\alpha(\rho\|\nu)$. This a weaker result since $D_\alpha(\rho\|\nu)\geq \kl(\rho\|\nu)$ for all $\alpha>0$.   Instead, to take advantage of Lemma~\ref{lem:renyi-variational}, we construct an altogether different process. This leads to the following result.  As in Section~\ref{sec:submart_bounds} we consider a stationary loss function and exchangeable data. Recall the shorthand $\vp_t(\theta) = \vp(\hR_t(\theta),\risk(\theta))$ for a convex function $\vp$, as well as the quantities $\bar{t} = 2^{\lfloor \log_2(t)\rfloor}$ and $\il_t = \log(\log_2^2(2t)\zeta(2))$.  

\begin{corollary}
\label{cor:renyi-convex}
Let $(Z_t)$ be exchangeable. 
		Let $\vp:\R_{\geq 0}\times \R_{\geq 0}\to\R_{>0}$ be a convex function and $\nu\in\Mspace{\Theta}$ be a prior. 
  Put $\alpha>1$. 
		Then, for all $\delta\in(0,1)$, with probability at least $1-\delta$ \edit{over the random draw of $(Z_t)$}, for all $\rho\in\Mspace{\Theta}$ and at all times $t\geq 1$, 
  		\begin{align}
  \label{eq:renyi-convex}
  \log \E_\rho \vp_t(\theta) \leq \frac{\alpha-1}{\alpha}\bigg(D_\alpha(\rho\|\nu) + \log \E_{\nu,\dist}[\vp_{\bar{t}}^{\frac{\alpha}{\alpha-1}}(\theta)] + \log(1/\delta) + \il_t\bigg).
		\end{align}
\end{corollary}

The proof is provided in Appendix~\ref{app:proof-renyi-convex}. Similarly to Corollary \ref{cor:anytime-convex-target}, we can obtain a version of the above result which holds for all times $t\geq n$ for some pre-selected time $n$. These results constitute a time-uniform extension of Theorem 9 in \citet{begin2016pac}, who give a fixed-time version for binary classification. By taking $\alpha=2$, we obtain a PAC-Bayes bound using the $\chi^2$ divergence (see \citealp[Corollary 10]{begin2016pac}). We note that unlike Corollary~\ref{cor:anytime-convex}, the above result is a bound on the logarithm of $\vp$. By exponentiating both sides, one obtains an intriguing PAC-Bayes bound in multiplicative form.

\subsection{Confidence Sequences and Choice of $(\lambda_t)$}
\label{sec:confseq}

Our anytime-valid bounds enable us, under some circumstances, to construct time-uniform \emph{confidence sequences}, i.e., 
sequences of sets which contain the true parameter of interest at all times with high probability~\citep{darling1967confidence,lai1976confidence}.
In our setting, the parameter of interest is the conditional mean $\frac{1}{t}\sum_{i=1}^t 
\E_{\theta\sim\rho} \mu_i(\theta)$, where $\mu_i(\theta) = \E_\dist [f_i(Z_i,\theta)|\F_{i-1}]$. 
A $(1-\delta)$-confidence sequence is then a random sequence $(C_t(\rho,\nu))_{t=1}^\infty$ 
such that 
\begin{equation}
    \Pr\bigg(\forall t\geq 1: \frac{1}{t}\sum_{i=1}^t \E_{\theta\sim\rho} \mu_i(\theta) \in C_t(\rho,\nu)\bigg) \geq 1-\delta.
\end{equation}
Observe that the confidence sequence depends on the prior $\nu$ and posterior $\rho$. It does not hold simultaneously across all such distributions. 

While we allow the conditional mean $t^{-1}\sum_{i\leq t} \mu_i(\theta)$ to change over time in general, let us begin the discussion with the case of a common conditional mean and stationary loss function $f$. More precisely, we assume that $\mu(\theta) = \mu_t(\theta) = \E_\dist[f(Z_t,\theta)|\F_{t-1}]$ is unchanging as a function of time. 
Many of the bounds generated in previous sections are based on processes which are themselves based on tail bounds on the term $\lambda \Delta_i(\theta) = \lambda(\mu_i(\theta) - f(Z_i,\theta))$. By considering $-\Delta_i(\theta)$ and applying the union bound, we may obtain a confidence sequence. For instance, the following confidence sequence may be derived from Corollary~\ref{cor:anytime-subGaussian-losses}. 

\begin{corollary}
\label{cor:confseq-subgaussian}
    Let $f$ be $\sigma$-subGaussian and  
    let $(Z_t)\sim\dist$ be such that $\mu(\theta) = \E_\dist [f(Z,\theta)|\F_{t-1}]$ is constant for all $t\geq 1$. Fix a prior $\nu\in\Mspace{\Theta}$. 
    Then, for all $\delta\in(0,1)$, with probability at least $1-\delta$ \edit{over the random draw of $(Z_t)$}, for all $\rho$ and $t\geq 1$,
    \begin{equation}\label{eqn:subG-cs}
        \E_\rho \mu(\theta) \in \bigg(\frac{\sum_{i=1}^t \lambda_i f(Z_i,\theta)}{\sum_{i=1}^t \lambda_i} \pm W_t\bigg), \quad \text{where}\quad W_t := \frac{\log(2/\delta) + \kl(\rho\|\nu) + \frac{\sigma^2}{2}\sum_{i=1}^t\lambda_i^2}{\sum_{i=1}^t \lambda_i} 
    \end{equation}
\end{corollary}
We note the factor of 2 in $\log(2/\delta)$ comes from the union bound. We state the above proposition as an example only; many other confidence sequences may be derived from the arguments throughout Sections~\ref{sec:supermart_bounds} and \ref{sec:submart_bounds}. 

Studying confidence sequences provides an opportunity to demonstrate why we allow $\lambda_t$ to change as a function of time. It is desirable that the width of the sequence, $W_t$, goes to 0 as $t\to\infty$ so that the confidence sequence asymptotically converges on the correct value with high probability.  This would not be possible with fixed $\lambda$, as $W_t$ would converge to $\sigma^2\lambda/2\neq 0$. On the other hand, following \citet{waudby2023estimating}, if we instead consider  
$\lambda_t \asymp (t\log t)^{-1/2}$, then we have $W_t = \widetilde{O}(\sqrt{\log(t)/t})$, where $\widetilde{O}$ hides log-log factors. 
Further, we can attain the optimal rate ${O}(\sqrt{\log \log t / t})$ due to the Law of the Iterated Logarithm (LIL) \citep{darling1967iterated} by the same technique of geometrically spaced union bounds that was used in Section~\ref{sec:bound-convex-functional}. Such a result applied to Corollary~\ref{cor:confseq-subgaussian} is stated and proved in Appendix~\ref{app:proof-conseq-stitch}, but is omitted here in favor of the following discussion which is more general and also provides a LIL bound.

Let us turn now to the case when $\mu_t(\theta)$ is not assumed to be independent of $t$. 
Similarly to Corollary~\ref{cor:confseq-subgaussian}, a union bound applied to Corollary~\ref{cor:anytime-subGaussian-losses} tells us that 
\[\sum_{i=1}^t \lambda_i\E_\rho \mu_i(\theta) \in \bigg( \sum_{i=1}^t \frac{\lambda_i^2\sigma_i^2}{2} \pm  [\kl(\rho||\nu) + \log(2/\delta)]\bigg),\]
for all $t\geq 1$ with probability at least $1-\delta$. However, this does not yield a closed-form expression for a confidence sequence. To construct an explicit confidence sequence with optimal width, we turn once again to stitching. The technique we use is applicable to general sub-$\psi$ processes (Section~\ref{sec:sub-psi-condition}), but we demonstrate it in the case of $1$-subGaussian losses for simplicity. 

\begin{corollary}
\label{cor:time-varying-stitch}
    Let $f_i$ be $1-$subGaussian and fix a prior $\nu\in\Mspace{\Theta}$. Then, for all $\delta\in(0,1)$, with probability at least $1-\delta$ \edit{over the random draw of $(Z_t)$}, for all $\rho$ and $t\geq 1$: 
    \begin{equation*}
        \frac{1}{t}\sum_{i=1}^t \E_\rho \mu_i(\theta) \in \bigg(\frac{1}{t}\sum_{i=1}^t \E_\rho f_i(Z_i,\theta) \pm W_t\bigg),
    \end{equation*}
    where 
    \begin{equation*}
        W_t \lesssim  \frac{\sqrt{\log(\log(t)) + \log(1/\delta)}}{\sqrt{t}} + \frac{\kl(\rho\|\nu)}{\sqrt{t\log(\log(t)) + t\log(1/\delta)}}.
    \end{equation*}
\end{corollary}

The proof can be found in Appendix~\ref{app:proof-time-varying-stitch}. There has been much recent work on developing sequences $(\lambda_t)$ which achieve optimal shrinkage rates; we refer the interested reader to \citet{catoni2012challenging,howard2021time,waudby2023estimating,wang2022catoni,wang2023huber} for further discussion on this point. 

We end this section by noting that we have now deployed the stitching technique in two capacities. In Section~\ref{sec:submart_bounds} it was used to apply a different reverse submartingale in each epoch, whereas in the above result it was used to choose appropriate constants in each epoch. While the intuition behind stitching is similar, the two applications yield different results. The former loses some tightness compared to fixed-time bounds, while the latter enables us to achieve optimal rates.

\subsection{Martingale Difference Sequences}
 \label{sec:martingale-differences} 

 Throughout this work we've considered loss functions $f_t$ acting on $\cZ$ and $\Theta$. While this is a natural setting for PAC-Bayes analysis owing to its connections to learning theory, different settings have been considered. \citet{seldin2012pac} and \citet{balsubramani2015pac}, for instance, consider PAC-Bayesian inequalities for martingale difference sequences. 
 In this section we briefly demonstrate that our results extend to this setting. This is due to the fact that our workhorse, Theorem~\ref{thm:pac-bayes-bounded-proc}, holds for general stochastic processes. 

We consider a sequence of random functions $(F_t)$ such that $F_t:\Theta\to\R$. We suppose that $(F_t)$ is a \emph{martingale difference sequence}, i.e.,  $\E[F_t|\F_{t-1}]=0$ for all $t\geq 1$, 
 where $\F_t=\sigma(F_1,\dots,F_t)$. That is, $\E[F_t(\theta)|\F_{t-1}]=0$ for all $\theta\in\Theta$. 
 Note that the  expectation is over the functions themselves, not over $\theta$. Let $S_t = \sum_{i=1}^t F_i$ (and, by extension, $S_t(\theta) = \sum_{i=1}^t F_i(\theta)$).

First, suppose the $F_t$ are bounded, say $F_t: \Theta\to [\alpha_t,\beta_t]$. 
 Just as we did in Corollary~\ref{cor:anytime-subGaussian-losses}, we can consider the nonnegative process $N_t(\theta) = \exp\big\{\sum_{i=1}^t \lambda_i F_i(\theta) - \frac{1}{8}\sum_{i=1}^t\lambda_i^2 (\beta_i-\alpha_i)^2\big\}$, which is a supermartingale since $\E[F_i|\F_{i-1}]=0$. (Note that we have substituted $(\beta_i-\alpha_i)^2/4$ for $\sigma_i^2$ in \eqref{eq:subgaussian-supermart}, since $F_i$ is $(\beta_i-\alpha_i)/2$-subGaussian.) This process, in conjunction with Theorem~\ref{thm:pac-bayes-bounded-proc}, leads to the following result, which is the time-uniform extension of Theorem 5 of \citet{seldin2012pac}. 
 
 \begin{corollary}[Anytime bound for bounded MDSs I]
 \label{cor:mds-bounded}
     Let $(F_t)$ be a martingale difference sequence where $F_t:\Theta\to[\alpha_t,\beta_t]$. Let $\nu\in\Mspace{\Theta}$ be a prior and $(\lambda_t)$ a nonnegative predictable sequence. Then, for all $\delta\in(0,1)$, with probability at least $1-\delta$ \edit{over the sequence of functions}, for all $t\geq 1$ and $\rho\in\Mspace{\Theta}$, 
     \begin{equation*}
         \sum_{i=1}^t \lambda_i \E_\rho F_i(\theta) \leq \frac{1}{8}\sum_{i=1}^t \lambda_i^2 (\beta_i-\alpha_i)^2 + \kl(\rho\|\nu) + \log(1/\delta).
     \end{equation*}
 \end{corollary}

Using similar techniques, we can provide a time-uniform version of Theorem 7 in \citet{seldin2012pac}, a result which also undergirds the main theorem of \citet{balsubramani2015pac}.

\begin{corollary}[Anytime bound for bounded MDSs II]
\label{cor:mds-bernstein}
   Let $(F_t)$ be a martingale difference sequence where $F_t:\Theta\to\R$ such that $|F_t(\theta)|\leq H$ for all $\theta\in\Theta$. Let $\nu\in\Mspace{\Theta}$ be a prior and $\lambda\in[0,1/H]$. Then, for all $\delta\in(0,1)$, with probability at least $1-\delta$ \edit{over the sequence of functions},
     \begin{equation*}
     \sum_{i=1}^t F_i(\theta) \leq \lambda(e-2)\sum_{i=1}^t\E[F_i^2(\theta)|\F_{i-1}] + 
         \frac{\kl(\rho\|\nu) + \log(1/\delta)}{\lambda},
     \end{equation*}
     for all $t\geq 1$ and $\rho\in\Mspace{\Theta}$. 
\end{corollary} 
Note that because $F_t$ is bounded, all (conditional) moments exists. The bound is therefore non-vacuous by assumption. 
The proof of the above result (and the statement itself) is very similar to that of Corollary~\ref{cor:bounded-bernstein}, and is thus omitted. Like that proposition, here $\lambda$ could be taken to be a sequence $\{\lambda_t\}\subset [0,1/H]$, but we leave it stationary for easier to comparison to prior work.  

Theorem 1 of \citet{balsubramani2015pac} is based on Corollary~\ref{cor:mds-bernstein} and then choosing $\lambda$ strategically (and stochastically) in order to tighten the bound. Such techniques have also been used to generate sharp martingale concentration bounds. 
\citet{seldin2012pac} also optimize over $\lambda$ in the fixed-time version of Corollary~\ref{cor:mds-bounded}
in order to provide a  tighter bound (see their Theorems 5 and 6). An anytime version of this result would follow from applying the same procedure to Corollary~\ref{cor:mds-bounded}, though we note that their optimization procedure employs knowledge of the sample size and thus cannot be replicated precisely in the anytime setting.   

% We note that is tempting to improve Corollary~\ref{cor:mds-bernstein} by using the conditional variance process underlying Corollary~\ref{cor:anytime-second-moment}. This would require that $F_i\geq 0$, however, which is impossible for a nontrivial martingale difference sequence. 

Our final result generalizes Theorem 4 of \citet{seldin2012pac}, by providing a version of Corollary~\ref{cor:anytime-convex} for difference sequences. 
Here we will broaden the setting slightly from martingale difference sequences and let $\E[F_t|\F_{t-1}]=G$ for all $t\geq 1$  some $G:\Theta\to\R$, meaning that $\E[F_t(\theta)|\F_{t-1}]=G(\theta)$ for all $\theta\in\Theta$. The proof of the following bound uses precisely the same mechanics as that of Corollary~\ref{cor:anytime-convex}, so we do not provide it. Recall that $\bar{t} = 2^{\lfloor \log_2(t)\rfloor}$ and $\il_t = \log(\log_2^2(2t)\zeta(2))$.

\begin{corollary}
    \label{cor:mds-convex} 
    Let $(F_t)$ be a random exchangeable sequence of functions with $F_t:\Theta\to\R$ such that $\E[F_t|\F_{t-1}]=G$ for all $t\geq 1$ and some fixed $G:\Theta\to\R$. 
    Let $\vp:\R_{\geq 0}\times \R_{\geq 0}\to\R$ be a convex function and set $S_t = \sum_{i=1}^t F_i$. Fix a prior $\nu\in\Mspace{\Theta}$ and let $(\lambda_t)$ be a positive sequence of real numbers. 
    Then, for all $\delta\in(0,1)$, with probability at least $1-\delta$ \edit{over the sequence of functions}, for all $\rho\in\Mspace{\Theta}$ and at all times $t\geq 1$, 
    \begin{equation*}
        \E_{\rho} \vp\bigg( \frac{1}{t}S_t(\theta), G(\theta)\bigg) \leq \frac{\log\E_{\edit{\nu}}\E\exp(\lambda_{\bar{t}} \vp(\frac{1}{\bar{t}}S_{\bar{t}}(\theta), G(\theta)))}{\lambda_{\bar{t}} } + \frac{\kl(\rho\|\nu) + \log(1/\delta)+\il_t}{\lambda_{\bar{t}} }.
    \end{equation*}
\end{corollary}

A time-uniform version of Theorem 4 of \cite{seldin2012pac} follows from the above bound by taking $\vp=\klsf$ (and taking $\lambda_t=\lambda$ for all $t$) as was done in both Sections \ref{sec:submart-seeger} and \ref{sec:beyond-kl}. \edit{Finally, we note that \citet{kuzborskij2019efron} also discuss bounds based on martingale difference sequences. In particular, they use the Doob decomposition to construct a canonical difference sequence when estimating a general function, and then bound the increments using an empirical Efron-Stein like term. We discuss their work more thoroughly in Section~\ref{sec:supermart_bounds}, where we relate it to sub-$\psi$ processes.}

\edit{
\subsection{Data-dependent Priors}
\label{sec:data-dependent-priors}
Several recent works have investigated the role of data-dependent priors \citep{rivasplata2020pac,awasthi2020pac}. The appeal is clear: A well chosen prior with mass close to the true parameter will typically enable much tighter bounds. Historically, this is often achieved with sample splitting, i.e., reserving some fraction of the sample to choose the prior, and then computing the bound on the remaining data~\citep{parrado2012pac,dziugaite2017computing}. Here we provide some remarks on how to extend our analysis to allow for data-dependent priors. 

To set the stage, let $P(\theta)=(P_t(\theta))_{t\geq 1}$ be a stochastic process such that $\exp P(\theta)$ is a supermartingale. 
In our results thus far, the prior $\nu$ is data-free to ensure that Fubini-Tonelli can be applied so that the mixture $\E_{\theta\sim\nu} [\exp P(\theta)]$ is also a supermartingale. 
We can, however, weaken this condition slightly. 
If $\nu$ is $\F_{t_0-1}$-measurable, then the process $(\E_{\theta\sim\nu}\exp P_t(\theta))_{t\geq t_0}$ 
remains a supermartingale, since $\nu$ is deterministic 
at time $t_0$. 
%The logic of Lemma~\ref{lem:mixtures} still holds (for times $t \geq t_0$). 
Of course, different priors result in different processes: for $\nu_i\in\Mspace{\Theta}$ which is $\F_{t_i-1}$-measurable, the process $(\E_{\nu_1} \exp U_t(\theta))_{t\geq t_1}$ can be distinct from $(\E_{\nu_2} \exp U_t(\theta))_{t\geq t_2}$. 
Thus, a bound which covers changing priors must cover different processes. We can ensure such coverage but at the price of a union bound.

More formally, let $\cT=\{t_1,t_2,\dots,t_N\}$, $N\in \mathbb{N}\cup\{\infty\}$ be a set of times at which we change priors. We allow the times $t_i$ to be stopping times (i.e., $\{t_i\leq n\}$ is $\F_n$-measurable for all $n$), and we allow $|\cT|$ to be infinite (e.g., $t_i = 2^i$ is an example of a deterministic set of times satisfying the condition). 
Suppose we begin with prior $\nu_0$. At time $t_i\in \cT$ a new prior $\nu_i$ is used, where $\nu_{i}$ is $\F_{t_i-1}$-measurable. 
Define, for each $\theta\in\Theta$, 
\begin{equation}
\label{eq:process_by_cases}
    \mathfrak{M}_t(\theta) = 
    \begin{cases}
    P_t(\theta) - \kl(\rho\|\nu_0),& 1\leq t<t_1, \\
    P_t(\theta) - \kl(\rho\|\nu_1),& t_1\leq  t<t_2,\\
    \vdots & \vdots 
    %P_t(\theta) - \kl(\rho\|\nu_N),& t\geq t_N\\
    \end{cases}
\end{equation}
When $\cT=\emptyset$, a bound on $\E_{\theta\sim\rho} \fM_t(\theta)$ 
for all $\rho$ and $t\geq 1$ is given by Theorem~\ref{thm:pac-bayes-bounded-proc}. 
The following theorem generalizes this result and provides a bound for an arbitrary number of priors. 
However, we pay a price (a loosening of the bound) for each change of prior. 
Interestingly, our result does not allow for the stochastic process to be a reverse submartingale, only a forward supermartingale. This is because, if $(\cR_t)$ is a reverse filtration and $\nu_i$ is $\cR_{t_0-1}$-measurable, it may not be $\cR_t$ measurable for $t\geq t_0$ since $\cR_t\supset \cR_{t+1}$. Thus, the mixtures may cease to be submartingales as time advances.  

\begin{theorem}
\label{thm:data-dependent}
     For each $\theta\in\Theta$, let 
     $P(\theta) = (P_t(\theta))_{t=1}^\infty$ be a stochastic process such that $\exp P(\theta)$ is a supermartingale adapted to the filtration $(\F_t)$ and $\E\exp P_1(\theta)\leq 1$. Let $\cT$ and $\fM_t$ be as above. 
 Then, for any $\delta\in(0,1)$, with probability at least $1-\delta$, for all $t\geq 1$ and $\rho\in\Mspace{\Theta}$,
\begin{equation}
\label{eq:data-dependent-bound}
\E_\rho[\mathfrak{M}_t(\theta)]\leq \log\left(\frac{(t^\dagger+1)(t^\dagger + 2)}{\delta}\right),
\end{equation}
where $t^\dagger$ is the number of times the prior has been changed up to and including time $t$. 
\end{theorem}

\begin{proof}
Let us briefly clarify notation. We begin with prior $\nu_0$, and switch to $\nu_1$ at time $t_1$, switch to $\nu_2$ at time $t_2$, and so on. We set $t_0=1$ for convenience. The optional stopping theorem for nonnegative supermartingales implies that for each $t_i$, $\E \exp P_{t_i}(\theta)\leq \E \exp P_1(\theta) \leq 1$~\citep[Theorem 5.7.6]{durrett2019probability}.
Moreover, for each $\theta$, the process $\exp(P_{t\vee t_i}(\theta))_{t\geq 0}$ is a supermartingale adapted to the filtration $(\F_{t\vee t_i})$ \citep[Theorem 10.15]{klenke2013probability}. Combining this with Lemma~\ref{lem:mixtures} and the fact that $\nu_i$ is $\F_{t_i-1}$-measurable implies that the process $(\E_{\nu_i}[\exp(P_{t\vee t_i}(\theta))])_{t \ge 0}$ is a supermartingale on the filtration $(\F_{t\vee t_i})$. 
%As above, let $\cT$ be the set of times at which the prior is changed. Note that $\cT$ may be random (in particular, the set of times we choose to change priors is a random process). 
    Applying Theorem~\ref{thm:pac-bayes-bounded-proc}, we have that for all times $t_i$, 
    \[\Pr\left\{\exists t\geq t_i, \exists \rho\in\Mspace{\Theta}: \E_\rho P_t(\theta) - \kl(\rho\|\nu_i) \geq \log\bigg(\frac{s(i)}{\delta}\bigg)\right\}\leq \frac{\delta}{s(i)},\] 
    where $s(i) = (i+1)(i+2)$. 
    If at time $t$ we are using the prior $\nu_i$, then we have switched priors $i$ times, so $t^\dagger = i$. 
    Therefore, 
    \[\big\{\exists t\geq 1,\exists \rho: \E_\rho \fM_t(\theta) \geq \log(s(t^\dagger)/\delta) \big\} \subset  \big\{\exists i\geq 0, \exists t\geq t_i, \exists \rho: \E_\rho \fM_t(\theta) \geq \log(s(i)/\delta) \big\}.\]
    The union bound then implies that $\Pr(\exists t\geq 1, \exists \rho: \E_\rho \fM_t(\theta) \geq \log(s(t^\dagger)/\delta))$ is bounded by 
    \begin{align*}
        & \quad \sum_{i\geq 0} \Pr\left(\{\exists t\geq t_i, \exists \rho: \E_\rho P_t(\theta) - \kl(\rho\|\nu_i)\geq \log(s(i)/\delta)\}\right) \leq \sum_{i\geq 0}\frac{\delta}{s(i)} = \delta.        
    \end{align*}
    %\hrmk{should be $\text{first line}\le \sum_{t_i \in \cT}\Pr\{\exists t\geq t_i, \exists \rho: \E_\rho P_t(\theta) - \kl(\rho\|\nu_i)\geq \log(s(i)/\delta)\} \le \sum_{t_i \in \cT} \frac{\delta}{s(i)} \le \delta $}
    completing the proof. 
\end{proof}
It's perhaps worth noting that there is nothing special about our function $s$ in the above proof, and it only needs to satisfy $\sum_{i=0}^\infty \frac{1}{s(i)}\leq 1$. 
Given such a function, the resulting bound reads $\E_\rho [\fM_t(\theta)] \leq \log (s(t^\dagger)/\delta)$ in place of \eqref{eq:data-dependent-bound}.  

\begin{remark}
    If $\cT$ is finite and deterministic, then we can replace \eqref{eq:data-dependent-bound} with $\E_\rho \fM_t(\theta) \leq \log (|\cT|/\delta)$, thus reducing the quadratic dependence on the number of priors used to a linear dependence. This can be seen by setting $s(i) = |\cT|$ for all $i$ in the above analysis, and noting that the final union bound need only cover $|\cT|$ events. 
\end{remark}

%We note that a drawback of the above analysis is that $\cT$ is assumed to be deterministic. Thus, we must decide beforehand the times at which we want to change the priors. We leave open the interesting question as to whether random times can be accommodated. 
Our result has a different flavor than those of \citet{rivasplata2020pac}
and \citet{awasthi2020pac}. However, we believe its general form is useful and interpretable: For every switch of the prior, the bound suffers an additional additive logarithmic factor. 
}

\edit{
\subsection{Application: Gaussian Process Classification}
\label{sec:gp}

Here we follow \citet{seeger2002pac} and apply the PAC-Bayes framework to Gaussian process classification. 
We take $\cZ = \X\times \{0,1\}$ and 
consider a supervised classification problem with features $x\in\X$ and binary labels $y(x)$. 
For a prediction $\hat{y}(x)$ our loss is the 0-1 loss $\ind(\hat{y}(x)\neq y(x))$. 
We assume that the labels $y(x)$ are generated as $y = \sgn \theta(x)$, where $\theta:\X\to\R$ is some function in a nonparametric family $\Theta$. 
We adopt a  Bayesian perspective and consider $\theta$ to be a random function. Accordingly, we place a zero-mean Gaussian process prior $\nu$ over $\Theta$, i.e., $\theta\sim \gp(0,k)$, where $k$ is a Mercer kernel. 
More precisely, given $x=(x_1,\dots,x_t)$, we have 
\[\nu(\theta(x)) = \frac{|K_x|^{-1/2}}{(2\pi)^{t/2}}\exp\bigg(-\frac{1}{2}x^\intercal K_x^{-1}x\bigg),\]
where $K_x$ is the symmetric matrix whose $ij$-th entry is given by $k(x_i,x_j)$.
Given a distribution $\rho$ over $\Theta$, the empirical risk at time $t$ is 
\begin{equation}
    \label{eq:gp_emp_risk}
    \E_{\theta\sim\rho} \hR_t(\theta) = \frac{1}{t}\sum_{i\leq t} \Pr_{\theta\sim\rho}(\sgn \theta(x_i) \neq y_i),
\end{equation}
and the expected risk is 
\begin{equation}
    \label{eq:gp_risk}
    \E_{\theta\sim\rho} \risk(\theta) = \E_{\theta\sim\rho} \E_{(x,y)\sim \dist} \ind(\sgn\theta(x) \neq y).
\end{equation}
Note that the expected risk is constant because the data are i.i.d.\ and the loss function is stationary. The posterior predictions $\theta(x^t)$ given training data $x^t = (x_1,\dots,x_t)$ and $y^t=(y_1,\dots,y_t)$  depend on the distribution of $y|\theta$, but can be written abstractly as 
\[\rho_t(\theta(x^t)|x^t,y^t) = N(K_{x^t}\mu_t,\Sigma_t),  \]
for some $\mu_t \in \R^t$, $\Sigma_t \in \R^{t\times t}$.
We have introduced the subscript $t$ on $\rho_t$ to emphasize that this is the posterior at time $t$. 
Here $N(\cdot,\cdot)$ is a multivariate normal. Having introduced $\mu_t$ and $\Sigma_t$, we can now write down the KL divergence between the prior $\nu$ and posterior $\rho_t$, which reduces to the KL divergence between their finite-dimensional distributions on the training data (given that $\rho_t$ is defined via the Bayesian update rule), and can be thus calculated via the usual KL divergence between multivariate Gaussians
(see \citealp{seeger2002pac} for the derivation): 
\begin{equation*}
    \kl(\rho_t\|\nu) = \kl(\rho_t(x^t)\|\nu(x^t)) = \frac{1}{2}\log|\Sigma_t^{-1}K_{x^t}| + \frac{1}{2}\text{tr}(\Sigma_t^{-1}K_{x^t})^{-1} + \frac{1}{2}\mu_t^\intercal K_{x^t} \mu_t -t/2.
\end{equation*}
Note that our bound holds simultaneously for the entire sequence of posteriors $(\rho_t)$. As was mentioned following the statement of Theorem~\ref{thm:pac-bayes-bounded-proc}, this showcases how one would typically make full use our anytime valid bounds. Plugging this into Corollary~\ref{cor:anytime-seeger} gives the following result.

\begin{corollary}[Anytime PAC-Bayes bound for GP Classification]
\label{cor:anytime-gp}
Consider the classification setting described above with GP prior $\nu$ and posterior $\rho$. 
With probability at least $1-\delta$ over $(X_t,Y_t)$, for all times $t\geq 1$,
\begin{align*}
  &\klsf(\E_{\rho_t} \hR_t(\theta)\|\E_{\rho_t} \risk(\theta))  
  \leq \frac{\log|\Sigma_t^{-1}K_{x^t}| + \textnormal{tr}(\Sigma_t^{-1}K_{x^t})^{-1} + \mu_t^\intercal K_{x^t} \mu_t -t}{2\bar{t}} + \frac{\log(\xi(\bar{t})/\delta) + \il_t}{\bar{t}},
\end{align*}
where the empirical risk and expected risk are as in \eqref{eq:gp_emp_risk} and \eqref{eq:gp_risk}, $\bar{t}= 2^{\lfloor \log_2(t)\rfloor}$ and $\il_t < 2\log \log 2t + 1.3$.  
\end{corollary}

We recall from Section~\ref{sec:submart_bounds} that $\bar{t}$ and $\il_t$ capture the ``lag'' in our anytime bounds. While we've chosen $\bar{t}$ such that $t/2\leq \bar{t}\leq t$ for convenience, we may choose it to lie in $[t/s,t]$ for any $s>1$, though we pay price in the $\il_t$ term. See Remark~\ref{remark:eta(t)} for details. 

Seeger's fixed time bound for the same problem reads: For any fixed $n$, with probability at least $1-\delta$, 
\begin{align*}
  &\klsf(\E_{\rho_n} \hR_n(\theta)\|\E_{\rho_n} \risk(\theta))  
  \leq \frac{\log|\Sigma_n^{-1}K_{x^n}| + \textnormal{tr}(\Sigma_n^{-1}K_{x^n})^{-1} + \mu_n^\intercal K_{x^n} \mu_n -n}{2n} + \frac{\log(\frac{n+1}{\delta})}{n},
\end{align*}
though we note that Theorem 2 of \citet{seeger2002pac} uses a quasi-inverse of the $\klsf$ function to isolate $\hR_n(\theta)$ and $\risk(\theta)$. Any such inversion also applies here. We refer the reader to \citet{seeger2002pac} for an extensive study on the behavior of the RHS for various GP models. All of his analyses---theoretical and empirical---apply here. 
Finally, let us note that we might have applied any other bound that handles bounded losses (e.g., Corollary~\ref{cor:anytime-subGaussian-losses}). However, the $\klsf$ based bound is often acknowledged as the tightest, often provably so \citep{biggs2023tighter,foong2021tight}. 

}

% The beginning of this is all in the short version only 

\section{Summary}
 We have demonstrated that underlying many PAC-Bayes bounds is a (typically implicit) supermartingale or reverse submartingale structure. 
 Such structure, when coupled with the method of mixtures and Ville's inequalities, provides a general method of deriving new bounds and illuminates the connection between existing ones (Table~\ref{tab:generalization}). 
 For instance, we are able to generate PAC-Bayes bounds for sub-$\psi$ processes \citep[Tables 3 and 4]{howard2020time}, a broad class of processes which itself encapsulates a large swath of existing concentration inequalities. 
 More generally, as soon as one identifies a nonnegative supermartingale or reverse submartingale with bounded initial value, our framework supplies a PAC-Bayes bound. We hope this serves to both ease the search for future bounds and to provide a more unified view of the existing literature. 

 Beyond such unification, our martingale-based 
 approach provides time-uniform bounds (i.e., valid at all stopping times), whereas the majority of previous bounds in the literature are fixed-time results. 
 Moreover, we are able to shed many traditional distribution assumptions. 
 Many of our bounds do not require i.i.d.\ data, and those based on supermartingales require no explicit assumptions (Table~\ref{tab:conditions}). We hope that the anytime nature of our bounds is not just a theoretical curiosity, but useful for computing generalization bounds in practice. Indeed, because they allow for adaptive stopping and continuous monitoring of data, practitioners are able to repeatedly compute the bounds as more data are used without sacrificing statistical validity. 
 %This enables, for instance, deciding when to stop gathering new data based on the evolution of the bound (or confidence sequence, Section~\ref{sec:confseq}) over time. 
 \edit{We hope, for instance, that these properties can serve efforts to generate non-vacuous generalization bounds for neural networks~\citep{dziugaite2017computing,biggs2022non,liao2020pac}. Instead of computing bounds off a fixed test set, our methods enable the collection of more data and the monitoring of the evolution of the bound over time.  
}

\edit{
Aside from applications to neural networks, there are additional practical benefits of the time-uniform nature of our results. Beyond those applications mentioned in the introduction, PAC-Bayes bounds have been used in bandit problems \citep{seldin2012pac,flynn2022lifelong,flynn2022pac}, policy evaluation \citep{fard2011pac,sakhi2023pac}, multiple testing \citep{blanchard2007occam}, estimating means of random vectors \citep{catoni2017dimension}, and
domain adaptation \citep{germain2016new}. Several of these applications could benefit from
our techniques. For instance, \citet{seldin2012pac} rely on a union bound to cover all time steps which might be circumvented by our analysis, leading to tighter bounds for contextual bandits. Similarly, our techniques applied to those of \citet{blanchard2007occam} may lead
to better bounds for sequential multiple testing procedures (bandit multiple testing for instance,
see \citealp{jamieson2018bandit} and \citealp{xu2021unified}), and anytime-valid off-policy evaluation,
in which there has been recent interest \citep{waudby2022anytime}. Overall, we hope our work
contributes to the increasing interest in applying PAC-Bayes ideas to interactive learning
settings, which have a wider scope for applications. 
}

\textbf{Acknowledgements.}
We graciously thank Felix Biggs who pointed out a flaw in the first version of the paper. We also thank the anonymous referees for valuable feedback which improved the paper. 
The authors acknowledge support from NSF grants IIS-2229881 and DMS-2310718. 
BC was partially supported by the  Natural Sciences and Engineering Research Council of Canada.

\bibliography{main.bib}

\appendix

% This contains both the appendix on omitted proofs and that on mixtures of martingales

%\newpage

\section{Omitted Proofs}
 \label{app:proofs}
	
	% \subsection{Proof of Lemma~\ref{lem:change-of-measure}}
	% \label{app:proof-change-of-measure}
	% Let $h:\Theta\to\R$ be measurable. Define the Gibbs measure
	% \begin{equation*}
	% 	\frac{\d\rho_G}{\d\pi}(\theta) = \frac{\exp h(\theta)}{\E_{\phi\sim\pi}{\exp(h(\phi))}}. 
	% \end{equation*}
	% Write  
	% \begin{align*}
	% 	\kl(\rho\|\pi) - \kl(\rho\|\rho_G)  &= \int_\Theta \log\bigg(\frac{\d\rho}{\d\pi}(\theta)\bigg) - \log\bigg(\frac{\d\rho}{\d\rho_G}(\theta)\bigg)\rho(\d\theta)  \\
	% 	&= \int_\Theta \log\bigg(\frac{\d\rho}{\d\pi}(\theta) \frac{\d\rho_G}{\d\rho}(\theta)\bigg)\rho(\d\theta)\\
	% 	&= \int_\Theta \log\bigg(\frac{\d\rho_G}{\d\pi}(\theta)\bigg)\rho(\d\theta) \\ 
	% 	&= \int_\Theta \log\bigg(\frac{\exp h(\theta)}{\E_{\phi\sim\pi}{\exp(h(\phi))}}\bigg)\rho(\d\theta) \\
	% 	&= \E_\rho \bigg\{ h(\theta) - \log \E_\pi \exp(h(\phi)) \bigg\}\\
	% 	&= \E_\rho h(\theta) - \log \E_\pi \exp(h(\phi)).
	% \end{align*}
	% That is, 
	% \begin{equation}
	% 	\label{eq:proof-com1}
	% 	\log\E_\pi \exp(h(\theta)) = \E_\rho h(\theta) - \kl(\rho\|\pi) + \kl(\rho\|\rho_G). 
	% \end{equation}
	% The KL divergence is nonnegative, hence
	% \begin{equation*}
	% 	\log\E_\pi \exp(h(\theta)) \geq  \E_\rho h(\theta) - \kl(\rho\|\pi). 
	% \end{equation*}
	% Since this holds for arbitrary $\rho$, it follows that 
	% \begin{equation*}
	% 	\log\E_\pi \exp(h(\theta)) \geq \sup_{\rho}\bigg\{  \E_\rho h(\theta) - \kl(\rho\|\pi)\bigg\}. 
	% \end{equation*}
	% Moreover, by Equation~\eqref{eq:proof-com1}, equality is reached when $\rho=\rho_G$ (since in that case $\kl(\rho\|\rho_G)=0$) so we may replace the inequality above with an equality.  

	\subsection{Proof of Corollary~\ref{cor:anytime-subGaussian-losses}}
	\label{app:proof-anytime-subGaussian}
	
 Let $\psi(\lambda) = \lambda^2/2$. Define the process $P(\theta)=(P_t(\theta))_{t\geq 1}$ as $P_t(\theta) = \sum_{i=1}^t \lambda_i \Delta_i(\theta) - \sum_{i=1}^t \psi(\lambda_i) \sigma_i^2$. 
We claim that $\exp(P(\theta))$ is a supermartingale. Since $\lambda_i$ and $f_i(Z_i,\theta)$ are $\F_{t-1}$ measurable for all $i\leq t-1$, we have 
	\begin{align*}
		\E_\dist [\exp(P_t(\theta)) | \F_{t-1}] &= \E_\dist \bigg[\prod_{i=1}^t \exp(\lambda_i \Delta_i(\theta) - \psi(\lambda_i)\sigma_i^2)\bigg| \F_{t-1}\bigg] \\
		&= \E_\dist [\exp(\lambda_t\Delta_t(\theta) - \psi(\lambda_t) \sigma_t^2)|\F_{t-1}] \prod_{i=1}^{t-1} \exp(\lambda_i\Delta_i(\theta) - 
		\psi(\lambda_i)\sigma_i^2) \\
		&= \E_\dist [\exp(\lambda_t\Delta_t(\theta) - \psi(\lambda_t)\sigma_t^2)|\F_{t-1}] \exp(P_{t-1}(\theta)),
	\end{align*}
	Now, the final line is at most $\exp(P_{t-1}(\theta))$ due to Hoeffding's lemma: 
	\[\E_\dist[ \exp(\lambda_t \Delta_t(\theta))|\F_{t-1}] = \E_\dist[\exp(\lambda_t (\mu_i(\theta) - f_i(Z_i,\theta)))|\F_{t-1}] \leq \exp(\lambda_t^2\sigma_t^2/8),\]
	for all $\lambda_t\in\R$. 
	This proves that $\exp(P_t(\theta))$ is a supermartingale, and also that $\E_\dist[\exp P_1(\theta)|\F_0]\leq 1$. 
 Consequently, we may apply Corollary~\ref{cor:anytime-pac-bayes-sub-psi} to obtain that with probability at least $1-\delta$,
	\begin{equation*}
		\sum_{i=1}^t\lambda_i 
		\E_\rho\Delta_i(\theta) - \sum_{i=1}^t \psi(\lambda_i) \sigma_i^2 \leq \kl(\rho\|\nu) + \log(1/\delta),
	\end{equation*}
	for all $\rho\in\Mspace{\Theta}$. 
	The result follows from rearranging.

	\subsection{Proof of Corollary~\ref{cor:mixture-subGaussian}}
	\label{app:proof-mixture-subGaussian}
	Let $g(\lambda;a,b^2)$ be the density of a Gaussian with mean $a$ and variance $b^2$. We are interested in  the mixing distribution $F$ with $\d F(\lambda) = g(\lambda;0,\gamma^2)\d \lambda$ for some fixed $\gamma$. 
	Before proving the PAC-Bayes bound, we prove the following lemma. 
 Let $D_t=\sum_{i=1}^t \Delta_i(\theta)$ and $H_t=\sum_{i=1}^t \sigma_i^2$. 
	
	\begin{lemma}
		For 
		\[M_t(\lambda,\theta) := \exp\bigg\{\lambda\sum_{i=1}^t \Delta_i(\theta) - \frac{\lambda^2}{2}\sum_{i=1}^t\sigma_i^2\bigg\},\]
		we have 
		\[M_t(\theta) = \int_{\lambda\in\R} M_t(\lambda,\theta)\d F(\lambda) = \frac{1}{\sqrt{1 + \gamma^2H_t}}\exp\bigg(\frac{\gamma^2D_t^2}{1+\gamma^2H_t}\bigg).\]
	\end{lemma}
	\begin{proof}
		Compute
		\begin{align*}
			M_t(\theta) &= \frac{1}{\gamma\sqrt{2\pi}}\int_{\lambda\in\R} \exp\bigg(\lambda D_t - \frac{\lambda^2 H_t}{2}\bigg)\exp\bigg(-\frac{\lambda^2}{2\gamma^2}\bigg)\d\lambda\\
			&= \frac{1}{\gamma\sqrt{2\pi}}\int_{\lambda\in\R} \exp\bigg(\frac{2\lambda \gamma^2 D_t -  \lambda^2\gamma^2 H_t - \lambda^2}{2\gamma^2}\bigg)\d\lambda\\
			&= \frac{1}{\gamma\sqrt{2\pi}}\int_{\lambda\in\R} \exp\bigg(\frac{-\lambda^2(1+\gamma^2H_t) + 2\lambda\gamma^2D_t}{2\gamma^2}\bigg)\d\lambda.
		\end{align*}
		Define $u=1+\gamma^2 H_t$ and $v=\gamma^2 D_t$. Now rewrite the above expression as 
		\begin{align*}
			M_t(\theta) &= \frac{1}{\gamma\sqrt{2\pi}}\int_{\lambda\in \R} \exp\bigg(\frac{-u(\lambda^2 - 2\lambda v/u)}{2\gamma^2}\bigg)\d\lambda \\
			&= \frac{1}{\gamma\sqrt{2\pi}}\int_{\lambda\in \R} \exp\bigg(\frac{-(\lambda - v/u)^2 + (v/u)^2} {2\gamma^2/u}\bigg)\d\lambda \\
			&= \frac{1}{\gamma\sqrt{2\pi}}\int_{\lambda\in \R} \exp\bigg(\frac{-(\lambda - v/u)^2} {2\gamma^2/u}\bigg)\d\lambda \exp\bigg(\frac{v^2}{2u\gamma^2}\bigg) \\ 
			&= \frac{\sqrt{1/u}}{\sqrt{1\gamma^2/u}\sqrt{2\pi}}\int_{\lambda\in \R} \exp\bigg(\frac{-(\lambda - v/u)^2} {2\gamma^2/u}\bigg)\d\lambda \exp\bigg(\frac{v^2}{2u\gamma^2}\bigg)\\
			&= \sqrt{1/u} \exp\bigg(\frac{v^2}{2u\gamma^2}\bigg),
		\end{align*}
		where the final equality follows because 
		\[\frac{1}{\sqrt{\gamma^2/u}\sqrt{2\pi}}\int_{\lambda\in \R} \exp\bigg(\frac{-(\lambda - v/u)^2} {2\gamma^2/u}\bigg)\d\lambda = \int_{\lambda\in\R} g(\lambda;v/u,2\gamma^2/u)\d\lambda=1,\]
		where $g(\lambda;v/u,\gamma^2/u)$ is the density of a Gaussian with mean $v/u$ and variance $2\gamma^2/u$.  Thus, we have obtained that 
		\begin{align*}
			M_t(\theta) = \frac{1}{\sqrt{u}} \exp\bigg(\frac{v^2}{2u\gamma^2}\bigg) = \frac{1}{\sqrt{1 + \gamma^2H_t}}\exp\bigg(\frac{\gamma^2D_t^2}{1+\gamma^2H_t}\bigg).
		\end{align*}
		This completes the proof of the lemma. 
	\end{proof}
	
	From here, in order to apply Corollary~\ref{cor:anytime-pac-bayes-sub-psi}, write this as 
	\begin{align*}
		M_t(\theta) &=  \frac{1}{\sqrt{1 + \gamma^2H_t}}\exp\bigg(\frac{\gamma^2D_t^2}{1+\gamma^2H_t}\bigg) \\
		&= \exp\bigg(\frac{\gamma^2D_t^2}{1+\gamma^2H_t} + \log([\sqrt{1+\gamma^2H_t}]^{-1})\bigg) \\
		&= \exp\bigg(\frac{\gamma^2D_t^2}{1+\gamma^2H_t} - \frac{1}{2}\log(1+\gamma^2H_t)\bigg). 
	\end{align*}
	Corollary~\ref{cor:anytime-pac-bayes-sub-psi} yields that with probability at least $1-\delta$, for all $t$ and $\rho$,
	\begin{align*}
		\E_\rho \bigg[\frac{\gamma^2D_t^2}{1+\gamma^2H_t} \bigg]\leq  \frac{1}{2}\log(1+\gamma^2H_t) + \kl(\rho\|\nu) +\log(1/\delta).
	\end{align*}
	Rearranging and taking square roots gives 
	\begin{align*}
		\E_\rho[D_t] &\leq  \bigg[(\gamma^{-2}+H_t)\log(1+\gamma^2H_t) + (\gamma^{-2}+H_t)\bigg(\kl(\rho\|\nu) + \log(1/\delta)\bigg)\bigg]^{1/2}\\
		&= \bigg[(\gamma^{-2}+H_t)\bigg(\kl(\rho\|\nu) + \log((1+\gamma^2H_t)/\delta)\bigg)\bigg]^{1/2} \\
  &= 
		\bigg[\frac{s_t(\beta)}{\beta}\bigg(\kl(\rho\|\nu) + \log(s_t(\beta))/\delta)\bigg)\bigg]^{1/2},
	\end{align*}
 where we've taken $\beta=\gamma^2$ and recalled that 
	$s_t(c) = 1 + cH_t$. 
	Expanding the definition of $D_t$ completes the proof.

 \subsection{Proof of Corollary~\ref{cor:bounded-bernstein}}
\label{app:proof-bounded-bernstein}
 Set 
\[\xi_t(\theta) := \lambda_t \Delta_t(\theta) - \lambda_t^2(e-2)\E[\Delta_t^2(\theta)|\F_{t-1}],\]
for all $t\geq 1$. First we claim that the process in Equation \eqref{eq:bernstein-process}, i.e., $B_t(\theta) = \prod_{i=1}^t \exp\xi_i(\theta)$, 
is a nonnegative supermartingale. To see this, we recall the inequality
\begin{equation}
    \label{eq:bernstein-inequalities}
    e^x \leq 1 + x + (e-2)x^2,
\end{equation}
for all $x\leq 1$. Since $\lambda_t\leq |\frac{1}{2H}|$ by assumption, we have 
\[|\lambda_t \Delta_t(\theta)|\leq \lambda_t(|\mu_t(\theta)| + |f_t(Z_t,\theta)|)\leq \lambda_t 2H \leq 1,\]
so we may apply \eqref{eq:bernstein-inequalities}  with $x=\lambda_t \Delta_t(\theta)$. This gives 
\begin{align*}
    \E[\exp(\lambda_t \Delta_t(\theta)|\F_{t-1}] &\leq 1 + \lambda_t \E[\Delta_t(\theta)|\F_{t-1}] + \lambda^2(e-2) \E[\Delta_t^2(\theta)|\F_{t-1}] \\
    &= 1  + \lambda_t^2(e-2) \E[\Delta_t^2(\theta)|\F_{t-1}] \\ 
    &\leq \exp(\lambda_t^2(e-2) \E[\Delta_t^2(\theta)|\F_{t-1}]),
\end{align*}
where the equality in the second line follows by definition of $\Delta_t(\theta)$. Hence, 
\[\E[\exp(\xi_t(\theta))|\F_{t-1}]=\E[\exp(\lambda_t \Delta_t(\theta) - \lambda_t^2(e-2) \E[\Delta_t^2(\theta)|\F_{t-1}])]\leq 1.\]
It follows that 
$(B_t(\theta))$ is a nonnegative supermartingale and the result is then obtained by applying Theorem~\ref{thm:pac-bayes-bounded-proc}. 

\subsection{Proof of Corollary~\ref{cor:bounded-bennet}}
\label{app:proof-bounded-bennett}

Recall that $\psi_P(x) = e^x - x -1$. 
Let $v_i^2(\theta) = \E_\dist [f_i^2(Z_i,\theta)|\F_{t-1}]$. 
Consider the nonnegative process 
\[S_t(\theta) = \prod_{i=1}^t \exp\bigg\{\lambda_i(f_i(Z_i,\theta) - \mu_i(\theta))  - \frac{v_i^2(\theta)}{H_i^2}\psi_P(\lambda_iH_i)\bigg\}.\]
The function $x^{-2}\psi_P(x)$ is nondecreasing (at $x=0$ we continuously extend the function to $1/2$ following the proof of Corollary~\ref{cor:anytime-second-moment}). Since $f_i$ is bounded by $H_i$, we have 
\[\frac{1}{(\lambda_i f_i(Z_i,\theta))^2} \psi_P(\lambda_i f_i(Z_i,\theta)) \leq \frac{1}{(\lambda_i H_i)^2} \psi_P(\lambda_i H_i),\]
that is, 
\[e^{\lambda_i f_i(Z_i,\theta)} \leq \frac{f_i^2(Z_i,\theta)}{H_i^2}\psi_P(\lambda_i H_i) + \lambda_i f_i(Z_i,\theta) + 1.\]
Taking expectations, 
\[\E [e^{\lambda_i f_i(Z_i,\theta)}|\F_{t-1}] \leq \frac{v_i^2(\theta)}{H_i^2}\psi_P(\lambda_i H_i) + \lambda_i \mu_i(\theta) + 1.\]
Note that $\psi_P(x)\geq 0$ for all $x$, so 
\[\frac{v_i^2(\theta)}{H_i^2}\psi_P(\lambda_i H_i) + \lambda_i \mu_i(\theta) \geq \lambda_i\mu_i(\theta) \geq \lambda_i v_i(\theta) > -1,\]
since $\lambda_i < 1/v_i(\theta)$ by assumption. Note that we've used $\mu_i^2(\theta) \leq v_i^2(\theta) $ (by Jensen) so $|\mu_i(\theta)|\leq v_i(\theta)$. 
Therefore, we may take the logarithm of the above and applying the inequality $\log(1+x)\leq x$ for $x>-1$ to obtain 
\[\log \E [e^{\lambda_i f_i(Z_i,\theta)}|\F_{t-1}] \leq \log\bigg(\frac{v_i^2(\theta)}{H_i^2}\psi_P(\lambda_i H_i) + \lambda_i \mu_i(\theta) + 1\bigg) \leq \frac{v_i^2(\theta)}{H_i^2}\psi_P(\lambda_i H_i) + \lambda_i \mu_i(\theta).\]
Adding $-\lambda_i \mu_i(\theta) = \log e^{-\lambda_i \mu_i(\theta)}$ to each side gives 
\[\log \E[e^{\lambda_i\Delta_i(\theta)}|\F_{t-1}] \leq 
\frac{v_i^2(\theta)}{H_i^2}\psi_P(\lambda_i H_i).\]
Exponentiating and rearranging implies that 
\[\exp\bigg\{\lambda_i(f_i(Z_i,\theta) - \mu_i(\theta)) - \frac{v_i^2(\theta)}{H_i^2}\psi_P(\lambda_i H_i)\bigg\}\leq 1,\]
thus implying that $(S_t(\theta))$ is a supermartingale, and the result thus follows from applying Theorem~\ref{thm:pac-bayes-bounded-proc}.

% Taking a log and adding the following inequality (due to Jensen's inequality) on both sides,
% \[-\log \E[ e^{\lambda_i f_i(Z_i,\theta)} | \F_{i-1}] \leq - \E[\lambda_i f_i(Z_i,\theta)|\F_{i-1}]= - \lambda_i\mu_i(\theta),\] 
% we have,
% \begin{align*}
%   \log \E[e^{\lambda_i \Delta_i(\theta)}|\F_{i-1}] &\leq \log\bigg(\frac{\mu_i^2(\theta)}{H_i^2}\phi(\lambda_i H_i) + \lambda_i \mu_i(\theta) + 1\bigg) - \lambda_i\mu_i(\theta)   
%   \leq \frac{\mu_i^2(\theta)}{H_i^2}\phi(\lambda_i H_i),
% \end{align*}
% where we've used the fact that $\log(1 + x)\leq x$. Exponentiating then yields 
% \[\E[e^{\lambda_i \Delta_i(\theta)}|\F_{i-1}] \leq \exp\bigg\{\frac{\mu_i^2(\theta)}{H_i^2}\phi(\lambda_i H_i)\bigg\}.\]
% This demonstrates that $(S_t(\theta))$ is a supermartingale, 

\edit{
\subsection{Proof of Corollary~\ref{cor:unexpected-bernstein}}
Due to the constraints on $(\lambda_t)$ and $(c_t)$, 
    \eqref{eq:unexpected-bernstein} implies that the process defined by 
    \[M_t(\theta) = \prod_{i=1}^t \exp\bigg\{\lambda_i(\mu_i(\theta) - f_i(Z_i,\theta) - c_i f_i^2(Z_i,\theta))\bigg\},\]
    is a nonnegative supermartingale. Proceed as usual and apply Theorem~\ref{thm:pac-bayes-bounded-proc}. 
}

 \subsection{Proof of Corollary~\ref{cor:anytime-bernstein-cond}}
 \label{app:proof-bernstein}
 Recall our assumption: $\E[(f_i(Z_i,\theta) - \mu_i(\theta))^k]\leq \frac{1}{2}k!\sigma_i^2 c_i^{k-2}$. By \citet[Proposition 2.10]{wainwright2019high}, this implies that 
 \begin{equation}
 \label{eq:bernstein-cond}
     \E[\exp(\lambda (\mu_i(\theta) - f_i(Z_i,\theta))) |\F_{i-1}] \leq \exp\bigg(\frac{\lambda^2 \sigma_i^2}{2(1 - c_i|\lambda|)}\bigg),
 \end{equation}
 whenever $|\lambda|<1/c_t$. Consider the quantity $N_t(\theta) = \prod_{i=1}^t \exp\big\{\lambda_i\Delta_i(\theta) - \frac{\lambda_i^2\sigma_i^2}{2(1-c_i\lambda_i)}\big\}.$ 
    Similarly to the proof in Appendix~\ref{app:proof-anytime-subGaussian}, $(N_t(\theta))$ is a supermartingale by appealing to  \eqref{eq:bernstein-cond}, since $0<\lambda_i< 1/c_i$ by assumption. From here we apply Theorem~\ref{thm:pac-bayes-bounded-proc}.

	\subsection{Proof of Corollary~\ref{cor:cgf}}
	\label{app:proof-anytime-cgf}
	
	Consider $W_t(\theta) = \lambda_t f_t(Z_t,\theta) - \log\E_\dist\exp(\lambda_tf_t(Z,\theta))$. Observe that 
 the conditional expectation of $W_t(\theta)$ is precisely 1: 
    \begin{align*}
		\E_\dist[\exp(W_t(\theta))|\F_{t-1}]  &= \E_\dist[\exp(\lambda_t f_t(Z_t,\theta) - \log\E_\dist\exp(\lambda_tf_t(Z,\theta))|\F_{t-1}] \\
		&= \E_\dist[\exp(\lambda_t f_t(Z_t,\theta) \cdot [\E_\dist\exp(\lambda_tf_t(Z,\theta))]^{-1}|\F_{t-1}]\\
		&= [\E_\dist\exp(\lambda_tf_t(Z,\theta))]^{-1} \E_\dist[\exp(\lambda_t f_t(Z_t,\theta)|\F_{t-1}] = 1.
	\end{align*}
 Therefore, $\E[\sum_{i\leq t}W_i(\theta)|\F_{t-1}] = \E[W_t(\theta)|\F_{t-1}]\sum_{i=1}^t W_i(\theta) = \sum_{i=1}^{t-1}W_i(\theta)$, so the process $(\sum_{i\leq t}W_i(\theta))_t$ is a nonnegative supermartingale. Applying Theorem~\ref{thm:pac-bayes-bounded-proc} we obtain that, with probability at least $1-\delta$, for all $t$ and $\rho\in\Mspace{\Theta}$,
	\begin{align*}
		\E_{\theta\sim\rho}\sum_{i=1}^t \lambda_i f_i(Z_i,\theta) \leq \E_{\theta\sim\rho}\sum_{i=1}^t \log \E_\dist\exp(\lambda_i f_i(Z,\theta)) + \kl(\rho\|\nu) + \log(1/\delta).
	\end{align*}
 Using the concavity of the logarithm then completes the argument.

	\subsection{Proof of Corollary~\ref{cor:anytime-second-moment}}
 \label{app:proof-second-moment}
	
	First we prove a self-contained result concerning the relevant supermartingale. From here, the result follows immediately from an application of Theorem~\ref{thm:pac-bayes-bounded-proc}. 
	
	\begin{lemma}
		\label{lem:second-moment-sub-psi}
		Let $(X_t)$ be nonnegative random variables where $X_i$ has conditional mean $\E_{i-1}[X_i]=\E[X_i|\F_{i-1}]$ and conditional variance $\Var_{i-1}(X_i)=\Var(X_i|\F_{i-1})<\infty$. For any predictable sequence of positive real numbers $\{\lambda_i\}$, the following process is a nonnegative supermartingale: 
		\begin{equation*}
			L_t := \prod_{i=1}^t\exp\bigg\{\lambda_i(\E_{i-1}[X_i] - X_i) - \frac{\lambda_i^2}{2}\E_{i-1}[X_i^2]\bigg\}.
		\end{equation*}
	\end{lemma}
	\begin{proof}
		Since $L_{t-1}$ is $\F_{t-1}$ measurable, we obtain 
		\begin{align*}
			\E[L_t|\F_{t-1}] = L_{t-1}\cdot \exp\bigg\{\lambda_t (\E_{t-1}[X_t] - X_t) - \frac{\lambda_t^2}{2}\E_{t-1}[X^2_t]\bigg|\F_{t-1}\bigg\}.
		\end{align*}
		Since $\lambda_t$ is predictable, in order to show the above term is bounded by $L_{t-1}$ it suffices to show that for any nonnegative random variable $X$ with finite mean $\mu$ and second moment we have
		\[\E[\exp(\lambda(\mu-X))] \leq \exp(\lambda^2\E[X^2]/2),\]
		for all $\lambda>0$. 
		This fact follows from applying a one-sided Bernstein inequality to $-X$, but we supply the proof for completeness. Let $Z=-X$ and put $\psi(s)= e^s - s -1$. Let 
		\[g(s) = \begin{cases}
			\psi(s)/s^2,&s\neq 0,\\
			1/2,&s=0.
		\end{cases}\]
		Note that $g(s)$ simply defines the continuous extension of $\psi(s)/s^2$ at $s=0$. Indeed, $\lim_{s\to0^+} \psi(s)/s^2=\lim_{s\to0^-}\psi(s)/s^2=1/2$. Note also $g(s)$ is an increasing function for all $s\in\R$. Therefore, for all $s\leq 0$, $\psi(s) = s^2 g(s) \leq s^2 g(0) = \frac{s^2}{2}.$
	Since $Z\leq 0$ and $\lambda>0$, we may take $s=\lambda Z$ to obtain $\phi(\lambda Z) \leq (\lambda Z)^2/2$. Thus,
		$\E[e^{\lambda Z}] - \lambda \E[Z] - 1 \leq \frac{\lambda^2}{2}\E[Z^2],$
  and 
  \begin{align*}
      \E[\exp(\lambda(Z-\E[Z]))] &\leq e^{-\lambda \E[Z]} ( 1 + \lambda \E[Z] + \lambda^2\E[Z^2]/2) \\
      &\leq e^{-\lambda \E[Z]}\exp(\lambda \E[Z] + \lambda^2\E[Z^2]/2) = \exp(\lambda^2\E[Z^2]/2).
  \end{align*}
   Replacing $Z$ with $-X$ completes the proof. 
\end{proof}

	\subsection{Proof of Corollary~\ref{cor:anytime-sn1}}
 \label{app:proof-anytime-sn1}

Let $(\lambda_i)$ be a predictable sequence. 
 \cite[Proposition 12]{delyon2009exponential} shows that for all $x\in\R$, $\exp(x-x^2/6)\leq 1 + x + x^2/3$. Applying this with $x=\lambda_t\Delta_t(\theta)$ and taking expectations, we obtain that 
 \begin{align*}
     \E[\exp\big\{\lambda_t\Delta_t(\theta) - \lambda_t^2\Delta_t^2(\theta)/6\big\}|\F_{t-1}] &\leq 1 + \E[\lambda_t \Delta_t(\theta)|\F_{t-1}] + \E[\lambda_t^2\Delta_t^2(\theta)/3|\F_{t-1}] \\ 
     &= 1 + \E[\lambda_t^2\Delta_t^2(\theta)/3|\F_{t-1}] \\ 
     &\leq \exp\big\{\E[\lambda_t^2\Delta_t^2(\theta)/3|\F_{t-1}]\big\},
 \end{align*}
 where the equality in the second line follows since $\Delta_t(\theta)$ is mean zero. 
Therefore, 
 \[\E\bigg[\exp\bigg\{\lambda_t\Delta_t(\theta) - \frac{\lambda_t^2}{6}(\Delta_t^2(\theta) + 2\E[\Delta_t^2(\theta)|\F_{t-1}])\bigg\} \bigg| \F_{t-1}\bigg]\leq 1,\]
 and we conclude that 
 \begin{equation*}
     M_t(\theta) = \exp\bigg\{\sum_{i\leq t}\lambda_i\Delta_i(\theta) - \frac{1}{6}\sum_{i\leq t}\lambda_i^2(\Delta_i^2(\theta) + 2\E[\Delta_i^2(\theta)|\F_{i-1}])\bigg\},
 \end{equation*}
 is a nonnegative supermartingale with initial value $\E[M_1(\theta)]\leq 1$. Applying Theorem~\ref{thm:pac-bayes-bounded-proc} gives that with probability at least $1-\delta$, for all $t$ and $\rho$,
 \begin{equation*}
     \sum_{i\leq t}\lambda_i\E_\rho\Delta_i(\theta) \leq  \frac{1}{6}\sum_{i\leq t}\bigg(\lambda_i^2\E_\rho[(\Delta_i^2(\theta) + 2\E[\Delta_i^2(\theta)|\F_{i-1}])]\bigg) + \log(1/\delta) + \kl(\rho\|\nu).
 \end{equation*}
This proves the first part of the result. From here, we can simplify the bound by observing that 
	\begin{align*}
		&\quad \sum_{i\leq t}\Delta_i^2(\theta) + 2\sum_{i\leq t}\E[\Delta_i^2(\theta)|\F_{i-1}] \\ 
  &= \sum_{i=1}^t (\mu_i(\theta) - f_i(Z_i,\theta))^2 + 2\sum_{i=1}^t \E[(\mu_i(\theta) - f_i(Z,\theta))^2 |\F_{i-1}] \\
		&= \sum_{i=1}^t \bigg\{f_i^2(Z_i,\theta) - 2\mu_i(\theta)f_i(Z_i,\theta) + 2\E[f_i^2(Z,\theta)|\F_{i-1}] -\mu_i^2(\theta) \bigg\}\\
		&\leq \sum_{i=1}^t [f_i^2(Z_i,\theta)  + 2\E_\dist[f_i^2(Z,\theta)|\F_{i-1}]],
	\end{align*}
	where we've used that the loss is nonnegative (therefore so is $\mu_i(\theta)$). 
This gives that with probability at least $1-\delta$, for all $t$ and $\rho$,
 \begin{equation*}
     \sum_{i\leq t}\lambda_i\E_\rho\Delta_i(\theta) \leq  \frac{1}{6}\sum_{i\leq t}\lambda_i^2\E_\rho\bigg(f_i(Z_i,\theta) + 2\E_\dist[f_i^2(Z,\theta)|\F_{i-1}]\bigg) + \log(1/\delta) + \kl(\rho\|\nu).
 \end{equation*}
 If we take $f=f_i$ and $\lambda=\lambda_i$ as constants and divide both sides by $t$ we obtain \eqref{eq:sn1-simplified}.

\subsection{Proof of Corollary~\ref{cor:anytime-seeger}}
\label{app:proof-anytime-seeger}
 For $Z_1,\dots,Z_n$ i.i.d, 
 \citet[Theorem 1]{maurer2004note} proved the inequality, 
 \[\E_{(Z_t)\sim \dist} \exp\big\{n \klsf(\hR_n(\theta)\|\risk(\theta))\big\} \leq \E_{B\sim \bin(n,R(\theta))} \exp\big\{n \klsf(B/n\|\risk(\theta))\big\},\]
 where $\bin$ denotes the binomial distribution. 
 Following \citet{germain2015risk}, the latter quantity is equal to $\xi(n)$. Indeed, 
 \begin{align*}
     & \E_{B\sim \bin(n,R(\theta))} \exp\bigg(n \klsf\bigg(\frac{B}{n}\bigg\|\risk(\theta)\bigg)\bigg) \\
     &=  \E_{B\sim \bin(n,R(\theta))} \bigg(\frac{B/n}{R(\theta)}\bigg)^{B}\bigg(\frac{1- B/n}{1 - R(\theta)}\bigg)^{n - B} \\
     &= \sum_{k=0}^n \Pr( B = k)\bigg(\frac{k/n}{R(\theta)}\bigg)^{k}\bigg(\frac{1- k/n}{1 - R(\theta)}\bigg)^{n - k} \\ 
     &= \sum_{k=0}^n {n \choose k} R(\theta)^k(1-R(\theta))^{n-k}\bigg(\frac{k/n}{R(\theta)}\bigg)^{k}\bigg(\frac{1- k/n}{1 - R(\theta)}\bigg)^{n - k} \\ 
     &= \sum_{k=0}^n {n \choose k} (k/n)^k (1 - k/n)^{n-k} = \xi(n).
 \end{align*}
 Therefore, applying Corollary~\ref{cor:anytime-convex} with $\vp=\klsf$ and $\lambda_{\bar{t}}=\bar{t}$ gives 
 \begin{align*}
   \E_\rho  \vp_t(\theta) &\leq \frac{\log\E_{\rho,\dist}\exp(\bar{t}\vp_{\bar{t}}(\theta))}{\bar{t}}  + \frac{\kl(\rho\|\nu) + \log(1/\delta) + \il_t}{\bar{t}} \\
     &\leq \frac{\kl(\rho\|\nu) + \log(\xi(\bar{t})/\delta) +  \il_t}{\bar{t}},
 \end{align*}
 as desired. Finally, \eqref{eq:anytime-seeger-target} follows from similar arguments and applying Corollary~\ref{cor:anytime-convex-target}. 

\edit{
 \subsection{Proof of Corollary~\ref{cor:tolstikhin}}
 \label{app:proof-tolstikhin}
Fix $\lambda>0$. 
    Put 
    \[M_t^j(\theta) = \exp\bigg\{j\lambda_j (\Var(\theta) - \Var_t(\theta)) - \frac{\lambda_j^2}{2}\frac{j^2}{j-1}\Var(\theta)\bigg\}.\]
    Note that $\Var_t(\theta)=U_t(\theta)$, i.e., it is the U-statistic for the functional $\Phi(P,\theta) = \Var_P(f(P,\theta))$. 
    Jensen's inequality combined with the fact that $U_t(\theta)$ is a reverse martingale with respect to $(\mE_t)$ implies that $(M_t^j(\theta))_{t\geq j}$ is a reverse submartingale with respect to $(\mE_t)$. (Note that everything else inside the exponential is constant with respect to $t$.) 
    Moreover, we note that $\E_P [M_j^j(\theta)]\leq 1$ due to the self-bounding property of $U_t(\theta)$ \citep[Equation (9)]{tolstikhin2013pac}. 
    From here, the argument resembles that of Corollary~\ref{cor:anytime-convex}. Theorem~\ref{thm:pac-bayes-bounded-proc} implies that 
    \[\Pr(\exists t\geq j: \E_\rho M_t^j(\theta) - \kl(\rho\|\nu) \geq \log(u/\delta)) \leq \delta/u.\]
    We then apply a union bound over the events $\{ \exists t\geq 2^k: \E_\rho M_t^{2k}(\theta) - \kl(\rho\|\nu) \geq \log(\ell(k+1)/\delta)\}$ implying that 
    \[\Pr(\exists t\geq 1: \E_\rho M_t^{\bar{t}}(\theta) - \kl(\rho\|\nu) \geq \log(\ell(\log_2(2t)/\delta))) \leq \delta,\]
    completing the proof. 
}

 \subsection{Proof of Corollary~\ref{cor:ipm-convex}}
 \label{app:proof-ipm-convex}
The proof follows that of Corollary~\ref{cor:anytime-convex} very closely, so we provide only the outline. Define 
\[h_t^j(\theta) = \lambda_j \vp_t(\theta) - \log \E_{\edit{\nu},\dist} \exp(\lambda_j \vp_j(\theta)).\]
Then $(\exp h_t^j(\theta))$ is a reverse submartingale with respect to $(\mE_t)$ obeying $\E_\dist [\exp h_j^j(\theta)]=1$. Theorem~\ref{thm:anytime-ipm} along with our assumption implies that 
\[\Pr(\exists t\geq 2^k: \E_\rho h_t^{2^k}(\theta) - \gamma_{\cG_t}(\rho,\nu) \geq \log(u/\delta))\leq \delta/u.\]
The event $\{\exists t\geq 1: \E_\rho h_t^{\bar{t}} (\theta) - \gamma_{\cG_t}(\rho,\nu)\geq  \log(\ell(\log_2(2t))/\delta)\}$ is contained in the event $\bigcup_{k=0}^\infty \{\exists t\geq 2^k: \E_\rho h_t^{2^k} - \gamma_{\cG_t}(\rho) \geq \log(\ell(k+1)/\delta)\}$, where $\ell$ is the stitching function introduced in Section~\ref{sec:bound-convex-functional}. The union bound over all such events implies that 
\[\Pr(\exists t\geq 1: \E_\rho h_t^{\bar{t}} (\theta) - \gamma_{\cG_t}(\rho,\nu)\geq  \log(\ell(\log_2(2t))/\delta)) \leq \delta.\]
This proves the first part the result. The second part comes from applying Theorem \ref{thm:anytime-ipm} to the process $(h_t^n(\theta))$ with $t_0=n$. 

\subsection{Proof of Corollary~\ref{cor:renyi-convex}}
\label{app:proof-renyi-convex}
Let $\alpha_0 = \alpha / (\alpha-1)$. 
Define the quantity
\[S_t^j(\theta) = \log \vp_t(\theta) - \frac{1}{\alpha_0}\log \E_{\vartheta\sim\nu,\dist}[ \vp_j^{\alpha_0}(\vartheta)].\]
Note that the final term is not a function of $\theta$. 
First, we claim that $(\exp S_t^j(\theta))_{t\geq 1}$ is a reverse submartingale with respect to $(\mE_t)$. Recalling that $\vp_t(\theta)$ is reverse submartingale with respect to the same filtration, we have 
\begin{align*}
    \E_\dist [\exp S_t^j(\theta)|\mE_{t+1}] &=  \frac{\E_\dist[ \vp_t(\theta)|\mE_{t+1}]}{\E_{\nu,\dist} [\vp_j^{\alpha_0}(\vartheta)]^{\frac{1}{\alpha_0}}} \geq 
    \frac{\vp_{t+1}(\theta)}{\E_{\nu,\dist} [\vp_j^{\alpha_0}(\vartheta)]^{\frac{1}{\alpha_0}}} = \exp S_{t+1}^j(\theta).
\end{align*}
Therefore, it follows from \eqref{eq:alpha-submart} that $([\exp S_t^j(\theta)]^{\alpha_0})_{t\geq 1}$ is a reverse submartingale with respect to $(\mE_t)$. Next we observe that, by construction, $\E_\dist [\exp S_j^j(\theta)^{\alpha_0}]=1$. 
Therefore, by Theorem~\ref{thm:anytime-alpha-div}, for all $\rho$,
\begin{equation*}
    \Pr(\exists t\geq j: \E_\rho S_t^j(\theta) \geq \alpha_0^{-1}(D_\alpha(\rho\|\nu) + \log(u/\delta)) )\leq \delta/u,
\end{equation*}
for $u>0$. 
Let $\ell(k) = k^2\zeta(2)$ be the stitching function introduced in Section~\ref{sec:bound-convex-functional}. 
Following the proof of Corollary~\ref{cor:anytime-convex}, we claim that 
\begin{align*}
    &\quad \big\{\exists t\geq 1: \E_\rho S_t^{\bar{t}}(\theta) \geq \alpha_0^{-1}(D_\alpha(\rho\|\nu) + \log(\ell(\log_2(2t))/\delta))\big\} \\
    &\subset \bigcup_{k=0}^\infty \big\{\exists t\geq 2^k: \E_\rho S_t^{2^k}(\theta) \geq \alpha_0^{-1}(D_\alpha(\rho\|\nu) + \log(\ell(k+1)/\delta))\big\}.
\end{align*}
The argument is identical to before: if there is some $t^*$ such that the first event holds, then $n(t^*) = 2^{k^*}$ for some $k^*$ so 
\begin{align*}
  \E_\rho S_{t^*}^{2^{k^*}}(\theta) &= \E_\rho S_{t^*}^{\bar{t^*}}(\theta) \geq \alpha_0(D_\alpha(\rho\|\nu) + \log(\ell(\log_2(2t^*))/\delta)) \\
  &\geq \alpha_0(D_\alpha(\rho\|\nu) + \log(\ell(k^*+1))/\delta)),  
\end{align*}
since $\log_2(2t^*) = 1 + \log_2(t^*)\geq 1 + \lfloor \log_2(t^*)\rfloor = 1 + k^*$. Applying the union bound, we conclude that 
\[\Pr(\exists t\geq 1: \E_\rho S_t^{\bar{t}}(\theta) \geq \alpha_0^{-1}(D_\alpha(\rho\|\nu) + \log(1/\delta) + \il_t)\leq \sum_{k=1}^\infty \frac{\delta}{\ell(k)}=\delta.\]
That is, with probability at least $1-\delta$, for all $t\geq 1$ and $\rho\in\Mspace{\Theta}$, 
\[\E_\rho \log \vp_t(\theta) - \frac{1}{\alpha_0}\log \E_{\vartheta\sim\nu,\dist}[\vp_{\bar{t}}^{\alpha_0}(\vartheta)] \leq \frac{1}{\alpha_0}(D_\alpha(\rho\|\nu) + \log(1/\delta) + \il_t).\]
The desired result then follows by rearranging, and by noting that $\log \E_\rho \vp_t(\theta) \geq \E_\rho \log \vp_t(\theta)$. 

\subsection{Proof of Corollary~\ref{cor:time-varying-stitch}}
\label{app:proof-time-varying-stitch}

As in Section~\ref{sec:submart_bounds}, we will make use of a nondecreasing function $\ell:\{0,1,2,\dots\}\to\R_{>0}$ such that $\sum_{k=0}^\infty \frac{1}{\ell(k)}\leq 1$. For concreteness, the reader is encouraged to keep $\ell(k) = (k+1)^2\zeta(2)$ in mind, but other options are available. We note that the domain of this function differs slightly from that in Section~\ref{sec:submart_bounds}. This is a matter of convenience only. 

Recall the notation $\Delta_i(\theta) = \mu_i(\theta) - f_i(Z_i,\theta)$ and set $S_t(\theta) = \sum_{i=1}^t \Delta_i(\theta)$. Let $\psi(\lambda) = \lambda^2/2$ be the $\psi$ function for subGaussian random variables. The process $(\exp\{\lambda S_t(\theta) - \psi(\lambda)t\})_{t\geq 1}$ is a nonnegative supermartingale, so 
Corollary~\ref{cor:anytime-pac-bayes-sub-psi} implies that, for all $\lambda\in \R$,
\[\Pr\bigg(\exists t\geq 1: \E_\rho S_t(\theta) \geq \frac{\psi(\lambda)t + \kl(\rho||\nu) + \log(1/\delta)}{\lambda}\bigg)\leq \delta.\]
Let $r=\log(1/\delta)$
and take $g_{\lambda,r}$ to be the lower bound on $\E_\rho S_t(\theta)$: 
\[g_{\lambda,r}(u) = \frac{\psi(\lambda)u + \kl(\rho||\nu) + r}{\lambda}.\]
We can rewrite the time-uniform bound on $\E_\rho S_t(\theta)$ as 
\begin{equation}
\label{eq:mixture-g}
  \Pr(\exists t\geq 1: \E_\rho S_t(\theta) \geq g_{\lambda,r}(t)) \leq e^{-r}.  
\end{equation}
As in the proof of Corollary~\ref{cor:anytime-convex}, we consider geometrically spaced epochs in time: $[2^k,2^{k+1})$ for $k=0,1,\dots$. We wish to employ \eqref{eq:mixture-g} in each epoch $[2^k,2^{k+1})$ with carefully chosen parameters $r_k$ and $\lambda_k$ and then take the union bound over all epochs to obtain our result. 
Following Theorem 1 of \citet{howard2021time}, we select $\lambda_k$ such that $g_{\lambda_k,r_k}(2^k)/2^k = g_{\lambda_k,r_k}(2^{k+1})/2^{k+1}$. This gives $\lambda_k = \psi^{-1}(r_k/2^{k+1/2}) = \sqrt{2r_k/2^{k+1/2}}$. Plugging this into $g$ gives 
\begin{align*}
    g_{\lambda_k,r_k}(u) &= \frac{\sqrt{r_k u}}{\sqrt{2}} \bigg(\sqrt{\frac{u}{2^{k+1/2}}} + \sqrt{\frac{2^{k+1/2}}{u}}\bigg) + \frac{\kl(\rho||\nu)\sqrt{2^{k+1/2}}}{\sqrt{2r_k}}.
\end{align*}
The first term on the right hand side can be bounded by $2\sqrt{r_k u}$ by maximizing $\sqrt{\frac{u}{2^{k+1/2}}} + \sqrt{\frac{2^{k+1/2}}{u}}$ over $u\in [2^k,2^{k+1}]$.
Consider taking $r_k = \log(\ell(k)/\delta)$. Then 
$k\leq \log_2(u)$, so 
$r_k = \log(\ell(k)/\delta) \leq \log(\ell(\log_2(u)/\delta))$, implying the first term can be upper bounded as $2\sqrt{u\log(\log_2(u)/\delta)}$. 
For the KL divergence term, note that $2^k\leq u$ so $\sqrt{2^{k+1/2}}< \sqrt{2u}$. Furthermore, $k+1\geq \log_2(u)$, so $r_k = \log(\ell(k)/\delta)\geq \log(\ell(\log_2(u)-1)/\delta)$. Putting this all together yields 
\[g_{\lambda_k,r_k}(u) \leq 2\sqrt{u\log(\ell(\log_2(u))/\delta)} + \kl(\rho||\nu)\sqrt{\frac{u}{\log(\ell(\log_2(u)-1)/\delta)}} = B_\delta(u).\]
That is, we have shown that for $2^k\leq u<2^{k+1}$, $g_{\lambda_k,r_k}(u)\leq B_\delta(u)$. 

Now, consider the event $\E_\rho S_{t^*}(\theta) >B_\delta(t^*)$. Let $k^*$ be such that $t^*\in[2^{k^*},2^{k^*+1})$. Then $\E_\rho S_{t^*}(\theta)>g_{\lambda_{k^*},r_{k^*}}(t^*)$, implying that the event $\{\exists t\geq 1: \E_\rho S_t(\theta)>B_\delta(t)\}$ is contained in the event $\bigcup_{k=0}^\infty \{\exists t\in[2^k,2^{k+1}): \E_\rho S_t(\theta) > g_{\lambda_k,r_k}(t)\}$. Consequently, \eqref{eq:mixture-g} in conjunction with the union bound implies that 
\begin{align*}
    \Pr(\exists t\geq 1: \E_\rho S_t(\theta) >B_\delta(t)) \leq \sum_{k=0}^\infty e^{-r_k} = \delta\sum_{k=0}^\infty \frac{1}{\ell(k)} \leq \delta.
\end{align*}
We have thus shown that, with probability at least $1-\delta$, for all $t\geq 1$, 
\[\frac{1}{t}\sum_{i=1}^t \E_\rho \mu_i(\theta) \leq \frac{1}{t}\sum_{i=1}^t \E_\rho f_i(Z_i,\theta) + \frac{2\sqrt{\log(\ell(\log_2(t))/\delta)}}{\sqrt{t}} + \frac{\kl(\rho\|\nu)}{\sqrt{t \log(\ell(\log_2(t)-1)/\delta)}}. \]
By considering $-\E_\rho S_t(\theta)$ and taking a union bound we conclude that 
\begin{align*}
  \frac{1}{t}|\E_\rho S_t(\theta)| & \leq \frac{2\sqrt{\log(2\ell(\log_2(t))/\delta)}}{\sqrt{t}} + \frac{\kl(\rho\|\nu)}{\sqrt{t \log(2\ell(\log_2(t)-1)/\delta)}} \\
  &\lesssim \frac{\sqrt{\log(\log(t)) + \log(1/\delta)}}{\sqrt{t}} + \frac{\kl(\rho\|\nu)}{\sqrt{t\log(\log(t)) + t\log(1/\delta)}},
\end{align*}
as claimed.

\subsection{LIL Bound for a Constant Mean}
\label{app:proof-conseq-stitch}
The following is obtained via an ingredient of stitching similar to both~\citet[Theorem 1]{howard2021time} and~\citet[Corollary 10.2]{wang2022catoni}. The resulting width of the boundary is the same as in Corollary~\ref{cor:time-varying-stitch}, but the argument is simpler as the mean is constant. 

\begin{corollary}
\label{cor:confseq-subgaussian-stitch}
    Let $f$ be $1$-subGaussian and  
    let $(Z_t)\sim\dist$ be such that $\mu(\theta) = \E_\dist [f(Z,\theta)|\F_{t-1}]$ is constant for all $t\geq 1$. Fix a prior $\nu\in\Mspace{\Theta}$. 
    Then, for all $\delta\in(0,1)$, with probability at least $1-\delta$, for all $\rho$ and $t\geq1$,
    \[\E_\rho \mu(\theta) \in \bigg(\frac{\sum_{i=1}^t  f(Z_i,\theta)}{t} \pm W_t^\mathsf{stch}\bigg),\]
    where the width $W_t^\mathsf{stch}$ is
    \[2\sqrt{\frac{\log(6.3/\delta) + 1.4 \log \log_2 2t}{t}}  + \frac{\kl(\rho\|\nu)}{\sqrt{(\log(6.3/\delta) + 1.4 \log \log_2 (t+1)) t}}.\]
\end{corollary}

\begin{proof}
    Let
    \begin{equation}
        W_t(\Lambda, \delta) = \frac{\log(2/\delta) + \kl(\rho\|\nu) + \frac{1}{2}t\Lambda^2}{t \Lambda} 
    \end{equation}
    be the width of the CS in \eqref{eqn:subG-cs} when the error level is set to $\delta$, the sequence $\{ \lambda_t \}$ is set to a constant $\Lambda > 0$, and $\sigma$ is set to 1. Let $t_j = 2^j$, $\delta_j = \frac{\delta (1+j)^{-1/4}}{3.15}$, and $\Lambda_j = \sqrt{\log(2/\delta_j) 2^{-j}}$. Note that $\sum_{j=0}^\infty \delta_j < \delta$. By Corollary~\ref{cor:confseq-subgaussian}, with probability at least $1-\delta_j$, for all $\rho$ and integers $t \in [t_j, t_{j+1})$, $\E_\rho \mu(\theta) \in \left(\frac{\sum_{i=1}^t f(Z_i,\theta)}{t} \pm W_t(\Lambda_j, \delta_j) \right)$. Therefore, by the union bound, we have for all $\rho$ and $t$,
    \begin{equation*}
        \E_\rho \mu(\theta) \in \bigg(\frac{\sum_{i=1}^t f(Z_i,\theta)}{t} \pm W_t^{\mathsf{stch*}} \bigg), \quad \text{where}\quad W_t^{\mathsf{stch*}} := W_t(\Lambda_j, \delta_j) \text{ for } t_j \le t < t_{j+1}.
    \end{equation*}
    Next, we show the straightforward fact that $W_t^{\mathsf{stch*}}$ satisfies an iterated logarithmic rate. Note that $\log(6.3/\delta) + 1.4 \log \log_2 (t+1) \le \log(2/\delta_j) \le \log(6.3/\delta) + 1.4 \log \log_2 2t$, so 
    \begin{align*}
        W_t^{\mathsf{stch*}} &=  \frac{\log(2/\delta_j) + \kl(\rho\|\nu) + \frac{1}{2}t\Lambda_j^2}{t \Lambda_j} \\
        &\le \frac{2\log(2/\delta_j) + \kl(\rho\|\nu) }{\sqrt{\log(2/\delta_j) t}}
         \\
         &= 2\sqrt{\frac{\log(2/\delta_j)}{t}} + \frac{\kl(\rho\|\nu)}{\sqrt{\log(2/\delta_j) t}}
        \\ &\le   2\sqrt{\frac{\log(6.3/\delta) + 1.4 \log \log_2 2t}{t}}  + \frac{\kl(\rho\|\nu)}{\sqrt{(\log(6.3/\delta) + 1.4 \log \log_2 (t+1)) t}}.
    \end{align*}
    This concludes the proof. 
\end{proof}

		\section{Mixtures of Martingales}
	\label{app:mixtures}
	
	\begin{lemma}[Mixture of martingales] 
		\label{lem:mixtures}
		Let $\{ (M_t(\theta))_{t \in \mathbb Z}: \theta \in \Theta \}$ be a family of martingales (resp., super/submartingales) on a filtered probability space $(\Omega, \mathcal{A}, (\mathcal{F}_t)_{t \in \mathbb Z}, \Pr)$, indexed by $\theta$ in a measurable space $(\Theta, \mathcal{B})$ such that 
		\begin{enumerate}
			\item[(i)] each $M_t(\theta)$ is $\mathcal{F}_t\otimes \mathcal{B}$-measurable; and
			\item[(ii)] each $\Exp[ M_t(\theta) | \mathcal{F}_{t-1} ]$ is $\mathcal{F}_{t-1}\otimes \mathcal{B}$-measurable.  
		\end{enumerate}
		Let $\mu$ be a %$\sigma$-
		finite measure on $(\Theta, \mathcal{B})$ such that for all $t$,
		\[\Pr\otimes\mu\text{-almost everywhere } M_t(\theta) \ge 0,\quad\text{ or}\quad \E_{\theta\sim\mu}\Exp[ |M_t(\theta)| ] < \infty.\]
		Then the mixture $( M^{\mathsf{mix}}_t )_{t \in \mathbb Z}$, where $M^{\mathsf{mix}}_t = \E_{\theta\sim\mu}M_t(\theta)$, 
		is also a martingale (or super/submartingale).
	\end{lemma}

	\begin{proof} First consider the case of supermartingales. Take any $A \in \mathcal{F}_{t-1}$. 
 %Recall conditions (i) and (ii): (i) each $M_t(\theta)$ is $\mathcal{F}_t\otimes \mathcal{B}$-measurable; and (ii) each $\Exp[ M_t(\theta) | \mathcal{F}_{t-1} ]$ is $\mathcal{F}_{t-1}\otimes \mathcal{B}$-measurable. 
 Employing assumptions (i) and (ii) we can apply Fubini's theorem to $M_t(\theta)$ on $\Pr|_A \otimes \mu$:
 \begin{align*}
     \Exp\left[ \id_{A}\int M_t(\theta) \mu(\d \theta) \right] &= \int \Exp\left[\id_{A}  M_t(\theta) \right] \mu(\d \theta) 
     =  \int \Exp\left[ \id_A \Exp\left[  M_t(\theta)  \middle\vert \mathcal{F}_{t-1} \right] \right] \mu(\d \theta).
 \end{align*}
		Next, again by the assumptions, either $\Pr|_A \otimes\mu$-a.e., $\Exp\left[  M_t(\theta)  \middle\vert \mathcal{F}_{t-1} \right] \ge 0$, or
  \begin{align*}
      \int \Exp\left[ |\Exp\left[  M_t(\theta)  \mid \mathcal{F}_{t-1} \right]| \right] \mu(\d \theta) & \le \int \Exp\left[ \Exp\left[  |M_t(\theta)|  \mid \mathcal{F}_{t-1} \right] \right] \mu(\d \theta) \\
      &= \int \Exp\left[  |M_t(\theta)|   \right] \mu(\d \theta) < \infty.
  \end{align*}
		Hence we can apply Fubini's theorem to $\Exp\left[  M_t(\theta)  \middle\vert \mathcal{F}_{t-1} \right]$ on $\Pr|_A \otimes \mu$:
		\begin{equation*}
			\int \Exp\left[ \id_A \Exp\left[  M_t(\theta)  \middle\vert \mathcal{F}_{t-1} \right] \right] \mu(\d \theta) = \Exp \left[ \id_A \int \Exp\left[  M_t(\theta)  \middle\vert \mathcal{F}_{t-1} \right] \mu(\d \theta)  \right].
		\end{equation*}
		Therefore, for all $A\in\F_{t-1}$, we have
		$\Exp\left[ \id_{A}\int M_t(\theta) \mu(\d \theta) \right] = \Exp \left[ \id_A \int \Exp\left[  M_t(\theta)  \middle\vert \mathcal{F}_{t-1} \right] \mu(\d \theta)  \right]$.
		Further, by Fubini's theorem, $\int \Exp\left[  M_t(\theta)  \middle\vert \mathcal{F}_{t-1} \right] \mu(\d \theta)$ is $\mathcal{F}_{t-1}$-measurable. Hence,
		\[\E\bigg[\int M_t(\theta) \mu(\d \theta) | \mathcal{F}_{t-1}\bigg] =  \int \E[  M_t(\theta)  | \mathcal{F}_{t-1}] \mu(\d \theta),\]
		and so,
		\begin{align*}
			\Exp[  M^{\mathsf{mix}}_t | \mathcal{F}_{t-1} ] &= \Exp\left[ \int M_t(\theta) \mu(\d \theta) \middle\vert \mathcal{F}_{t-1} \right] \\
            &=  \int \Exp\left[  M_t(\theta)  \middle\vert \mathcal{F}_{t-1} \right] \mu(\d \theta) 
            \le \int   M_{t-1}(\theta)   \mu(\d \theta) = M^{\mathsf{mix}}_{t-1}.
		\end{align*}
		The fact that $M^{\mathsf{mix}}_t$ is $\mathcal{F}_{t}$-measurable is again guaranteed by Fubini's theorem. Hence $(M^{\mathsf{mix}}_t )$ is a supermartingale. The case with submartingales can be proven by considering $-M_t(\theta)$. The case with martingales is proven by combining the cases with supermartingales and submartingales.
	\end{proof}

% The proof for reverse submartingales follows the same logic \emph{mutatis mutandis}. 

We remark that the above lemma, albeit stated in terms of forward (super/sub)martingales, immediately implies that the mixture of \emph{reverse} (super/sub)martingales is again a reverse (super/sub)martingale. This is because we allow the indices of the process to run through $t \in \mathbb Z$. To wit, letting $\{ ( N_t(\theta) )_{t =1}^\infty : \theta \in \Theta \}$ be a family of reverse submartingales on a reverse filtered probability space $( 
\Omega, \mathcal{A}, (\mathcal{G}_{t} )_{t =1}^\infty, \mathbb P )$ satisfying the equivalent measurability assumptions,
we may set $M_{-t}(\theta) = N_t(\theta)$ and $\mathcal{F}_{-t} = \mathcal{G}_t$ for $t = 1, 2,\dots$, and trivially extrapolate $M_{0}(\theta) = M_1(\theta) = \dots = N_1(\theta)$, $\mathcal{G}_0=\mathcal{G}_1 = \dots = \mathcal{F}_1$ to make each $( M_t(\theta) )_{t \in \mathbb Z}$ a forward submartingale on the forward filtration $( \mathcal{F}_{t} )_{t \in \mathbb Z}$.  Lemma~\ref{lem:mixtures} is therefore applicable.

\end{document}